\documentclass{article}
\usepackage{microtype}
\usepackage{graphicx}
\usepackage{booktabs}
\usepackage{fullpage}
\usepackage{hyperref}

\newcommand{\citep}[1]{\cite{#1}}
\newcommand{\citet}[1]{\cite{#1}}

\usepackage{tikz}
\usepackage{graphicx}
\usepackage{caption}
\usepackage{subcaption}
\usepackage{selectp}

\usepackage{amsmath,amsfonts,amssymb}
\usepackage{mathtools}
\usepackage{amsthm} 
\usepackage{latexsym}
\usepackage{relsize}
\usepackage{xspace}

\usepackage[frozencache,cachedir=.]{minted}
\usepackage[capitalise]{cleveref}
\usepackage{dsfont}

\newcommand{\defeq}{\stackrel{\text{def}}{=}}

\newcommand{\ignore}[1]{}

\theoremstyle{plain}
\newtheorem{theorem}{Theorem}
\newtheorem{lemma}[theorem]{Lemma}
\newtheorem{corollary}[theorem]{Corollary}
\newtheorem{proposition}[theorem]{Proposition}

\newtheorem{fact}[theorem]{Fact}

\usepackage{subcaption}
\newtheorem*{theorem*}{Theorem}
\newtheorem*{lemma*}{Lemma}
\newtheorem*{corollary*}{Corollary}
\newtheorem*{proposition*}{Proposition}
\newtheorem*{claim*}{Claim}
\newtheorem*{fact*}{Fact}
\newtheorem*{observation*}{Observation}

\allowdisplaybreaks
\theoremstyle{definition}
\newtheorem{definition}[theorem]{Definition}
\newtheorem{remark}[theorem]{Remark}

\newtheorem*{definition*}{Definition}
\newtheorem*{remark*}{Remark}
\newtheorem*{example*}{Example}

 \theoremstyle{plain}
\newtheorem*{theoremaux}{\theoremauxref}
\gdef\theoremauxref{1}

\DeclareMathAlphabet{\mathbfsf}{\encodingdefault}{\sfdefault}{bx}{n}

\DeclareMathOperator*{\argmin}{arg\,min}

\newcommand{\abs}[1]{\left|#1\right|}

\newcommand{\eps}{\varepsilon}

\renewcommand{\leq}{~\le~}
\renewcommand{\geq}{~\ge~}

\let\oldtfrac\tfrac
\renewcommand{\tfrac}[2]{\smash{\oldtfrac{#1}{#2}}}

\newcommand{\pa}[1]{\left(#1\right)}

\newcommand{\bra}[1]{\left[#1\right]}
\renewcommand{\abs}[1]{\left|#1\right|}

\newcommand{\Nbb}{\mathbb{N}}
\newcommand{\N}{\mathcal{N}}
\newcommand{\R}{\mathbb{R}}

\newcommand{\norm}[1]{\left\lVert#1\right\rVert}

\newcommand{\Tcheb}{\mathcal{T}}
\newcommand{\Ucheb}{\mathcal{U}}
\newcommand{\out}{\mathrm{out}}
\newcommand{\Fcal}{\mathcal{F}}
\newcommand{\Gcal}{\mathcal{G}}
\newcommand{\Pcal}{\mathcal{P}}
\newcommand{\Bcal}{\mathcal{B}}
\newcommand{\Vcal}{\mathcal{V}}
\newcommand{\bits}{\mathrm{bits}}
\newcommand{\kappahat}{\widehat{\kappa}}

\newcommand{\acosh}{\mathrm{acosh}}

\newcommand{\goodcolor}[1]{{\color{blue}#1}}
\newcommand{\badcolor}[1]{{\color{red}#1}}

\newtheorem{manualtheoreminner}{Theorem}
\newenvironment{manualtheorem}[1]{%
  \renewcommand\themanualtheoreminner{#1}%
  \manualtheoreminner
}{\endmanualtheoreminner}

\newtheorem{manualpropositioninner}{Proposition}
\newenvironment{manualproposition}[1]{%
  \renewcommand\themanualpropositioninner{#1}%
  \manualpropositioninner
}{\endmanualpropositioninner}

\def\arxiv{} 

\title{Acceleration via Fractal Learning Rate Schedules}

\author{
  Naman Agarwal$^1$ \qquad Surbhi Goel$^2$ \qquad Cyril Zhang$^{2}$ \\
  \vspace{-2mm} \\
  $^1$Google AI Princeton \\
  $^2$Microsoft Research \\
  \vspace{-2mm} \\
   \texttt{namanagarwal@google.com}, \\ \texttt{\{goel.surbhi, cyrilzhang\}@microsoft.com} 
}

\begin{document}

\maketitle

\begin{abstract}
In practical applications of iterative first-order optimization, the learning rate schedule remains notoriously difficult to understand and expensive to tune. We demonstrate the presence of these subtleties even in the innocuous case when the objective is a convex quadratic. We reinterpret an iterative algorithm from the numerical analysis literature as what we call the \emph{Chebyshev learning rate schedule} for accelerating vanilla gradient descent, and show that the problem of mitigating instability leads to a \emph{fractal} ordering of step sizes. We provide some experiments to challenge conventional beliefs about stable learning rates in deep learning: the fractal schedule enables training to converge with locally unstable updates which make negative progress on the objective.
\end{abstract}

\section{Introduction}

In the current era of large-scale machine learning models, a single deep neural network can cost millions of dollars to train. Despite the sensitivity of gradient-based training to the choice of learning rate schedule, no clear consensus has emerged on how to select this high-dimensional hyperparameter, other than expensive end-to-end model training and evaluation. Prior literature indirectly sheds some light on this mystery, showing that the learning rate schedule governs tradeoffs between accelerated convergence and various forms of algorithmic stability.

In this work, we highlight the surprising consequences of these tradeoffs in a very simple setting: first-order optimization of a convex quadratic function. We start by pointing out the existence of a non-adaptive step size schedule, derived from the roots of Chebyshev polynomials, which allows plain gradient descent to obtain accelerated convergence rates without momentum. These learning rates overshoot the region of guaranteed local progress, resulting in unstable optimization trajectories. Extending a relatively obscure line of work motivated by numerical imprecision in PDE solvers \cite{lebedev1971order}, we show that stable acceleration is achieved by selecting a \emph{fractal} permutation of the Chebyshev step sizes.

Acceleration via large step sizes may provide an useful alternative to momentum: it is less stable according to our worst-case bounds, but inherits the memory-efficiency and statelessness of vanilla gradient descent. More broadly, we discuss how this form of acceleration might implicitly present itself in settings like deep learning, introducing hidden entanglements and experimental confounds. We hope that these ideas will lead to new adaptive algorithms which overstep the ``edge of stability'' (the largest constant learning rate at which model training converges) \citep{giladi2019stability,cohen2020gd}, and accelerate training via carefully scheduled negative progress. We provide some supporting experiments towards bridging the theory-practice gap, as well as open questions for future investigation.

\begin{figure}
    \centering
\ifdefined\arxiv 
    \includegraphics[height=0.2\linewidth]{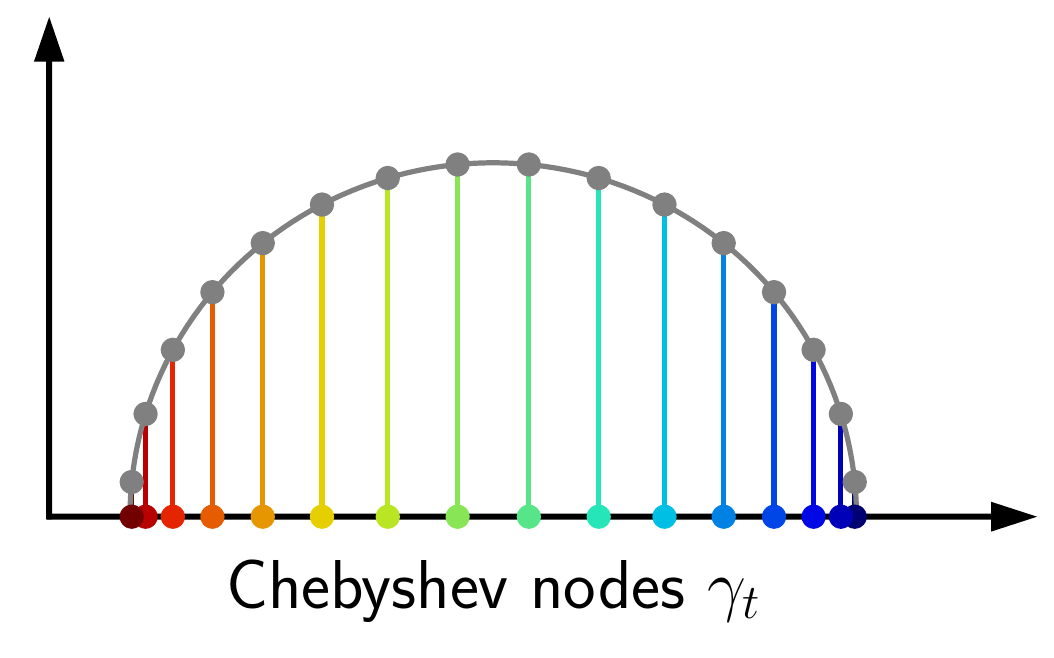}
    \includegraphics[height=0.2\linewidth]{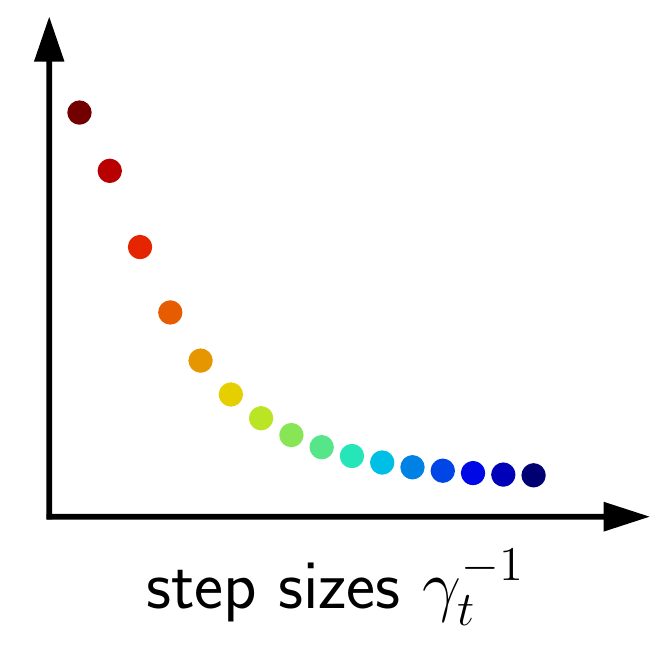}
    \includegraphics[height=0.2\linewidth]{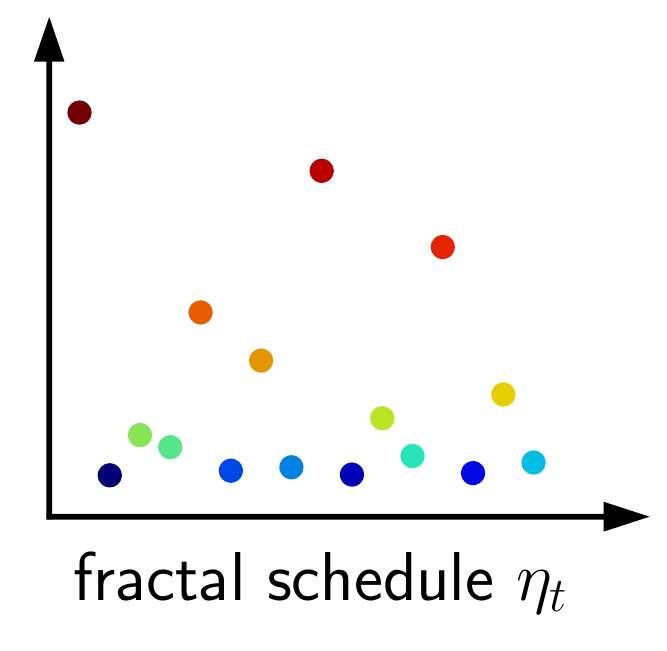}
\else
    \includegraphics[height=0.26\linewidth]{figures/semicircle.pdf}
    \includegraphics[height=0.26\linewidth]{figures/sched_monotone.pdf}
    \includegraphics[height=0.26\linewidth]{figures/sched_sigma.pdf}
\fi
    \caption{Visualization of the Chebyshev nodes $\gamma_t$, their corresponding step sizes $\gamma_t^{-1}$, and the fractal permutation \cite{lebedev1971order} studied in this paper.}
    \label{fig:cheb-lr-overview}
\end{figure}

\subsection{Our contributions}
\paragraph{Provably stable acceleration without momentum.} We revisit an oft-neglected variant of the Chebyshev iteration method for accelerating gradient descent on convex quadratics. In lieu of momentum, it uses a recursively-defined sequence of large step sizes derived from Chebyshev polynomials, which we call the fractal Chebyshev schedule. We prove a new stability guarantee for this algorithm: under bounded perturbations to all the gradients, no iterate changes by more than $O(\mathrm{poly}(\kappa))$, where $\kappa$ is the condition number of the problem. We also some provide theoretically-grounded practical variants of the schedule, and negative results for function classes beyond convex quadratics.

\paragraph{Empirical insights on stable oscillating schedules.} We demonstrate empirically that the fractal Chebyshev schedule stabilizes gradient descent on objectives beyond convex quadratics. We observe accelerated convergence on an instance of multiclass logistic regression, and convergent training of deep neural networks at unstable learning rates. These experiments highlight the power of optimizing the ``microstructure'' of the learning rate schedule (as opposed to global features like warmup and decay). We discuss how these findings connect to other implicit behaviors of SGD and learning rate schedules.

\subsection{Related work}
The predominant algorithms for accelerated first-order optimization are the momentum methods of \citet{polyak1964some} and \citet{nesterov1983method}. The former, known as the \emph{heavy-ball} method, only achieves provable acceleration on quadratic objectives. The latter achieves minimax optimal convergence rates for general smooth convex objectives. Both are widely used in practice, far beyond their theoretical scope; for instance, they are the standard options available in deep learning frameworks.

\paragraph{Empirical challenges and tradeoffs.} \citep{bottou2007tradeoffs} discuss the competing objectives of stability, acceleration, and computation in large-scale settings, where one cannot afford to consider a single asymptotically dominant term. \citet{devolder2014first,chen2018stability,agarwal2020stochastic} study this specifically for acceleration. Optimizing the learning rate schedule remains a ubiquitous challenge; see Section~\ref{subsec:experiments-dl} and Appendix~\ref{subsec:appendix-lr-schedule-practice} for references.

\paragraph{Numerical methods and extremal polynomials.} There are many connections between algorithm design and approximation theory \citep{vishnoi2012laplacian,sachdeva2013faster}. We emphasize that the beautiful idea of the fractal permutation of Chebyshev nodes is an innovation by \citet{lebedev1971order,lebedev1973solution,lebedev1976utilization}; our technical results are generalizations and refinements of the ideas therein.
We give an overview of this line of work in Appendix~\ref{subsec:appendix-lebedev-literature}.

\paragraph{Learning rate schedules in stochastic optimization.} Bias-variance tradeoffs in optimization are studied in various theoretical settings, including quadratics with additive and multiplicative noise \citep{lan2012optimal,ge2019step,gorbunov2020unified}. Many of them also arrive at theoretically principled learning rate schedules; see Appendix~\ref{subsec:appendix-lr-schedule-theory}. On the more empirical side, \citet{zhang2019algorithmic} use a noisy quadratic model to make coarse predictions about the dynamics of large-scale neural net training. Cyclic learning rate schedules have been employed in deep learning, with various heuristic justifications \cite{loshchilov2016sgdr,smith2017cyclical,fu2019cyclical}. In parallel work, \cite{oymak2021super} considers a cyclic ``1 high, $n$ low'' schedule, which gives $\log(\kappa)$ convergence rates in the special case of convex quadratics whose Hessians have bimodal spectra. We discuss in Appendix~\ref{subsec:appendix-vanilla-spiky} why this approach does not provide acceleration in the general case; the MNIST experiments in Appendix~\ref{subsec:appendix-deeplearning} include a comparison with this schedule.
\section{Preliminaries}

\subsection{Gradient descent}
We consider the problem of iterative optimization of a differentiable function $f : \R^d \rightarrow \R$, with a first-order oracle $\nabla f : \R^d \rightarrow \R^d$ which computes the gradient of $f$ at a query point. The simplest algorithm in this setting is gradient descent, which takes an arbitrary initial iterate $x_1 \in \R^d$ and executes $T$ update steps
\begin{equation}
\label{eq:gd}
\{ x_{t+1} \leftarrow x_t - \eta_t \nabla f(x_t) \}_{t=1}^{T}
\end{equation}
according to a learning rate schedule $(\eta_1, \ldots, \eta_T)$, producing a final iterate $x_{\out} := x_{T+1}$. When the $\{\eta_t\}$ do not depend on $T$, an analogous infinite sequence of iterates $\{x_t\}_{t \in \Nbb}$ can be defined.

There are many ways to choose the learning rate schedule, depending on the structure of $f$ and uncertainty in the gradient oracle. Some schedules are static (non-adaptive): $(\eta_1, \ldots, \eta_T)$ are chosen before the execution of the algorithm. For instance, when $f$ is an $M$-smooth convex function, $\eta_t = 1/M$ achieves the classical convergence rates.

Adaptive choices of $\eta_t$ are allowed to depend on the observed feedback from the current execution (including $x_t$ and $\nabla f(x_t)$), and are considerably more expressive. For example, $\eta_t$ can be chosen adaptively via line search, adaptive regularization, or curvature estimation.

\subsection{The special case of quadratics}
\label{subsec:prelims-quadratic-opt}

Consider the case where the objective is of the form
\[f(x) = \frac{1}{2} x^\top A x - b^\top x,\]
where $A \in \R^{d\times d}$ is symmetric and positive definite, and $b \in \R^d$,
so that $\nabla f(x) = Ax - b$ is an affine function of the query point $x$. Then, the mapping $\Gcal : x_t \mapsto x_{t+1}$ induced by gradient descent is also affine. Let $x^* := \min f$ (a fixed point of $\Gcal$). Then,
\begin{align*}
x_{t+1} - x^* &= \Gcal(x_t) - x^* = \Gcal(x_t) - \Gcal(x^*) \\
&= (I - \eta_t A)(x_t - x^*).
\end{align*}

By induction, we conclude that
\begin{align*}
x_{\out} - x^* = \bra{ \prod_{t=1}^T (I - \eta_t A) }(x_1 - x^*).
\end{align*}

Thus, the residual after $T$ steps of gradient descent is given by a degree-$T$ matrix polynomial times the initial residual:
\begin{definition}[Residual polynomial]
Fix a choice of non-adaptive $(\eta_1, \ldots, \eta_T)$. Then, define the \emph{residual polynomial} $p : \R^{d \times d} \rightarrow \R^{d \times d}$ as
\[ p(A) := \prod_{t=1}^T (I - \eta_t A). \]
\end{definition}
When clear, we will interchange to denote scalar and matrix polynomials with the same coefficients. Thus, overloading $p : \R \rightarrow \R$, we have $p(0) = 1$, and $p(1/\eta_t) = 0$ for each $t$.

\begin{remark}
The matrices in the above product all commute. Thus, when $f$ is quadratic, $p(A)$ (and thus $x_{\out}$ given $x_1$) does not depend on the permutation of $(\eta_1, \ldots, \eta_T)$.
\end{remark}

\subsection{Chebyshev polynomials and Chebyshev methods}
\label{subsec:prelims-cheb}

The problem of choosing $p(A)$ to optimize convergence for least-squares has roots in numerical methods for differential equations \cite{richardson1911ix}. The Chebyshev polynomials, which appear ubiquitously in numerical methods and approximation theory \cite{chebyshev1853theorie,mason2002chebyshev}, provide a minimax-optimal solution \cite{flanders1950numerical,gavurin1950use,young1953richardson}\footnote{For a modern exposition, see the blogpost \url{http://fa.bianp.net/blog/2021/no-momentum/}.}: choose positive real numbers $m \leq M$, and set
\[ p(\lambda) = \frac{ \Tcheb_T\pa{ z } }{ \Tcheb_T(\theta) }, \]
where $z := \frac{M+m-2\lambda}{M-m}$, $\theta := \frac{M+m}{M-m} = 1 + \frac{2m}{M-m}$, and $\Tcheb_n(\cdot)$ is the degree-$n$ Chebyshev polynomial of the first kind. One of many equivalent definitions is $\Tcheb_n(z) = \cos (n \arccos z)$ for $|z| \leq 1$. From this definition it follows that the roots of $p$ occur at the \emph{Chebyshev nodes}
\[ \gamma_t := \frac{M+m}{2} - \frac{M-m}{2} \cos \frac{(t-\frac{1}{2})\pi}{T}, \; t = 1, \ldots, T. \]
Setting $\{\eta_t\}$ to be any permutation of $\{1/\gamma_t\}$ suffices to realize this choice of $p$. Note that $1/\gamma_t$ is decreasing in $t$.
The limiting case $m = M$ is gradient descent with a constant learning rate, and $p(\lambda) = (1 - \lambda/m)^T$.

Let $\lambda_{\min}, \lambda_{\max}$ denote the smallest and largest eigenvalues of $A$, so that the \emph{condition number} of $A$ is $\kappa := \lambda_{\max}/\lambda_{\min}$. Viewing $m,M$ as estimates for the spectrum, we define
\[\kappahat := \frac{M}{m} \geq \frac{\lambda_{\max}}{\lambda_{\min}} = \kappa.\]
We state a classic end-to-end convergence rate for Chebyshev iteration (proven in Appendix~\ref{sec:appendix-cheb-background} for completeness):
\begin{theorem}[Convergence rate of Chebyshev iteration]
\label{thm:cheb-convergence-rate}
Choose spectral estimates $m \leq M$ such that $0 < m \leq \lambda_{\min} \leq \lambda_{\max} \leq M$. Then, setting $\{\eta_t\}$ to be any permutation of $\{1/\gamma_t\}$, the final iterate of gradient descent $x_{\out}$ satisfies the following:
\begin{align*}
\norm{x_{\out} - x^*} &\leq \frac{2\rho^T}{1+\rho^{2T}} \norm{x_1 - x^*} \\
&\leq e^{-\Omega(T)/\sqrt{\kappahat}} \norm{x_1 - x^*},
\end{align*}
where $\rho := \frac{ \sqrt{M} - \sqrt{m} }{ \sqrt{M} + \sqrt{m} } \leq 1 - \Omega\pa{\frac{1}{\sqrt{\kappahat}}}$.
\end{theorem}

Thus, accelerated methods like Chebyshev iteration get $\eps$-close to the minimizer in $O(\sqrt{\kappahat} \log (1/\eps))$ iterations, a quadratic improvement over the $O(\kappahat \log (1/\eps))$ rate of gradient descent with a constant learning rate. Theorem~\ref{thm:cheb-convergence-rate} is proven using approximation theory: show that $|p(\lambda)|$ is small on an interval containing the spectrum of $A$.
\begin{definition}[Uniform norm on an interval]
Let $p : \R \rightarrow \R$, and $m \leq M \in \R$. Define the norm
\[\norm{p}_{[m,M]} := \norm{p}_{L_\infty([m,M])} = \max_{\lambda \in [m,M]} |p(\lambda)|. \]
\end{definition}
Then, any upper bound on this norm gives rise to a convergence rate like Theorem~\ref{thm:cheb-convergence-rate}:
\[ \norm{x_{\out} - x^*} \leq \norm{p}_{[m,M]} \cdot \norm{x_1 - x^*}. \]
These can be converted into optimality gaps on $f$ by considering the polynomial $\lambda \, p^2(\lambda)$.

Moving beyond infinite-precision arithmetic, the optimization literature typically takes the route of \citet{stiefel1958kernel}, establishing a higher-order recurrence which ``semi-iteratively'' (iteratively, but keeping some auxiliary state) constructs the same final polynomial $p$. This is the usual meaning of the Chebyshev iteration method, and coincides with Polyak's momentum on quadratics.

This is where we depart from the conventional approach.\footnote{For instance, this is not found in references on acceleration \citep{bubeck2017convex,d2021acceleration}, or in textbooks on Chebyshev methods \citep{gottlieb1977numerical,higham2002accuracy}.} We revisit the idea of \emph{working directly with the Chebyshev step sizes}, giving a different class of algorithms with different trajectories and stability properties.
\section{The fractal Chebyshev schedule}
\label{sec:quadratic}
In this section, we work in the strongly\footnote{Accelerated rates in this paper have $O(1/T^2)$ analogues when $\lambda_{\min} = 0$ \citep{allen2016optimal}.} convex quadratic setting from Section~\ref{subsec:prelims-quadratic-opt}.
Our new contributions on top of the existing theory address the following questions:
\begin{enumerate}
    \item[(1)] How noise-tolerant is gradient descent with Chebyshev learning rates, beyond numerical imprecision?
    \item[(2)] How do we choose the ordering of steps?
\end{enumerate}

We first introduce the construction originally motivated by numerical error, which provides an initial answer to (2). Then, our extended robustness analysis provides an answer to (1), and subsequently a more refined answer to (2).

\subsection{Construction}
We begin with the construction from \cite{lebedev1971order}, defined below and visualized in Figure~\ref{fig:scheds}.

\begin{definition}[Fractal Chebyshev schedule]
Let $\sigma_1 := [1]$, and for each $T \geq 1$ a power of 2, define
\[\sigma_{2T} := \mathrm{interlace}(\sigma_T, 2T+1-\sigma_T), \]
where
\[\mathrm{interlace}([a_1 \ldots a_n], [b_1 \ldots b_n]) := [a_1 \; b_1 \; a_2 \; b_2 \ldots a_n \; b_n].\]
Then, for given $m \leq M$, and $T$ a power of 2, the \emph{fractal Chebyshev schedule} is the sequence of learning rates
\[ \eta_t := 1/\gamma_{\sigma_{T}(t)}, \quad t = 1, \ldots, T. \]
\end{definition}
Below are the first few nontrivial permutations $\sigma_T$:

\[\sigma_2 = [1\;2],\]
\[\sigma_4 = [1\;4\;2\;3],\]
\[\sigma_8 = [1\;8\;4\;5\;2\;7\;3\;6],\]
\[\sigma_{16} = [1\;16\;8\;9\;4\;13\;5\;12\;2\;15\;7\;10\;3\;14\;6\;11].\]

\begin{figure}
    \centering
\ifdefined\arxiv 
    \includegraphics[width=0.8\linewidth]{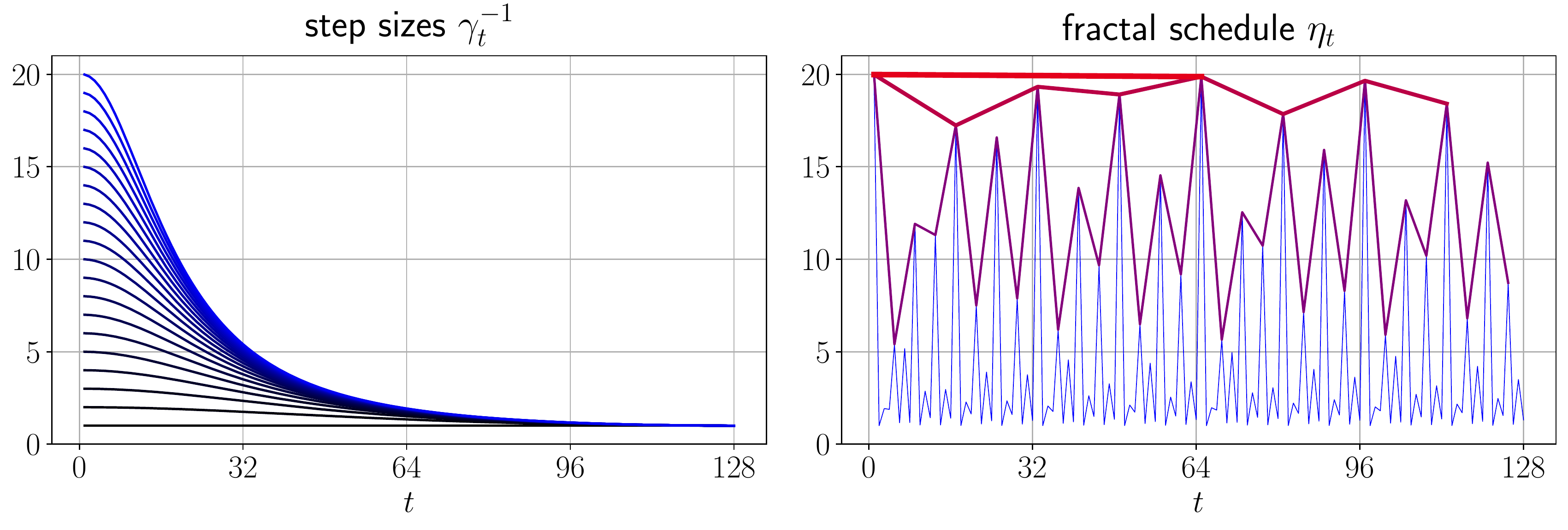}
\else
    \includegraphics[width=\linewidth]{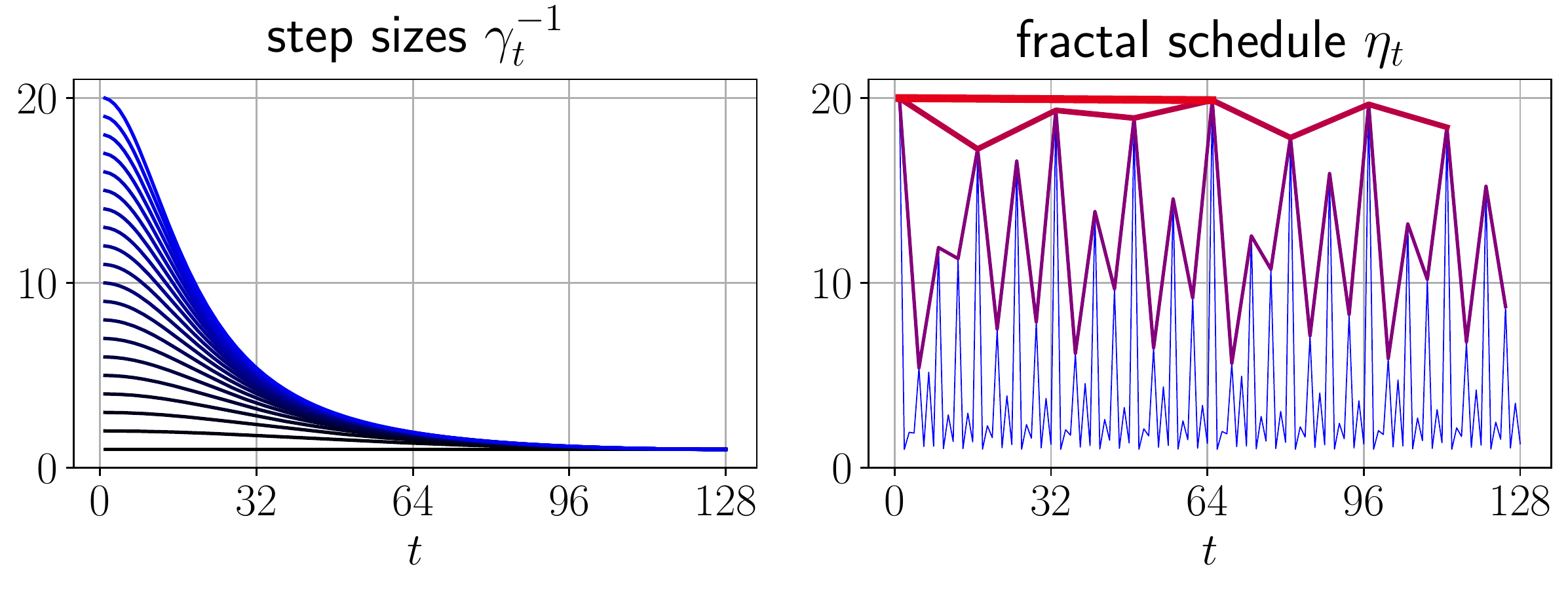}
    \vspace{-5mm}
\fi
    
    \caption{Shapes of the Chebyshev step sizes and fractal permutations. \emph{Left:} Step sizes in sorted order for $M=1$, and $m=1,\frac{1}{2},\ldots,\frac{1}{20}$ (black to blue). \emph{Right:} Permuted schedule with $M=1, m=\frac{1}{20}, T=128$ (red). Subsequences with strides $\{1, 4, 16, 64\}$ are overlaid, demonstrating self-similarity arising from the interlacing construction.}
    \label{fig:scheds}
\end{figure}

\subsection{Basic properties}
We first list some basic facts about the unordered step sizes:
\begin{proposition}
\label{prop:cheb-basics}
For all $m < M$ and $T$, the fractal Chebyshev step sizes $\{\gamma_t^{-1}\}$ satisfy the following:
\begin{enumerate}
    \item[(i)] $\frac{1}{M} < \gamma_t^{-1} < \frac{1}{m} = \frac{\kappahat}{M}$.
    \item[(ii)] The number of step sizes greater than $\frac{2}{M}$ is $\pa{ \frac{1}{2} - \eps }T$, where $0 \leq \eps \leq O(1/\kappahat)$ as $\kappahat \rightarrow \infty$.
    \item[(iii)] For $t \leq \frac{T}{2}$, we have $\gamma_t^{-1} < \frac{1}{m + \frac{2(M-m)t^2}{T^2}}$, and
    
    \hspace{-3mm}$\frac{1}{T}\sum_{t=1}^T \gamma_t^{-1} =  \frac{\tanh\pa{ T \, \acosh\left(\frac{2m}{M-m}\right)}}{\sqrt{Mm}} < \frac{1}{\sqrt{Mm}} = \frac{\sqrt{\kappahat}}{M}.$
\end{enumerate}
\end{proposition}

Interpreting $m, M$ as estimates for $\lambda_{\min}, \lambda_{\max}$:
\begin{enumerate}
    \item[(i)] \emph{Every} step size in the schedule exceeds the classic fixed learning rate of $1/\lambda_{\max}$. As $T$ gets large, the largest step approaches $1/\lambda_{\min}$, a factor of $\kappa$ larger.
    \item[(ii)] For large $\kappa$, close to half of the step sizes \emph{overshoot} the stable regime $\eta \in [0, 2/\lambda_{\max}]$, where local progress on $f$ is guaranteed.
    \item[(iii)] The large steps are neither highly clustered nor dispersed. The largest $\gamma_t^{-1}$ overshoots the stable regime by a factor of $\Theta(\kappa)$, but the average factor is only $O(\sqrt{\kappa})$.
\end{enumerate}

Next, some basic observations about the fractal schedule:
\begin{proposition}[Hierarchy and self-similarity]
\label{prop:cheb-fractal-basis}
For all $m,M,T$ and $0 \leq i \leq \log_2 T$:
\begin{enumerate}
    \item[(i)] The largest $\frac{T}{2^i}$ steps $\eta_t$ in the fractal Chebyshev schedule occur when $t = 1 + 2^i(\tau-1)$, with $\tau = 1, \ldots, \frac{T}{2^i}$. 
    \item[(ii)] The subsampled sequence $\{\eta_{1+2^i(\tau-1)}\}$ has the same ordering as the fractal permutation of the same length:
    \[ \eta_{1+2^i\tau} = \gamma^{-1}_{1+2^i(\tau' - 1)}, \quad \text{ where } \tau' = \sigma_{T/2^i}(\tau). \]
\end{enumerate}
\end{proposition}

Figure~\ref{fig:scheds} visualizes these observations, while Appendix~\ref{subsec:appendix-basics} contains formal statements and proofs.

\subsection{Self-stabilization via infix polynomial bounds}
Now, let us examine why the fractal ordering is needed. As discussed, in the noiseless infinite-precision setting, the final iterate $x_{\out}$ is invariant to the permutation of $\{\eta_t\}$. However, the intermediate iterates $x_t$ depend on a sequence of \emph{partial} products, which depend very sensitively on the permutation; Figure~\ref{fig:perm-stability} illustrates these tradeoffs; details are found in Appendix~\ref{subsec:appendix-perm-stability}.

\begin{figure}
    \centering
\ifdefined\arxiv 
    \includegraphics[width=0.8\linewidth]{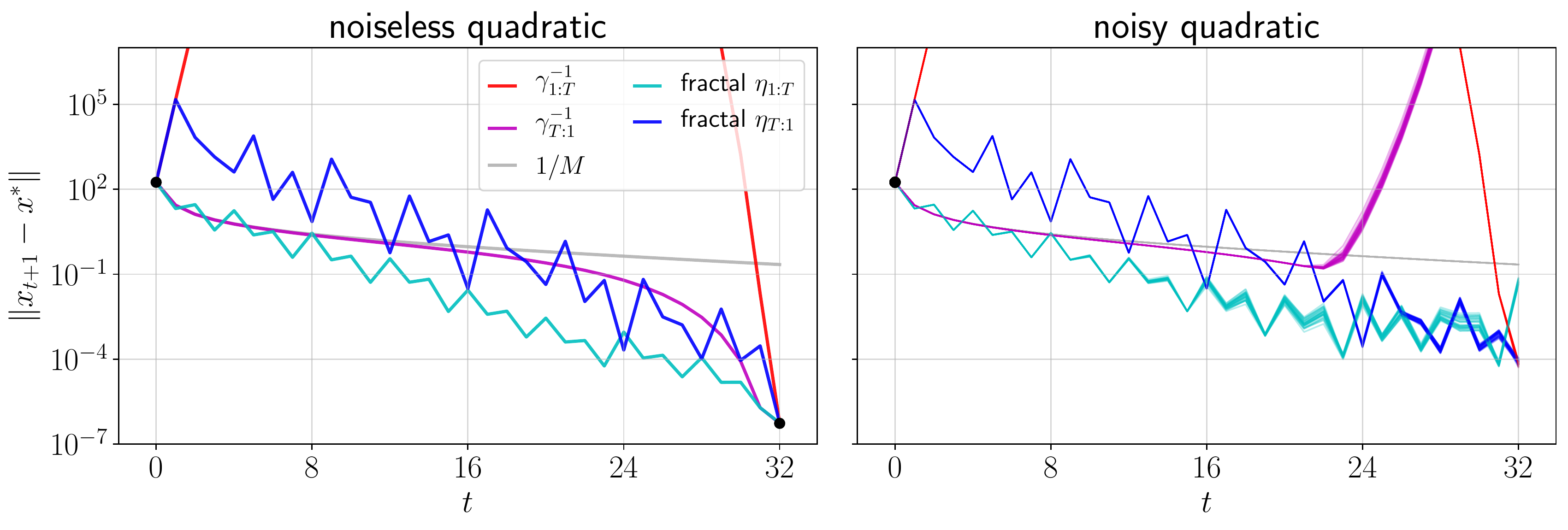}
\else
    \includegraphics[width=\linewidth]{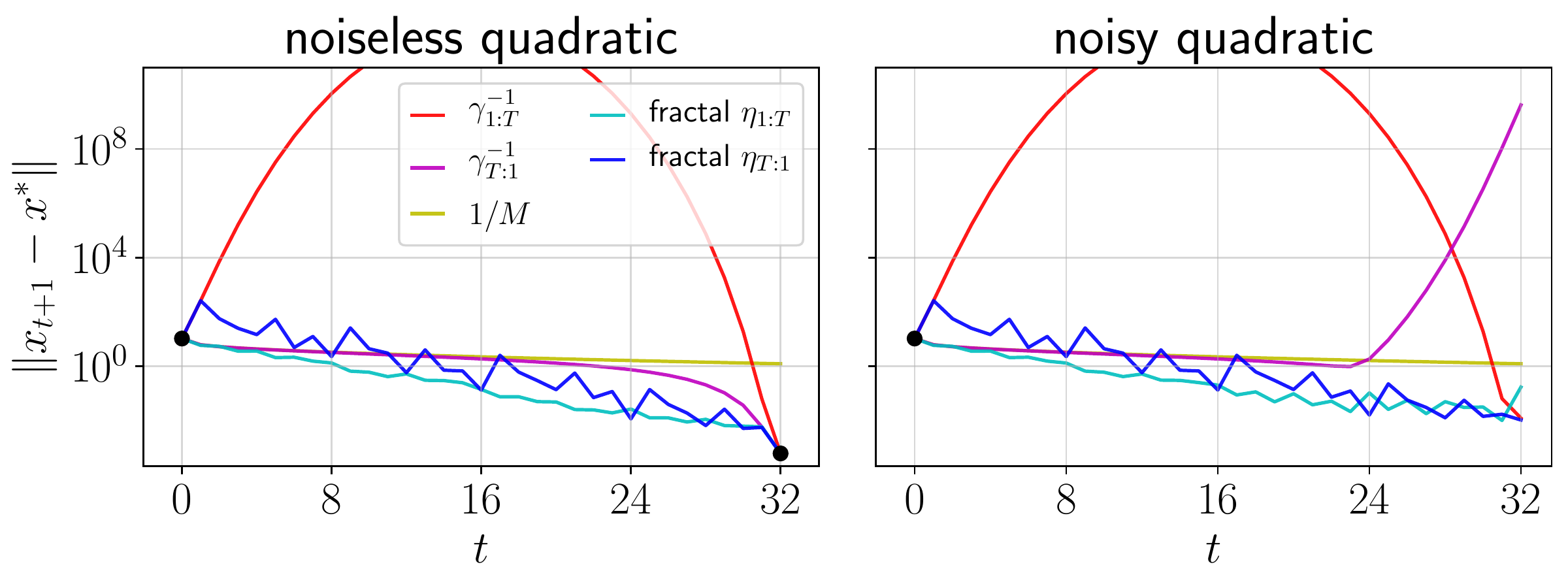}
    \vspace{-5mm}
\fi
    \caption{The optimization trajectories of various permutations of the Chebyshev step sizes. \emph{Left:} In the noiseless case, the final iterates coincide, but $x_t$ can wander exponentially far away. \emph{Right:} With (i.i.d. Gaussian) noise, there is a tradeoff between $\norm{x_t}$ and the stability of $x_{\out}$.}
    \label{fig:perm-stability}
\end{figure}

We motivate our first new results using an additive noise model; this is a refinement of \cite{lebedev1971order,lebedev1973solution,lebedev1976utilization}, which are only concerned with preventing exponential blowup of negligible perturbations at the numerical noise floor. We consider adding a sequence of perturbations $(\xi_1, \ldots, \xi_T)$ to gradient descent (Equation~\ref{eq:gd}):
\begin{equation}
\label{eq:perturbed-gd}
\{ x_{t+1} \leftarrow x_t - \eta_t \nabla f(x_t) + \xi_t \}_{t=1}^{T}.
\end{equation}
Note that this captures an inexact (e.g. stochastic) gradient oracle $\widetilde{\nabla f}(\cdot)$, in which case
\begin{equation}
\label{eq:gradient-noise}
\xi_t = \eta_t( \nabla f(x_t) - \widetilde{\nabla f}(x_t) ).
\end{equation}
Unrolling the recursion, we get:
\begin{gather*}
    x_2 - x^* = (I - \eta_1 A) (x_1 - x^*) + \xi_1, \\
    x_3 - x^* = (I - \eta_2 A)\bra{ (I - \eta_1 A)(x_1 - x^*) + \xi_1 } + \xi_2, \\
    \cdots \\
    x_t - x^* = p_{1:t-1}(A)(x_1 - x^*) + \sum_{t'=2}^{t} p_{t':t-1}(A) \xi_{t' - 1},
\end{gather*}
where we have defined the \emph{infix polynomial} as the (possibly empty) product
\[p_{s:t}(A) := \prod_{\tau=s}^t (I - \eta_\tau A).\]

\citet{lebedev1971order} give bounds on the norms of the \emph{prefix polynomials} $p_{1:t}$ and \emph{suffix polynomials} $p_{s:T}$:
\begin{theorem}[Prefix and suffix bounds]
\label{thm:lebedev-main}
For a fractal Chebyshev schedule with $m,M,T$, and all $1 \leq s \leq t \leq T$:
\begin{enumerate}
    \item[(i)] $\norm{p_{1:t}}_{[m,M]} \leq \frac{ \kappahat - 1 }{ 4^{\min(\bits(t))} } \prod_{j \in \bits'(t)} \frac{2}{1+\Tcheb_{2^j}(\theta)}$;
    \item[(ii)] $\norm{p_{s:T}}_{[m,M]} \leq \prod_{j \in \bits(T+1-s)} \frac{2}{1+\Tcheb_{2^j}(\theta)}$,
\end{enumerate}
where $\bits(n)$ denotes the sequence $j_1 > j_2 > \ldots > j_k$ of indices in the binary expansion of $n$, and $\bits'(n) := \bits(n) \setminus j_k $. For example, when $n = 6 = 2^2 + 2^1$, $\bits(n) = \{2,1\}$, and $\bits'(n) = \{2\}$.
\end{theorem}
Let $\Vcal(\cdot), \Vcal'(\cdot)$ denote the bounds from Theorem~\ref{thm:lebedev-main},
so that $\norm{p_{1:t}}_{[m,M]} \leq \Vcal'(t)$, and $\norm{p_{s:T}}_{[m,M]} \leq \Vcal(T+1-s)$. Notice that $\Vcal(t) \leq \frac{2}{1+\Tcheb_{\lfloor t/2 \rfloor}(\theta)} \leq e^{-\Omega(t)/\sqrt{\kappahat}}$ for all $t \geq 1$,
and $\Vcal'(t) \leq \kappahat \Vcal(t)$.

To fully understand the propagation of $\xi_t$ through Equation~\ref{eq:perturbed-gd}, we provide bounds on the infix polynomial norms:
\begin{theorem}[Infix polynomial bounds]
\label{thm:infix-bound}
For the fractal Chebyshev schedule with $m,M,T$, and all $1 \leq s \leq t \leq T$:
\[ \norm{p_{s:t}}_{[m,M]} \leq \Vcal(\zeta+1-s) \cdot \Vcal'(t-\zeta), \]
where $\zeta$ is the index such that $s-1 \leq \zeta \leq t$ and $\zeta, \zeta+1$ differ at the most significant bit.
\end{theorem}

Then, analyzing the decay of $\Vcal, \Vcal'$, we derive cumulative error bounds:
\begin{theorem}[Infix series bounds]
\label{thm:infix-series-bound}
For a fractal Chebyshev schedule with $m,M,T$, and all $1 \leq s \leq t \leq T$:
\[ \sum_{t'=s}^t \norm{p_{t':t}}_{[m,M]} \leq O(\widehat{\kappa}^{1 + \frac{1}{\ln 4}} \log \widehat{\kappa}) = o\pa{ \widehat{\kappa}^{1.73} }. \]
This bound, a sum of up to $T$ terms, is independent of $T$.
\end{theorem}

These require generalizations of the combinatorial proofs for Theorem~\ref{thm:lebedev-main},
presented (along with more precise statements) in Appendices~\ref{subsec:appendix-infix-proof} and \ref{subsec:appendix-infix-series-proof}.

\subsection{Implications for gradient descent}
Theorem~\ref{thm:infix-series-bound} translates to the following end-to-end statement about gradient descent with the fractal schedule:
\begin{corollary}
\label{thm:cheb-noise-stability}
Suppose $0 < m \leq \lambda_{\min} \leq \lambda_{\max} \leq M$. Then, gradient descent 
with the fractal Chebyshev schedule of length $T$, and perturbations (as in Equation~\ref{eq:perturbed-gd}) such that $\norm{\xi_t} \leq \eps$, outputs iterates $x_t$ satisfying
\[ \norm{x_{t+1} - x^*} \leq \norm{p_{1:t}}_{[m,M]} \cdot \norm{x_1 - x^*} + o(\widehat\kappa^{1.73}) \cdot \eps. \]
Recall that Theorems~\ref{thm:lebedev-main} and \ref{thm:cheb-convergence-rate} guarantee
\ifdefined\arxiv 
    \[ \norm{p_{1:t}}_{[m,M]} \leq e^{-\Omega(T) \cdot \log(\kappahat)/\sqrt{\kappahat}} \; ; \qquad \norm{p_{1:T}}_{[m,M]} \leq e^{-\Omega(T)/\sqrt{\kappahat}}. \]
\else
    \begin{align*} \norm{p_{1:t}}_{[m,M]} &\leq e^{-\Omega(T) \cdot \log(\kappahat)/\sqrt{\kappahat}} \; ; \\
    \norm{p_{1:T}}_{[m,M]} &\leq e^{-\Omega(T)/\sqrt{\kappahat}}.
    \end{align*}
\fi
\end{corollary}
The fractal schedule allows the stability factor to be independent of $T$. When the perturbations arise from noisy gradients (as in Equation~\ref{eq:gradient-noise}), so that each $\xi_t$ is $\eta_t \eps$-bounded, this factor becomes $o(\kappahat^{2.73})$.

\paragraph{Provable benefit of negative progress.}
A striking fact about the fractal Chebyshev schedule is that this \emph{non-adaptive} method provably beats the minimax convergence rate of line search, the most fundamental \emph{adaptive} algorithm in this setting \citep{boyd2004convex}:
\begin{equation}
\label{eq:line-search}
\eta_{t}^{\mathrm{(ls)}} := \argmin_{\eta \geq 0} f(x_t - \eta \nabla f(x_t)).
\end{equation}
\begin{proposition}[No acceleration from line search]
\label{prop:line-search-bad}
On a strongly convex quadratic objective $f(x) = \frac{1}{2} x^\top A x + b^\top x$,
let $\{x_t\}$ be the sequence of iterates of gradient descent with the adaptive learning rate schedule $\eta_{t}^{\mathrm{(ls)}}$ from Equation~\ref{eq:line-search}. Then, for each $A,b$, there exists a setting of $x_1$ such that
\[ \norm{x_{t+1} - x^*} \geq \pa{1-\frac{1}{\Omega(\kappa)}}^T \!\!\!\! \cdot \norm{x_1 - x^*}, \quad \forall t \geq 1. \]
\end{proposition}
This is a classic fact; for a complete treatment, see Section~3.2.2 of \cite{kelley1999iterative}. In the context of our results, it shows that greedily selecting the locally optimal learning rates is provably suboptimal, even compared to a feedback-independent policy.

Adaptive estimation of the local loss curvature is an oft-attempted approach, amounting to finding the best conservative step size $\frac{1}{M}$. Proposition~\ref{prop:line-search-bad} suggests that although such methods have numerous advantages, greedy local methods can miss out on acceleration. The fact that acceleration can be obtained from carefully scheduled overshooting is reminiscent of simulated annealing \cite{aarts1989simulated}, though we could not find any rigorous connections.

\paragraph{Comparison with momentum.} We stress that this form of acceleration does not replace or dominate momentum. The dependence of the stability term on $\kappahat$ is suboptimal \citep{devolder2014first}. In exchange, we get a \emph{memoryless} acceleration algorithm: gradient descent has no auxiliary variables or multi-term recurrences, so that $x_t$ fully specifies the state. This bypasses the subtleties inherent in restarting stateful optimizers \citep{o2015adaptive,loshchilov2016sgdr}.

Finally, our theory (especially Theorem~\ref{thm:underoverstepping}) implies that experiments attempting to probe the acceleration benefits of momentum might be confounded by the learning rate schedule, even in the simplest of settings (thus, certainly also in more complicated settings, like deep learning).

\subsection{Brief overview of proof ideas}
Figure~\ref{fig:perm-stability} suggests that there is a tradeoff between taking large $\Omega(1/m)$ steps for acceleration vs. small $O(1/M)$ steps for stability. To get acceleration, we must take all of the large steps in the schedule. However, we must space them out: taking $k = o(T)$ of the largest steps consecutively incurs an exponential blowup in the infix polynomial:
\begin{align*}
\prod_{i=1}^k \norm{ \pa{1 - \frac{\lambda}{\gamma_i}} }_{[m,M]} \!\!\!\! \approx \norm{ \pa{1 - \frac{\lambda}{m}}^k }_{[m,M]} \!\!\!\! = \pa{\kappahat - 1}^k\!\!.
\end{align*}

The difficulty arises from the fact that there are not enough small steps in the schedule, so that a large step will need to be stabilized by \emph{internal copies of Chebyshev iteration}. This is why the fractal schedule is necessary. Theorem~\ref{thm:infix-bound} shows that this is surprisingly possible: the fractal schedule is only as unstable as the largest single step.

This intuition does not get us very far towards an actual proof: the internal copies of Chebyshev iteration, which form a complete binary tree, are ``skewed'' in a way that is sometimes better, sometimes worse. Isolating a combinatorial \emph{tree exchange lemma} used to prove Theorem~\ref{thm:lebedev-main}, we can iteratively swap two special infix polynomials with two others, and localize ``bad skewness'' to only one large step. Theorem~\ref{thm:infix-bound} follows from decomposing each infix into two infixes amenable to the tree exchange procedure. Theorem~\ref{thm:infix-series-bound} follows by combining Theorem~\ref{thm:infix-bound} with sharpened generalizations of the original paper's series bounds.

The proofs involve delicate trigonometric inequalities and various interesting facts about the geometry of polynomials. Appendices~\ref{sec:appendix-cheb-background}, \ref{sec:appendix-lebedev}, and \ref{sec:appendix-proofs} build up to self-contained proofs.
\section{Extensions and variants}
\label{sec:extensions}

Next, we explore some theoretically justified variants.

\subsection{Useful transformations of the fractal schedule}

\paragraph{Reversing the schedule.} Notice that the first step $\eta_1$ is the largest step in the schedule. This might not be desirable when $\xi_t$ is proportional to $\norm{x-x^*}$ (like in linear regression with minibatch SGD noise). It is a simple consequence of the symmetries in the main theorems that reversing the fractal Chebyshev schedule produces a contractive variant:

\begin{proposition}
\label{prop:cheb-sched-reverse}
Suppose we run gradient descent with the reversed fractal Chebyshev schedule
$\sigma_{T}(T+1-t)$. Then:
\begin{enumerate}
    \item[(i)] For any $1 \leq t < t' \leq T$, we have
    \[ \overline{\norm{p_{1:t}}}_{[m,M]} \leq \overline{\norm{p_{1:t'}}}_{[m,M]} \leq 1, \]
    where $\overline{\norm{\cdot}}$ denotes the corresponding suffix norm bound from Theorem~\ref{thm:lebedev-main} (ii).
    \item[(ii)] The bounds from Theorem~\ref{thm:lebedev-main} are swapped: replace $(p_{1:t},p_{s:T}) \rightarrow (p_{T+1-t:T}, p_{1:T+1-s})$.
    \item[(iii)] Theorem~\ref{thm:infix-bound} holds, swapping $\Vcal \leftrightarrow \Vcal'$. Theorem~\ref{thm:infix-series-bound} holds.
\end{enumerate}
\end{proposition}

\paragraph{Concatenating schedules.} One can also repeat the fractal Chebyshev schedule indefinitely.\footnote{This is known as a cyclic iterative method, and was in fact the original motivation for \citep{lebedev1971order}.} Note that each infix polynomial of a repeated schedule can be written as a product of one prefix $p_{1:t}$, one suffix $p_{s:T}$, and a power of $p_{1:T}$, so stability bounds analogous to Theorems~\ref{thm:infix-bound} and \ref{thm:infix-series-bound} follow straightforwardly. It is also possible to concatenate schedules with different lengths $T$. Choosing $T$ to be successive powers of 2, one obtains an infinitely long schedule suitable for unknown time horizons.

\subsection{Conservative overstepping and partial acceleration}
\label{subsec:underoverstepping}

In this section, we decouple the eigenvalue range $[\lambda_{\min}, \lambda_{\max}]$ from the Chebyshev node range $[m,M]$ used in constructing the schedule. This can simply arise from an incorrect estimation of the eigenvalue range. However, more interestingly, if we think of $[m, M]$ as purposefully omitting the lower spectrum of $A$ (and thus taking smaller large steps), this allows us to interpolate between the fractal Chebyshev schedule and the vanilla constant learning rate.

\paragraph{Easy cases.} If $m < \lambda_{\min}$ or $M > \lambda_{\max}$, then $[m,M]$ is still an interval containing the spectrum of $A$; it is simply the case that convergence rates and stability bounds will depend on a worse $\kappahat > \kappa$. On the other hand, if $M < \lambda_{\max}$, the residual blows up exponentially.

The subtle case is when $m > \lambda_{\min}$, when we are overstepping with restraint, trading off acceleration for stability via more conservative step sizes. This requires us to reason about $\norm{p}_{[\lambda_{\min}, M]}$ when $p$ was constructed to shrink $\norm{p}_{[m, M]}$. Analyzing this case, we get \emph{partial} acceleration:

\begin{theorem}
\label{thm:underoverstepping}
Given a quadratic objective with matrix $A$ and $0 < \lambda_{\min} \leq m \leq \lambda_{\max} \leq M$, gradient descent with the Chebyshev step sizes results in the following convergence guarantee:
\[
    \|x_{\mathrm{out}} - x^*\| \leq 2\left(1 - \phi^{-1}(\lambda_{\min},m,M) \right)^T \cdot \|x_1 - x^*\|,
\]
with
\ifdefined\arxiv
\[
\phi^{-1}(\lambda_{\min},m,M) :=
    2 \cdot \frac{ \lambda_{\min} + \sqrt{Mm} - \sqrt{(M-\lambda_{\min})(m-\lambda_{\min})}}{ (\sqrt{M} + \sqrt{m})^2 }.
\]
\else
\begin{multline*}
\phi^{-1}(\lambda_{\min},m,M) \\ :=
    2 \cdot \frac{ \lambda_{\min} + \sqrt{Mm} - \sqrt{(M-\lambda_{\min})(m-\lambda_{\min})}}{ (\sqrt{M} + \sqrt{m})^2 }.
\end{multline*}
\fi

\end{theorem}
This is an interpolation between the standard and accelerated convergence rates of $O(\kappa \log(1/\eps))$ and $O(\sqrt{\kappa} \log(1/\eps))$. Figure~\ref{fig:understepping} shows the shape of $\phi$ for $m \in [\lambda_{\min},M]$, as it ranges from $\sim\sqrt{\kappa} \rightarrow \kappa$.

\begin{figure}
    \centering
\ifdefined\arxiv 
    \includegraphics[width=0.8\linewidth]{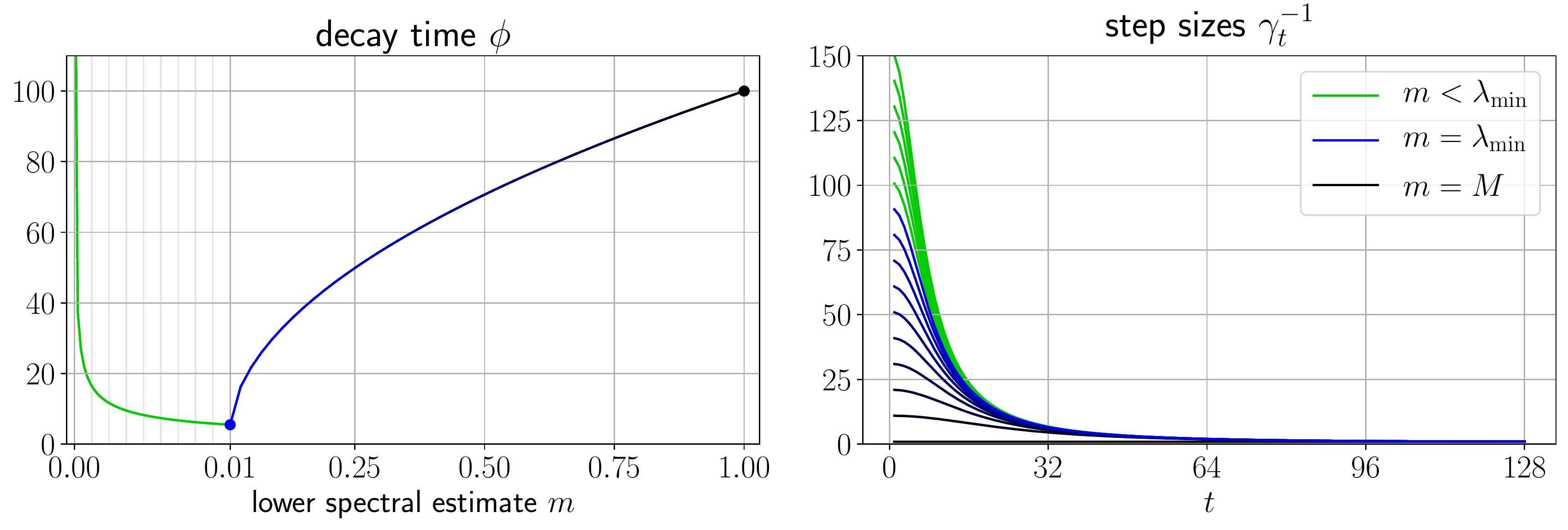}
\else
    \includegraphics[width=\linewidth]{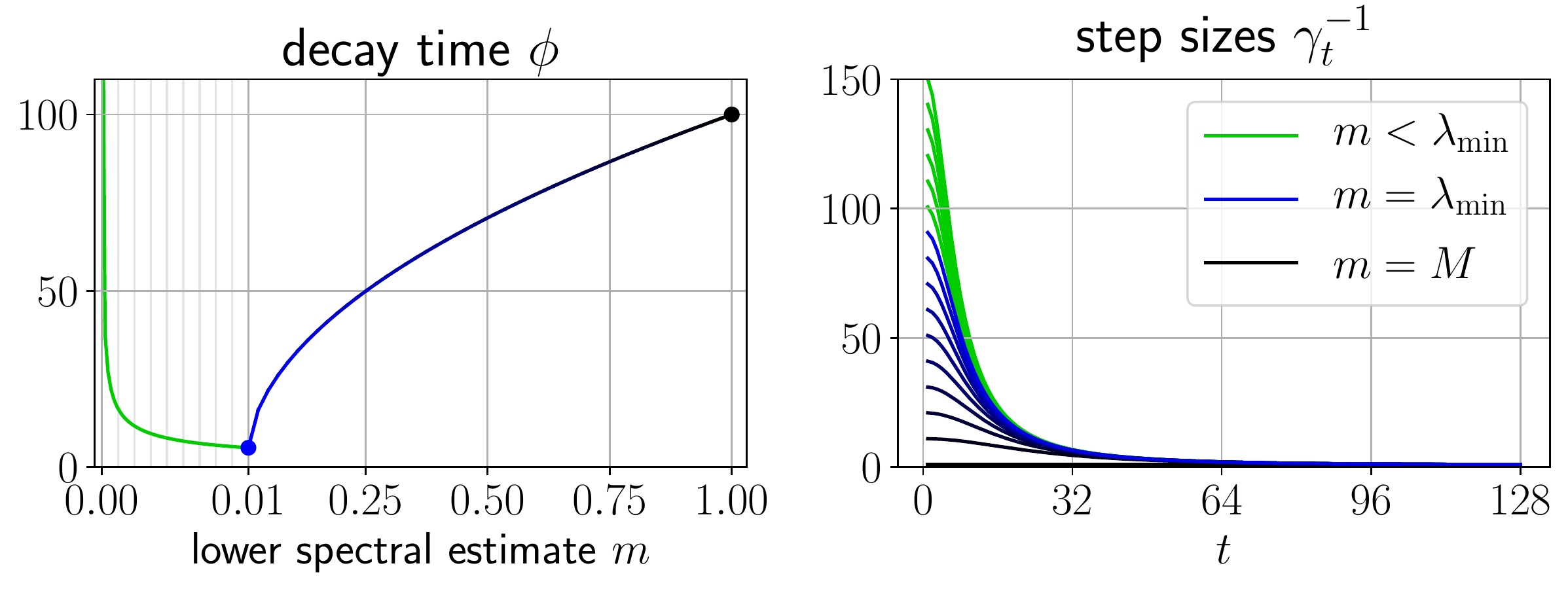}
    \vspace{-5mm}
\fi
    
    \caption{Summary of the discussion in Section~\ref{subsec:underoverstepping}. Suboptimal decay times $\phi(\lambda_{\min}=0.01,m,M=1)$ interpolate between the standard and accelerated rates. Green curves correspond to settings of $m < \lambda_{\min}$ where Theorem~\ref{thm:cheb-convergence-rate} applies; notice the distorted horizontal scale.}
    \label{fig:understepping}
\end{figure}

\subsection{Existence of clairvoyant non-adaptive schedules}
\label{subsec:clairvoyant}

Finally, we present one more view on the provable power of tuning (i.e. searching globally for) a learning rate schedule on a fixed problem instance. An ambitious benchmark is the conjugate gradient method \cite{hestenes1952methods}, which is optimal for \emph{every} (rather than the worst-case) choice of $A, b$. That is, at iteration $t$, it outputs
\[x_{t+1} := \argmin_{\substack{\deg p \leq t \\ p(0) = 1}} \norm{ p(A) (x_1 - x^*) }_A ,\]
where $\norm{x}_A := \sqrt{x^\top A x}$. This can be much stronger than the guarantee from Theorem~\ref{thm:cheb-convergence-rate} (e.g. when the eigenvalues of $A$ are clustered). In Appendix~\ref{subsec:appendix-cg}, we prove that there are non-adaptive (but instance-dependent) learning rate schedules that compete with conjugate gradient:
\begin{theorem}[Conjugate gradient schedule; informal]
\label{thm:cg-sched}
For every problem instance $(A, b)$, there is a learning rate schedule $\{\eta_t\}$ for gradient descent, with each $\eta_t \in [\frac{1}{\lambda_{\max}}, \frac{1}{\lambda_{\min}}]$, such that $x_{\out}$ is the output of conjugate gradient.
\end{theorem}

\section{Beyond convex quadratics}

\subsection{General convex objectives: a counterexample}
\label{subsec:logcosh}
A mysterious fact about acceleration is that some algorithms and analyses transfer from the quadratic case to general convex functions, while others do not. \cite{lessard2016analysis} exhibit a smooth and strongly convex non-quadratic $f$ for which Polyak's momentum gets stuck in a limit cycle.

For us, $f(x) = \log \cosh (x) + 0.01 x^2$ serves as a one-dimensional ``proof by simulation'' that gradient descent with the fractal Chebyshev schedule can fail to converge. This is shown in Appendix~\ref{subsec:appendix-logcosh}; note that this is a tiny instance of ridge logistic regression.

\subsection{Non-convex objectives: a no-go}
None of this theory carries over to worst-case non-convex $f$: the analogue of Theorem~\ref{thm:cg-sched} is vacuously strong. We point out that global optimization of the learning rate schedule is information-theoretically intractable.

\begin{proposition}[Non-convex combination lock; informal]
\label{prop:nonconvex-combination-lock}
For every ``passcode'' $\{\eta_1^*, \ldots, \eta_T^*\}$ and $\delta > 0$, there is a smooth non-convex optimization problem instance $(f(\cdot), x_1)$ for which the final iterate $x_{\out}$ of gradient descent is an $1$-approximate global minimum only if
\[ |\eta_t - \eta^*_t| \leq \delta, \quad \forall t = 1, \ldots, T. \]
\end{proposition}

A formal statement and proof are given in Appendix~\ref{subsec:appendix-nonconvex}.

\subsection{More heuristic building blocks}
\label{subsec:heuristics}

With Polyak momentum as the most illustrious example, an optimizer can be very useful beyond its original theoretical scope. We present some more ideas for heuristic variants (unlike the theoretically justified ones from Section~\ref{sec:extensions}):

\paragraph{Cheap surrogates for the fractal schedule.} The worst-case guarantees for Chebyshev methods depend sensitively on the choice of nodes. However, beyond worst-case objectives, it might suffice to replace $\{\gamma_t^{-1}\}$ with any similarly-shaped distribution (like the triangular one considered by \cite{smith2017cyclical}), and $\sigma$ with any sequence that sufficiently disperses the large steps. We show in Appendix~\ref{subsec:appendix-vanilla-spiky} that acceleration cannot arise from the simple cyclic schedule from \cite{oymak2021super}. An intriguing question is whether adaptive gradient methods or the randomness of SGD implicitly causes partial acceleration, alongside other proposed ``side effect'' mechanisms \cite{keskar2016large,jin2017escape,staib2019escaping}.

\paragraph{Inserting slow steps.}
We can insert any number of steps $\eta \in [0, \frac{2}{M}]$ at any point in a schedule without worsening stability or convergence, because $\norm{(1-\eta \lambda)}_{[m,M]} \leq 1$. That is, $\norm{p_{s':t'}}$ in the supersequence is bounded by the corresponding $\norm{p_{s:t}}$ in the original schedule, and Theorems~\ref{thm:infix-bound} and \ref{thm:infix-series-bound} apply.
A special case of this is \emph{warmup} or \emph{burn-in}: take any number of small steps at the beginning.

Another option is to insert the small steps cyclically: notice from Propositions~\ref{prop:cheb-basics} (ii) and \ref{prop:cheb-fractal-basis} (i) that the steps $\{\eta_t\}$ come in ``fast-slow'' pairs: an odd step overshoots, and an even step corrects it. This suggests further heuristics, like the following ``Chebyshevian waltz'': in minibatch SGD, run triplets of iterations with step sizes $(\eta_{2t-1}, \eta_{2t}, \frac{1}{M})$.\footnote{In non-GPU-bound regimes \citep{choi2019faster,agarwal2020disentangling} and deep RL, one can sometimes take these steps for free, without causing a time bottleneck.} In theory, this degrades the worst-case convergence rate by a constant factor, but improves stability by a constant factor.

\section{Experiments}

\subsection{Convex problems and non-local progress}
In spite of the simple negative result in Section~\ref{subsec:logcosh}, we find that the fractal Chebyshev schedule can exhibit accelerated convergence beyond quadratic objectives. Figure~\ref{fig:convex-preview} shows training curves for logistic regression for MNIST classification; details are in Appendix~\ref{subsec:appendix-convex-experiments}. We leave a theoretical characterization of the schedule's acceleration properties on general convex functions to future work; this may require further assumptions on ``natural'' problem instances beyond minimax bounds.

\begin{figure}
    \centering
\ifdefined\arxiv 
    \includegraphics[width=0.8\linewidth]{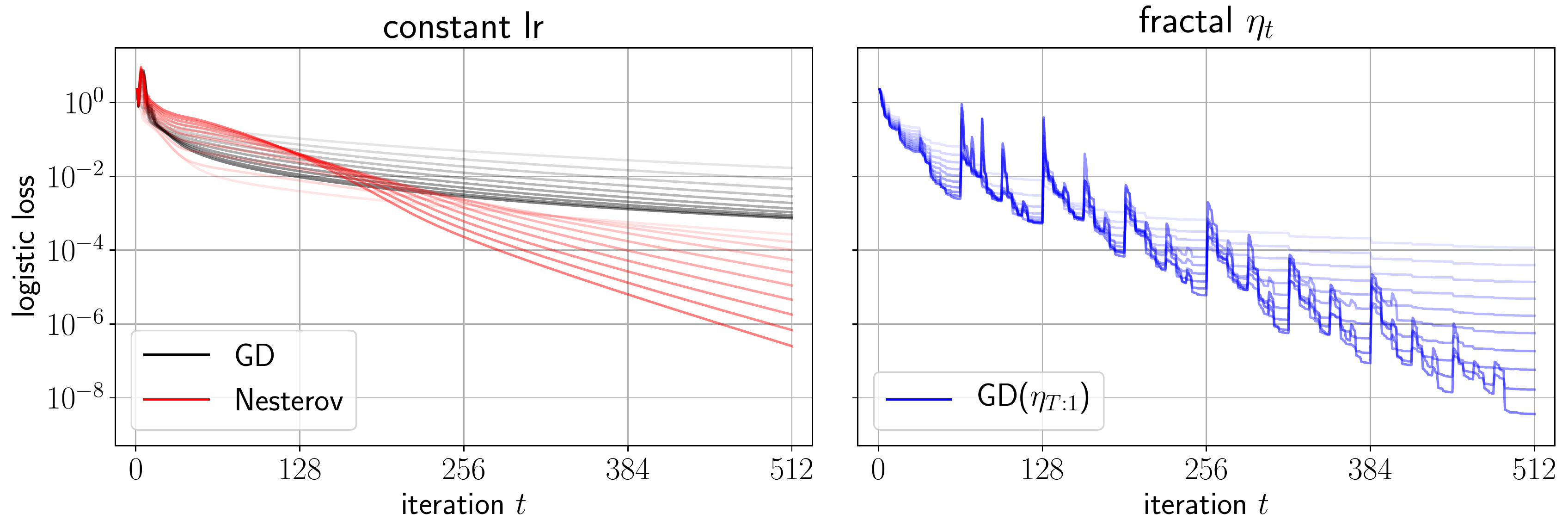}
\else
    \includegraphics[width=\linewidth]{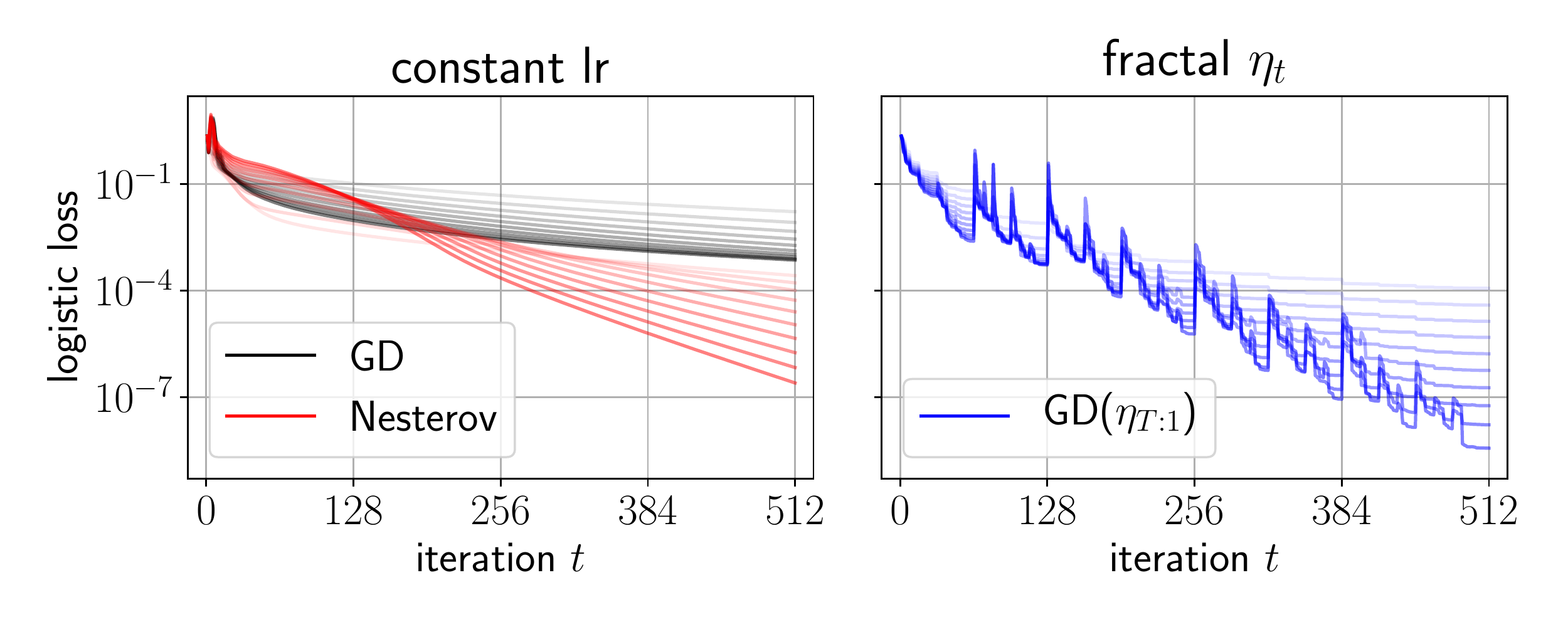}
    \vspace{-5mm}
\fi
    
    \caption{Logistic regression/MNIST training loss curves. \emph{Left:} Standard algorithms, with constant (more opaque = larger) learning rates. \emph{Right:} A fractal Chebyshev schedule.}
    \label{fig:convex-preview}
\end{figure}

\subsection{Beyond the edge of stability in deep learning}
\label{subsec:experiments-dl}

We provide a small set of deep learning experiments, finding that the fractal Chebyshev schedule can overstep the empirical ``edge of stability'' (i.e. the largest constant multiplier on the learning rate for which training does not diverge). Figure~\ref{fig:cifar-preview} gives an overview of these findings; details are in Appendix~\ref{subsec:appendix-deeplearning}.

\begin{figure}
    \centering
\ifdefined\arxiv 
    \includegraphics[width=0.8\linewidth]{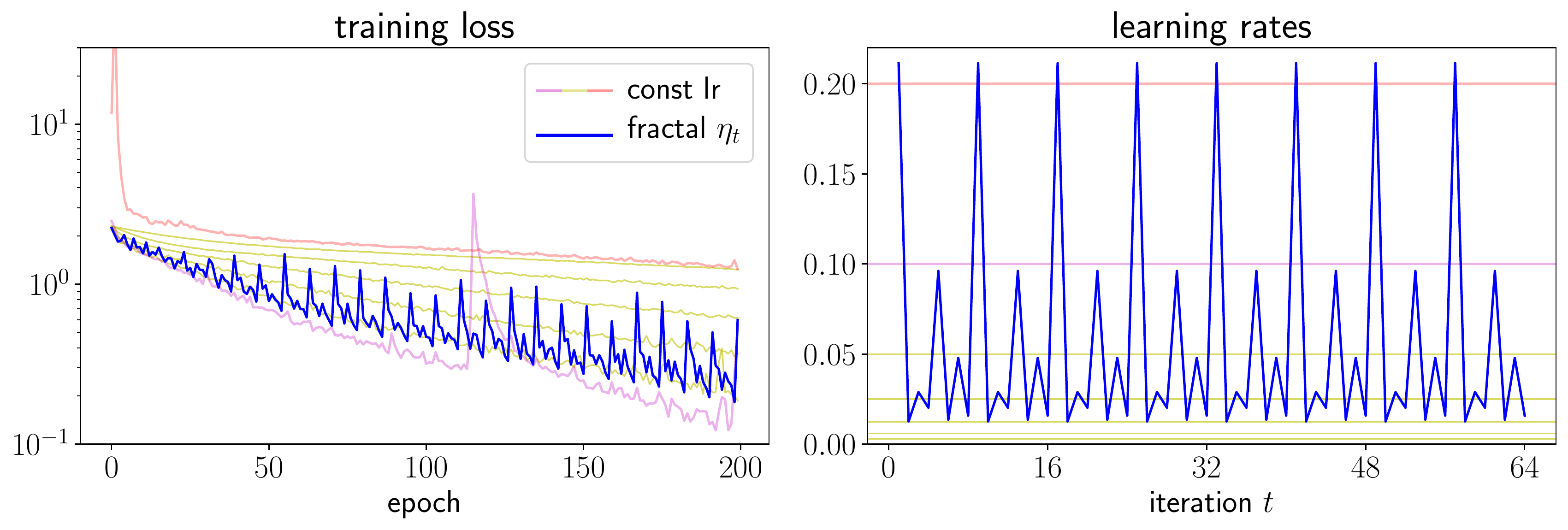}
\else
    \includegraphics[width=\linewidth]{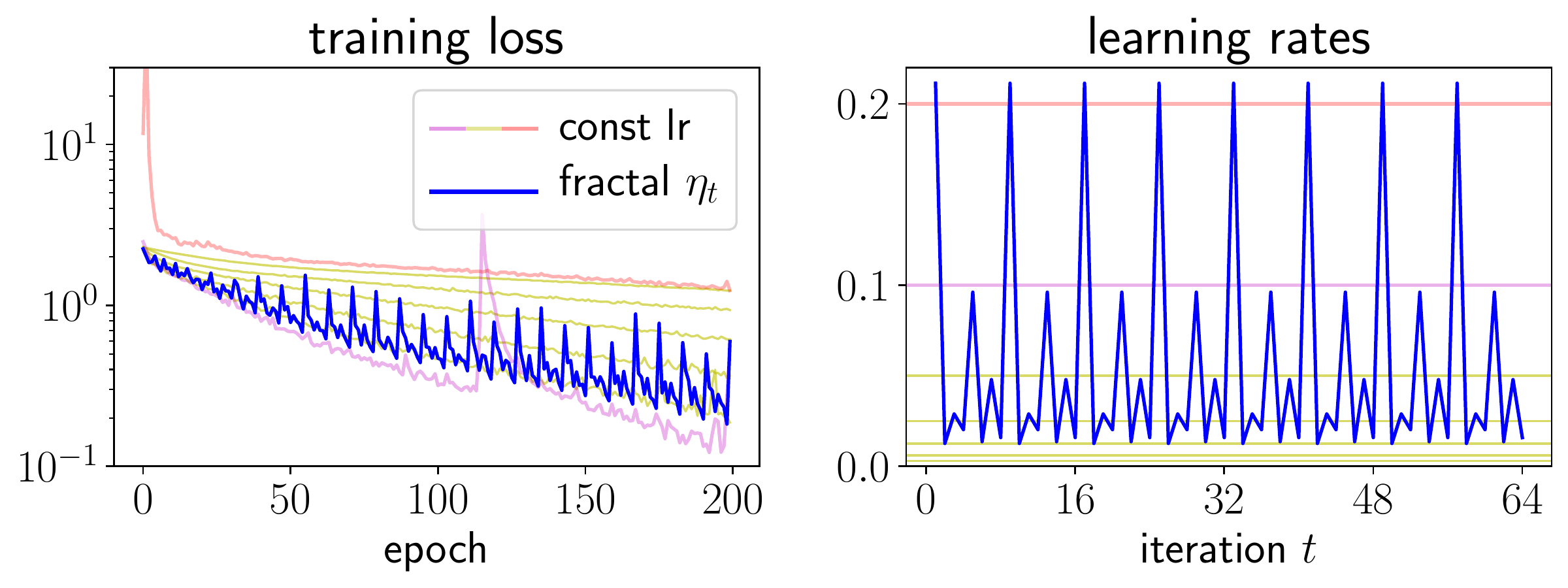}
    \vspace{-5mm}
\fi
    
    \caption{ResNet-18/CIFAR-10 training with batch size $8192$ and a repeated $T=8$ fractal Chebyshev schedule. \emph{Left:} Training loss curves. \emph{Right:} Learning rates; the schedule pokes through the edge of stability (magenta and red) without destabilizing training.}
    \label{fig:cifar-preview}
\end{figure}

Estimating the scale of $\lambda_{\max}(\nabla^2 f)$ is an old paradigm for selecting learning rates \cite{lecun1993automatic,schaul2013no}; there are many proposed mechanisms for the success of larger learning rates. Our theory (especially Theorem~\ref{thm:underoverstepping}) and experiments point to the possibility of \emph{time-varying} schedules to enable larger learning rates, on a much finer scale than cyclic restarts \cite{loshchilov2016sgdr,smith2017cyclical,fu2019cyclical}.
A nascent line of work also challenges the classical $\eta_t \sim 1/\lambda_{\max}$ wisdom from an empirical angle \citep{cohen2020gd}, finding a phenomenon dubbed \emph{progressive sharpening} during normal (smooth $\eta_t$) training.

End-to-end improvements on training benchmarks are outside the scope of this work: the learning rate schedule interacts with generalization \citep{jiang2020characterizing}, batch normalization + weight decay \citep{li2019exponential}, batch size \citep{smith2018don}, adaptive preconditioners \citep{agarwal2020disentangling} and now (from this work) acceleration. This adds yet one more perspective on why it is so difficult to standardize experimental controls and ablations in this space. Analogously, it has been proposed that momentum acts as a variance reduction mechanism \citep{li2017stochastic,cutkosky2019momentum}, alongside its classical role in acceleration.

As an invitation to try these ideas in various experimental settings, we provide in Appendix~\ref{sec:appendix-code} some Python code to generate Chebyshev learning rates and fractal schedules.
\section{Conclusion}
We have revisited a lesser-known acceleration algorithm which uses a fractal learning rate schedule of reciprocal Chebyshev nodes, proved a stronger stability guarantee for its iterates, and developed some practical variants. Our experiments demonstrate promising empirical behaviors of the schedule beyond low-noise quadratics. We hope that this work provides new foundations towards investigating local optimization algorithms which take carefully scheduled ``leaps of faith''.

\paragraph{Open questions.} We conclude with some natural follow-up questions for future work:
\begin{itemize}
    \item Find ``reasonable''\footnote{One example which is unreasonable in every way: run conjugate gradient ahead of time, maintaining monomial-basis expansions of the $A$-orthogonal basis. Compute the roots of the final polynomial, and use their inverses as a learning rate schedule.} (computationally efficient, oracle-efficient, and perturbation-stable) adaptive learning rate schedulers with accelerated convergence rates. What are the acceleration properties of commonly-used adaptive step size heuristics \citep{duchi2011adaptive,kingma2014adam,ward2019adagrad}?
    \item Do there exist learning rate schedules (adaptive or non-adaptive) which obtain the accelerated rate for general strongly convex $f$, as opposed to only quadratics?
\end{itemize}

\section*{Acknowledgments}
We are grateful to Sham Kakade for helpful discussions and pointers to prior literature.
Special thanks go to Maria Ratskevich for helping with the translation of \cite{lebedev1971order}.

\bibliography{bib}
\bibliographystyle{alpha}

\newpage
\onecolumn
\appendix

\section{Code snippets}
\label{sec:appendix-code}

Below, we provide some Python code to compute the Chebyshev step sizes $\{1/\gamma_t\}$, and the permutation $\sigma_T$ that generates the fractal Chebyshev schedule $\{\eta_t\}$.

\begin{minted}{python}
import numpy as np

def cheb_steps(m, M, T):
    C, R = (M+m)/2., (M-m)/2.
    thetas = (np.arange(T)+0.5)/T * np.pi
    return 1./(C - R*np.cos(thetas))

def cheb_perm(T):
    perm = np.array([0])
    while len(perm) < T:
        perm = np.vstack([perm, 2*len(perm)-1-perm]).T.flatten()
    return perm

steps = cheb_steps(0.1, 1, 8)  # [9.20, 5.69 ... 1.01]
perm = cheb_perm(8)            # [0, 7, 3, 4, 1, 6, 2, 5]
steps[perm]                    # [9.20, 1.01 ... 1.25]
\end{minted}
\section{Notation and background on Chebyshev polynomials}
\label{sec:appendix-cheb-background}

First, we gather the notation and classic results on Chebyshev polynomials that will be useful for the proofs in Appendices~\ref{sec:appendix-lebedev}, \ref{sec:appendix-proofs}, and \ref{sec:appendix-proofs-misc}.

\subsection{Definitions}
To review the notation in Section~\ref{subsec:prelims-cheb}, for given values of $m, M, T$, we have defined the following:
\begin{itemize}
    \item The shifted and scaled Chebyshev polynomial construction: set \[ p(\lambda) := \frac{ \Tcheb_T\pa{ z } }{ \Tcheb_T(\theta) } \qquad\text{ where } z := \frac{M+m-2\lambda}{M-m}, \quad \theta := \frac{M+m}{M-m} = 1 + \frac{2m}{M-m}.\]
    We will keep using the auxiliary notation from above, which is useful in switching between different coordinate systems using the bijection $\lambda \leftrightarrow z$ which allows us to switch between the horizontal scales of $p(\lambda)$ and its corresponding $\Tcheb(z)$. This bijection maps $\lambda \in [m, M]$ to $z \in [-1, 1]$.
    Note that $\theta$ is the value of  $z$ corresponding to applying the above bijection to $\lambda=0$.
    \item A characterization of $p$ by its roots: \[ \gamma_t := \frac{M+m}{2} - \frac{M-m}{2} \cos \frac{(t-\frac{1}{2})\pi}{T}, \quad t = 1, \ldots, T. \]
\end{itemize}

If $m = \lambda_{\min}(A)$ and $M = \lambda_{\max}(A)$, then $M/m$ coincides with $\kappa := \lambda_{\max}(A)/\lambda_{\min}(A)$, the condition number of the matrix $A$. Then, we can think of $\theta = 1 + \frac{2}{\kappa-1}$.

Next, we review some well-known facts about Chebyshev polynomials beginning with the definition.
\begin{definition}[Chebyshev polynomials \cite{chebyshev1853theorie}]
For each $n \geq 0$, the Chebyshev polynomials of the first kind $\Tcheb_n(z)$ are defined by the following recurrence:
\begin{itemize}
    \item $\Tcheb_0(z) = 1$.
    \item $\Tcheb_1(z) = z$.
    \item $\Tcheb_n(z) = 2z T_{n-1}(z) - T_{n-2}(z)$, for all $n \geq 2$.
\end{itemize}
\end{definition}
The equivalence of the above with the alternate definition ($\Tcheb_n(z) = \cos (n \arccos z)$ for $|z| \leq 1$) follows by verifying the base cases and applying the cosine sum-of-angles formula.

\subsection{Basic lemmas}
We gather some basic lemmas used to prove the results in \cite{lebedev1971order}, which are classic.

\begin{lemma}[Alternative characterizations of the Chebyshev polynomials]
\label{lem:cheb-alt}
The following are true for all non-negative integers $n$ and $|z| \geq 1$:
\begin{enumerate}
    \item[(i)] $T_n(z) = \pm \cosh( n \; \acosh(z) ).$ The sign is $+1$ when $z$ is positive, and $(-1)^n$ when $z$ is negative.
    \item[(ii)] We have
    \[\Tcheb(z) = \frac{\pa{z - \sqrt{z^2 - 1}}^n + \pa{z + \sqrt{z^2 - 1}}^n}{2}.\]
\end{enumerate}
\end{lemma}
\begin{proof}
(i) follows from recursively applying the identities $\cosh(a+b) = \sinh(a)\sinh(b) + \cosh(a)\cosh(b)$ and $\sinh^2(a) = \cosh^2(a) - 1$. (ii) follows from (i), performing the hyperbolic substitution $z = \cosh(a)$, noticing that all odd-powered $\sinh(a)$ terms cancel, and verifying the base cases and recurrence relation.
\end{proof}

Combining the $\cos$ and $\cosh$ characterizations of $\Tcheb_n(z)$, we obtain the composition property: for all integers $k, n \geq 0$ and all $z \in \R$,
\[ \Tcheb_{kn}(z) = \Tcheb_k( \Tcheb_n(z) ). \]

The half-angle cosine formulas will be useful: for any $\alpha, z \in \R$, and positive even $n$,
\[ \cos(\theta) \; = \; 2\cos^2(\theta/2) - 1 \; = \; 1 - 2\sin^2(\theta/2); \qquad \Tcheb_n(z) = 2\Tcheb^2_{n/2}(z) - 1. \]
The third statement is true for all $z$, since it is the composition property with $k=2$.

The key reason why we are interested in the Chebyshev polynomials is their \emph{extremal} property: outside the range of their roots, they expand faster than any other polynomial of the same degree. There are various ways to formalize this. We will only need the following:
\begin{lemma}[Expansion lower bounds]
\label{lem:cheb-expand}
For all $\delta \geq 0$, each $\Tcheb_n$ satisfies the following:
\begin{enumerate}
\item[(i)] $\Tcheb_n(1+\delta) = \frac{(1 + \delta + \sqrt{2\delta+ \delta^2})^{2n} + 1}{ 2(1 + \delta + \sqrt{2\delta+ \delta^2})^n } \geq \frac{(1 + \sqrt{2\delta})^n}{2} .$
    \item[(ii)] $\Tcheb_n(1+\delta) \geq 1 + n^2 \delta.$
\end{enumerate}
\end{lemma}
\begin{proof}
To prove (i) note that, using Lemma \ref{lem:cheb-alt}(ii), we have that 
\[\Tcheb_n(1+\delta) = \frac{(1 + \delta + \sqrt{2\delta+ \delta^2})^n + (1 + \delta - \sqrt{2\delta+ \delta^2})^n}{2} = \frac{(1 + \delta + \sqrt{2\delta+ \delta^2})^{2n} + 1}{ 2(1 + \delta + \sqrt{2\delta+ \delta^2})^n }, \]
where the last equality follows by noticing that $(1 + \delta - \sqrt{2\delta+ \delta^2})^{-1} = (1 + \delta + \sqrt{2\delta+ \delta^2})$. 
The inequality in (i) is concluded by noticing that
\[\frac{(1 + \delta + \sqrt{2\delta+ \delta^2})^n + (1 + \delta - \sqrt{2\delta+ \delta^2})^n}{2} \geq \frac{(1 + \sqrt{2\delta})^n}{2},\]
by dropping the positive terms. To conclude (ii), we perform a series expansion upto degree 2 to get,
\begin{align*}
    \frac{(1 + \delta + \sqrt{2\delta+ \delta^2})^n + (1 + \delta - \sqrt{2\delta+ \delta^2})^n}{2} &\geq \frac{2 + 2n\delta + \frac{n(n-1)}{2}(4\delta + 4\delta^2) }{2}\\
    &= 1 + n^2 \delta.
\end{align*}


\end{proof}

\paragraph{Classic convergence rate of Chebyshev iteration.} Though Theorem~\ref{thm:cheb-convergence-rate} is classic, the exact statement of the convergence rate has several variants. For sake of completeness, we give a quick proof of the convergence rate of Chebyshev iteration implied by the exact formula in Lemma~\ref{lem:cheb-expand} (i): 

\begin{manualtheorem}{\ref{thm:cheb-convergence-rate}}
Choose spectral estimates $m \leq M$ such that $0 < m \leq \lambda_{\min} \leq \lambda_{\max} \leq M$. Then, setting $\{\eta_t\}$ to be any permutation of $\{1/\gamma_t\}$, the final iterate of gradient descent $x_{\out}$ satisfies the following:
\begin{align*}
\norm{x_{\out} - x^*} &\leq \frac{2\rho^T}{1+\rho^{2T}} \norm{x_1 - x^*} \leq e^{-\Omega(T)/\sqrt{\kappahat}} \norm{x_1 - x^*},
\end{align*}
where $\rho := \frac{ \sqrt{M} - \sqrt{m} }{ \sqrt{M} + \sqrt{m} } \leq 1 - \Omega\pa{\frac{1}{\sqrt{\kappahat}}}$.
\end{manualtheorem}

\begin{proof}
First, assume $m<M$. Then, we have
\[ \|x_{\mathrm{out}} - x^*\| \leq \|p\|_{[\lambda_{\min}, \lambda_{\max}]} \cdot \|x_{\mathrm{out}} - x^*\| \leq \|p\|_{[m,M]} \cdot \|x_{\mathrm{out}} - x^*\|. \]
So, we need to bound $\|p\|_{[m,M]}$.
Setting $\delta = \theta-1 = \frac{2m}{M-m}$, notice that $1+\delta+\sqrt{2\delta + \delta^2} = 1/\rho$. Then, using Lemma~\ref{lem:cheb-expand} (i), we have

\[\|p\|_{[m,M]} = \max_{|z|\leq 1} \frac{\Tcheb_T(z)}{\Tcheb_T(\theta)}
= \frac{1}{\Tcheb_T(\theta)}
= \frac{2\rho^{-T}}{1+\rho^{-2T}}
= \frac{2\rho^{T}}{1+\rho^{2T}},
\]
as desired.
When $m = M$, the inequality is trivially true because $\rho$ and $\norm{p}_{[m,M]}$ are both 0.

\end{proof}
\section{Theorems and proofs from \cite{lebedev1971order}}
\label{sec:appendix-lebedev}

In the hope of bridging old algorithmic ideas from numerical methods with modern optimization for machine learning, we present a self-contained exposition of the results and proofs from \cite{lebedev1971order} used in this paper. This is far from an exact translation from the original Russian-language manuscript, whose exposition is somewhat terse. We provide some more intuitive proofs, fix some small (inconsequential) typos, change some notation to match this paper, isolate lemmas which are useful for proving our other results, and omit some weaker and irrelevant results.

\subsection{Skewed Chebyshev polynomials}

First, we show an equivalent divide-and-conquer root-partitioning construction of the fractal permutation $\sigma_T$. To review, this construction defines $\sigma_1 := [1]$, and for each $T \geq 1$ a power of 2, uses the recurrence
\[\sigma_{2T} := \mathrm{interlace}(\sigma_T, 2T+1-\sigma_T), \]
where
\[\mathrm{interlace}([a_1 \ldots a_n], [b_1 \ldots b_n]) := [a_1 \; b_1 \; a_2 \; b_2 \ldots a_n \; b_n].\]
Some examples are below:
\[\sigma_2 = [1\;2],\]
\[\sigma_4 = [1\;4\;2\;3],\]
\[\sigma_8 = [1\;8\;4\;5\;2\;7\;3\;6],\]
\[\sigma_{16} = [1\;16\;8\;9\;4\;13\;5\;12\;2\;15\;7\;10\;3\;14\;6\;11].\]

Let us formalize the sense in which $\sigma$ contains internal copies of Chebyshev iteration. For positive integers $n$ and $\alpha \in (0, \pi)$, let us define the \emph{skewed Chebyshev polynomials}
\[ \Pcal_{n,\alpha}(z) := \frac{\Tcheb_n(z) - \cos(\alpha)}{\Tcheb_n(\theta) - \cos(\alpha)}, \]
noting that $\Pcal_{T,\frac{\pi}{2}}(z) = p(\lambda)$. If $\alpha \in (0, \frac{\pi}{2}]$, then call $\Pcal_{n,\alpha}$ \goodcolor{good} and if $\alpha \in [\frac{\pi}{2}, \pi)$, then call $\Pcal_{n,\alpha}$ \badcolor{bad}. We use the colours blue and red to highlight good and bad polynomials for clarity in our proofs. Note that $p(\lambda)$ is both good and bad. Next, we note some additional facts:

\begin{lemma}[Properties of skewed Chebyshev polynomials]
\label{lem:cheb-skew}
The following are true for all $n \in \Nbb$ and $\alpha \in (0, \pi)$:
\begin{enumerate}
    \item[(i)] $\Pcal_{n,\alpha}(z)$ has $n$ real roots.
    \item[(ii)] If $\Pcal_{n,\alpha}$ is good, then
    \[ \norm{\Pcal_{n,\alpha}}_{L_\infty([-1,1])} \leq \frac{2}{1 + \Tcheb_n(\theta)}. \]
    \item[(iii)] If $\Pcal_{n, \alpha}$ is bad, then
    \[ \norm{\Pcal_{n,\alpha}}_{L_\infty([-1,1])} \leq \frac{2}{n^2 (\theta - 1)}. \]
    \item[(iv)] If $n$ is even, then for all $-1 \leq z \leq 1$,
    \[ \Pcal_{n,\alpha}(z) = \Pcal_{\frac{n}{2},\frac{\alpha}{2}}(z) \cdot \Pcal_{\frac{n}{2},\pi-\frac{\alpha}{2}}(z). \]
\end{enumerate}
\end{lemma}
\begin{proof}
(i) follows from the fact that $\Tcheb_n(z_i) = (-1)^i$ at $z_i = \arccos(i\pi/n), i = 0, \ldots, n$.

(ii) follows from the fact that $|\Tcheb(z)| \leq 1$, $\Tcheb(\theta) > 1$, and $\cos(\alpha) \leq 0$, so that for $u = 1 + \cos(\alpha) \geq 0$, 
\[ \norm{\Pcal_{n,\alpha}}_{L_\infty([-1,1])} \leq \frac{1 - \cos(\alpha)}{\Tcheb_n(\theta) - \cos(\alpha)} \leq \frac{1 - \cos(\alpha) + u}{\Tcheb_n(\theta) - \cos(\alpha) + u} = \frac{2}{1+\Tcheb_n(\theta)}, \]
where the last inequality uses the mediant inequality.

(iii) is a weaker bound than the above, because we can't use the mediant inequality, as $-\cos(\alpha)$ is negative. Using part (ii) of Lemma~\ref{lem:cheb-expand} and $\cos(\alpha) \leq 1$, we get
\[ \norm{\Pcal_{n,\alpha}}_{L_\infty([-1,1])} \leq \frac{2}{\Tcheb_n(\theta) - \cos(\alpha)} \leq \frac{2}{(1+n^2(\theta - 1))-1}. \]

(iv) follows from half-angle formulas and factorizing differences of squares:
\begin{align*}
\Pcal_{n,\alpha}(z) &= \frac{\Tcheb_n(z) - \cos(\alpha)}{\Tcheb_n(\theta) - \cos(\alpha)} = \frac{(2\Tcheb^2_{n/2}(z)-1) - (2\cos^2(\frac{\alpha}{2})-1)}{(2\Tcheb^2_{n/2}(\theta)-1) - (2\cos^2(\frac{\alpha}{2})-1)} = \frac{\Tcheb^2_{n/2}(z) - \cos^2(\frac{\alpha}{2})}{\Tcheb^2_{n/2}(\theta) - \cos^2(\frac{\alpha}{2})} \\
&= \frac{\Tcheb_{n/2}(z) - \cos(\frac{\alpha}{2})}{\Tcheb_{n/2}(\theta) - \cos(\frac{\alpha}{2})} \cdot \frac{\Tcheb_{n/2}(z) + \cos(\frac{\alpha}{2})}{\Tcheb_{n/2}(\theta) + \cos(\frac{\alpha}{2})}
= \Pcal_{\frac{n}{2},\frac{\alpha}{2}}(z) \cdot \Pcal_{\frac{n}{2},\pi-\frac{\alpha}{2}}(z),
\end{align*}
as claimed.
\end{proof}

\subsection{The fractal schedule splits skewed Chebyshev polynomials} In this section, we connect the skewed polynomials $\Pcal_{n,\alpha}$ to the construction of the fractal permutation $\sigma_T$, obtained via recursive binary splitting. This construction will provide the basis for all the proofs regarding the fractal schedule. The starting point for the construction is Lemma \ref{lem:cheb-skew}, which shows that when $T \geq 2$ is a power of 2,
\[ p(\lambda) = \Pcal_{T,\frac{\pi}{2}}(z) = \Pcal_{\frac{T}{2},\frac{\pi}{4}}(z) \cdot \Pcal_{\frac{T}{2},\frac{3\pi}{4}}(z). \]

The above splitting procedure can be recursively repeated (since $T$ is a power of 2) on the pieces produced till we reach degree 1 polynomials of the form $\Pcal_{1,\alpha}$ for some $\alpha$. Such splitting can easily be visualized via the construction of a complete binary tree of depth $\log_2(T)$ (see Figure \ref{fig:tree_construct}), by associating to every node a polynomial of the form $\Pcal_{n,\alpha}$ and setting its left child to be $\Pcal_{n/2,\alpha/2}$ and right child to be $\Pcal_{n/2,\pi - \alpha/2}$. 
Note that every non-leaf node is a product of its children by Lemma \ref{lem:cheb-skew}. The root node corresponds to the polynomial $p(\lambda)$ and the leaf nodes correspond to one degree polynomials which can be equivalently identified by its root, which are by construction the roots of polynomial $p(\lambda)$. 

The key fact regarding this construction is the following.
\begin{fact}
The fractal schedule corresponds to the ordering of the roots as generated by a pre-order traversal of the tree.
\end{fact}
To see this note that every time a split is made in the tree, a constraint on the pre-order traversal is placed, i.e. the roots of the left child polynomial $\Pcal_{n/2,\alpha/2}$ precede that of the right child polynomial $\Pcal_{n/2,\pi - \alpha/2}$. It can be easily verified that the procedure for generating $\sigma_T$ produces an ordering of the roots $\gamma_t$ satisfying all of these constraints; the corresponding learning rate schedule $\eta_t$ is by definition the fractal Chebyshev schedule for each $T$, a power of 2.

Using the above, it can be seen that every node in the tree also corresponds to a particular infix polynomial $p_{s:t}(x)$ which includes all the roots corresponding to all the leaves underneath the node. 

\begin{figure}[h!]
    \centering
    \begin{tikzpicture}[font=\small, level/.style={sibling distance=40mm/#1}, align=center ]
\node [circle,draw] (z){$\Pcal_{T,\frac{\pi}{2}}$}
  child {node [circle,draw=red] (a){\badcolor{$\Pcal_{\frac{T}{2},\frac{\pi}{4}}$}}
    child {node {$\vdots$}
        child{node[circle,draw=red] (b){\badcolor{$\Pcal_{1,\alpha_1}$}}
            child { node [above=4mm](n ){$\sigma_T(1)$} edge from parent[draw=none]
            } 
        }
        child{ node[circle,draw=blue] (c){\goodcolor{$\Pcal_{1,\alpha_2}$}}
            child { node[above=4mm] (o){$\sigma_T(2)$} edge from parent[draw=none]
            } 
        }
    }
    child {node {$\vdots$}}
  }
 child {node [circle,draw=blue] (g) {\goodcolor{$\Pcal_{\frac{T}{2},\frac{3\pi}{4}}$}}
    child {node {$\vdots$}}
    child {node {$\vdots$}
        child{ node[circle,draw=red, scale=0.8] (d){\badcolor{$\Pcal_{1,\alpha_{T-1}}$}}
            child { node [above=4mm](n ){$\sigma_T(T-1)$} 
        edge from parent[draw=none]}
        }
        child{ node[circle,draw=blue] (e){\goodcolor{$\Pcal_{1,\alpha_{T}}$}}
            child { node [above=4mm](n ){$\sigma_T(T)$} 
        edge from parent[draw=none]
            }
        }
    }
 };
 \path (c) -- (d) node [midway] {$\ldots$};
\end{tikzpicture}
    \caption{The binary tree construction for the decomposition of skewed Chebyshev polynomials. Each non-leaf node corresponds to the product of the polynomials corresponding to its children. The root corresponds to the entire polynomial $p_{1:T}(\lambda)$. Leaf nodes correspond to 1 degree polynomials and the pre-order traversal on the leaves induces the ordering given by $\sigma_T$.}
    \label{fig:tree_construct}
\end{figure}
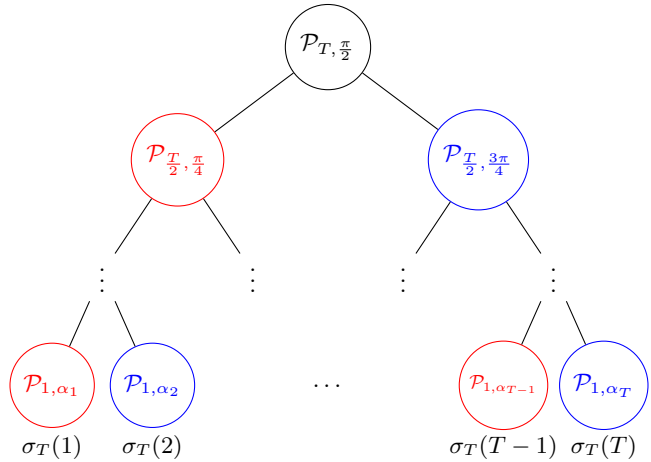

\subsection{Tree partitions and tree exchanges}

In this section we collect some observations regarding the tree construction, which are essential to the proofs for the bounds on the substring polynomials. 

Firstly note that in the binary tree, a node/polynomial is \goodcolor{good} (\badcolor{bad}) if and only if it is the left (right) child of its parent. We begin by analyzing the special case of suffix polynomials $p_{s:T}$ and prefix polynomials $p_{1:s}$ respectively.

\paragraph{Suffix polynomials:}
The following key observation follows from the tree construction. 

\begin{fact}
Every suffix polynomial $p_{s:T}$ (for any $s$) can be written as a product of good polynomials
\end{fact}
To see the above, consider the binary expansion of the number $T+1-s$, $\bits(T+1-s) = \{s_1, s_2 \ldots s_k\}$ defined by the unique decomposition, $T+1-s = 2^{s_1} + 2^{s_2} + \ldots + 2^{s_k}$ such that $s_1 > s_2 > \ldots s_k \ge 0$. We now perform the following iterative decomposition of the  polynomial $p_{s:T}$,
\begin{align*}
p_{s:T} &= p_{s:T_1} \cdot \goodcolor{p_{T_1+1:T}} &\text{where } T_1 := T - 2^{s_1}, \\
&= p_{s:T_2} \cdot \goodcolor{p_{T_2:T_1}} \cdot \goodcolor{p_{T_1+1:T}} &\text{where } T_2 := T_1 - 2^{s_2}, \\
&\ldots,
\end{align*}
until we reach $s:T_k$, which is the empty interval. It can be seen that every intermediate polynomial $p_{T_i + 1:T_{i-1}}$ produced is a good polynomial because each one is the rightmost node at level $s_i$ (i.e. with distance $\log_2 T - s_i$ from the root node), restricted to the subtree rooted at the lowest common ancestor of roots $s$ through $T_{i-1}$ (setting $T_0 := T$). An example of the above decomposition is highlighted in Figure \ref{fig:suffix_decomp}. Combining with statement (ii) in Lemma~\ref{lem:cheb-skew}, we get
\begin{equation}
\norm{p_{s:T}}_{L_\infty([-1,1])} \le \prod_{i=1}^k \frac{2}{1 + \Tcheb_{2^{s_i}}(\theta)}
\end{equation}

\begin{figure}
\captionsetup[subfigure]{justification=centering}
\centering
\begin{subfigure}[b]{0.5\linewidth}
\begin{tikzpicture}
[font=\small, level/.style={sibling distance=40mm/#1}, align=center ]
\node [circle,draw=gray,fill=gray] (z){}
  child {node [circle,draw=red, fill=red] (a){}
    child {node [circle,draw=red, fill=red]{}
        child{node[circle,draw=red, fill=red] (b){}
        }
        child{ node[circle,draw=blue, fill=blue] (c){} node[circle,draw=black, scale=1.7, line width=0.6mm] (c){}
            child { node[above=4mm] (o){$\sigma(2)$} edge from parent[draw=none]
            } 
        }
    }
    child {node [circle,draw=blue, fill=blue]{} node[circle,draw=black, scale=1.7, line width=0.6mm] {}
        child{ node[circle,draw=red, fill=red] (d){}
        }
        child{ node[circle,draw=blue, fill=blue] (e){}
        }
    }
  }
 child {node [circle,draw=blue, fill=blue] (g){} node[circle,draw=black, scale=1.7, line width=0.6mm] {}
    child {node[circle,draw=red, fill=red] {}
        child{ node[circle,draw=red, fill=red] (d){}
        }
        child{ node[circle,draw=blue, fill=blue] (e){}
        }
    }
    child {node[circle,draw=blue, fill=blue] (r){}
        child{ node[circle,draw=red, fill=red] (d){}
        }
        child{ node[circle,draw=blue, fill=blue] (e){}
            child { node [above=4mm] (n){$\sigma(8)$} 
        edge from parent[draw=none]
            }
        }
    }
 };
 \draw[<->](o) to node [above, midway] {$p_{2:8}$}(n);
\end{tikzpicture}
\caption{Suffix decomposition (eg. $p_{2:8}$) into good polynomials}
\label{fig:suffix_decomp}
\end{subfigure}%
~ 
\begin{subfigure}[b]{0.5\linewidth}
\begin{tikzpicture}
[font=\small, level/.style={sibling distance=40mm/#1}, align=center ]
\node [circle,draw=gray,fill=gray] (z){}
  child {node [circle,draw=red, fill=red] (a){} node[circle,draw=black, scale=1.7, line width=0.6mm] {}
    child {node [circle,draw=red, fill=red]{} 
        child{node[circle,draw=red, fill=red] (b){}
            child { node[above=4mm] (o){$\sigma(1)$} edge from parent[draw=none]
                }
        }
        child{ node[circle,draw=blue, fill=blue] (c){}
        }
    }
    child {node [circle,draw=blue, fill=blue]{}
        child{ node[circle,draw=red, fill=red] (d){}
        }
        child{ node[circle,draw=blue, fill=blue] (e){}
        }
    }
  }
 child {node [circle,draw=blue, fill=blue] (g) {}
    child {node[circle,draw=red, fill=red] {} node[circle,draw=black, scale=1.7, line width=0.6mm] {}
        child{ node[circle,draw=red, fill=red] (d){}
        }
        child{ node[circle,draw=blue, fill=blue] (e){}
            child { node [above=4mm] (n){$\sigma(6)$} 
        edge from parent[draw=none]
            }
        }
    }
    child {node[circle,draw=blue, fill=blue] (r){}
        child{ node[circle,draw=red, fill=red] (d){}
        }
        child{ node[circle,draw=blue, fill=blue] (e){}
        }
    }
 };
 \draw[<->](o) to node [above, midway] {$p_{1:6}$}(n);
\end{tikzpicture}
\caption{Prefix decomposition (e.g. $p_{1:6}$) into bad polynomials before exchange}
\label{fig:prefix_decomp}
\end{subfigure}
\end{figure}

\paragraph{Prefix polynomials:}

The prefix polynomials $p_{1:s}$ are more challenging to analyze, as the immediate approach of reversing the above construction gives a decomposition into bad polynomials only. 

To this end, consider the binary expansion of $s = 2^{s_1} + 2^{s_2} + \ldots + 2^{s_k}$ such that $s_1 > s_2 > \ldots s_1 \ge 0$. We decompose $p_{1:s}$ into products in the following manner: starting with $\{1, \ldots, s-1\}$, we iteratively partition

\begin{align*}
p_{1:s} &= \badcolor{p_{1:T_1}} \cdot p_{T_1+1:s} &\text{where } T_1 := 2^{s_1}, \\
&= \badcolor{p_{1:T_1}} \cdot \badcolor{p_{T_1 + 1:T_2}} \cdot p_{T_2+1:s}  &\text{where } T_2 := T_1 + 2^{s_2}, \\
&\ldots,
\end{align*}
until we reach $T_k+1:s$, which is the empty interval. Note that this partition results in all bad polynomials. An example of the above decomposition is highlighted in Figure \ref{fig:prefix_decomp}. We can in fact exactly characterize these polynomials. Define the angle recurrence $\alpha_1 = \frac{2^{s_1}}{T} \cdot \frac{\pi}{2}$ and $\alpha_{i+1} = \frac{\pi - \alpha_i}{2^{s_i - s_{i+1}}}$. It can be seen that
\begin{equation}
\label{eqn:prefixdecomp}
    p_{1:s}= \prod_{i=1}^k \Pcal_{2^{s_i}, \alpha_i}.
\end{equation}

To get a tight bound for the norms of these polynomials, we require another innovation from \cite{lebedev1971order}. This innovation can be captured as a \emph{tree exchange property} that allows us to switch two bad polynomials for a good and bad polynomial. Starting with a partition of the roots of the polynomial we want to analyze, applying this tree switching repeatedly allows us to convert a partition with multiple bad polynomials into a set containing only one bad polynomial (and the rest good).

To further elucidate this tree exchange trick, let us establish some notation. We will be manipulating upper bounds for norms of $\Pcal_{n,a}$, the product of which will serve as an upper bound for a prefix or suffix polynomial. Let
\[ \Bcal_{n, \alpha} := \norm{\Pcal_{n,\alpha}}_{L_\infty([-1,1])} = \max_{z \in [-1, 1]} \abs{ \frac{\Tcheb_n(z) - \cos(\alpha)}{\Tcheb_n(\theta) - \cos(\alpha)} }. \]
Note that the denominator of this fraction is positive independent of $z$. In the numerator, the maximum is achieved when $\Tcheb_n(z)$ is either $+1$ or $-1$, depending on the sign of $\cos(\alpha)$. When $\Pcal_{n,\alpha}$ is good, $\cos(\alpha) \le 0$, thus we have
\[
\Bcal_{n, \alpha} = \frac{1 - \cos(\alpha)}{\Tcheb_n(\theta) - \cos(\alpha)} = \frac{2 \sin^2\left(\frac{\alpha}{2}\right)}{\Tcheb_n(\theta) - \cos(\alpha)}.
\]
When $\Pcal_{n,\alpha}$ is bad, $\cos(\alpha) \ge 0$, thus we have
\[
\Bcal_{n, \alpha} = \frac{1 + \cos(\alpha)}{\Tcheb_n(\theta) - \cos(\alpha)} = \frac{2 \cos^2\left(\frac{\alpha}{2}\right)}{\Tcheb_n(\theta) - \cos(\alpha)}.
\]

With this notation, using \eqref{eqn:prefixdecomp}, we get that,
\[
\norm{p_{1:s}}_{L_\infty([-1,1])} \le \prod_{i=1}^k \Bcal_{2^{s_i}, \alpha_i}.
\]

Now, we introduce the key tool which will allow us to handle products of bad polynomials.
\begin{lemma}[Tree Exchange Property] \label{lem:tree-exchange} For any $0 < \alpha < \frac{\pi}{2}$, and integers $n \ge 2$, $r \geq 1$, we have
\[
\badcolor{\Bcal_{nr,\alpha}} \cdot \badcolor{\Bcal_{r,\frac{\pi - \alpha}{n}}} \leq \goodcolor{\Bcal_{nr, \pi-\alpha}} \cdot \badcolor{\Bcal_{r, \frac{\alpha}{n}}}.
\]
\end{lemma}
If we view the arguments as indexing the corresponding subtrees in our construction, then the right hand side can be viewed as exchanging a (bad) $\badcolor{\Pcal_{nr,\alpha}}$ with its (good) sibling $\goodcolor{\Pcal_{nr,\pi - \alpha}}$ at the cost of exchanging $\badcolor{\Pcal_{r,\frac{\pi - \alpha}{n}}}$ with the leftmost degree-$r$ polynomial $\badcolor{\Pcal_{r,\frac{\alpha}{n}}}$ under $\badcolor{\Pcal_{nr,\alpha}}$ (see Figure \ref{fig:exchanges}). 

\begin{proof}
Using the derived bounds on $\Bcal$, our claim reduces to proving the inequality
\begin{align*}
\frac{\cos^2\left(\frac{\alpha}{2}\right)}{\Tcheb_{nr}(\theta) - \cos(\alpha)} \cdot  \frac{ \cos^2\left(\frac{\pi - \alpha}{2n}\right)}{\Tcheb_{r}(\theta) - \cos\left(\frac{\pi - \alpha}{n}\right)} &\le \frac{ \cos^2\left(\frac{\alpha}{2}\right)}{\Tcheb_{nr}(\theta) + \cos(\alpha)} \cdot \frac{\cos^2\left(\frac{\alpha}{2n}\right)}{\Tcheb_{r}(\theta) - \cos\left(\frac{\alpha}{n}\right)},
\end{align*}
which is equivalent to the following (since all terms are positive). 
\begin{align*}
\frac{\Tcheb_{nr}(\theta) - 1 + 2 \sin^2\left(\frac{\alpha}{2}\right)}{\cos^2\left(\frac{\alpha}{2}\right)} \cdot  \frac{\Tcheb_{r}(\theta) - 1 + 2\sin^2\left(\frac{\pi - \alpha}{2n}\right)}{\cos^2\left(\frac{\pi - \alpha}{2n}\right)} &\ge \frac{\Tcheb_{nr}(\theta) - 1 + 2\cos^2\left(\frac{\alpha}{2}\right)}{ \cos^2\left(\frac{\alpha}{2}\right)} \cdot \frac{\Tcheb_{r}(\theta) - 1 + 2 \sin^2\left(\frac{\alpha}{2n}\right)}{\cos^2\left(\frac{\alpha}{2n}\right)}.
\end{align*}
For ease of exposition, let us set $\Delta_{i}(\theta) = \Tcheb_{i}(\theta) - 1$. Note that by definition of $\theta$, we have $\Delta_i(\theta) \ge 0$. Observe that the second inequality holds if the following three inequalities are true,
\begin{align}
    &\frac{\Delta_{nr}(\theta)\Delta_{r}(\theta)}{\cos^2\left(\frac{\alpha}{2}\right)}\left(\frac{1}{\cos^2\left(\frac{\pi - \alpha}{2n}\right)} - \frac{1}{\cos^2\left(\frac{\alpha}{2n}\right)}\right) \ge 0 \label{eq:prod1}\\
    &\tan^2\left(\frac{\alpha}{2}\right)\tan^2\left(\frac{\pi - \alpha}{2n}\right) - \tan^2\left(\frac{\alpha}{2n}\right) \ge 0\label{eq:prod2}\\
    &\Delta_{nr}(\theta)\left(\frac{\sin^2\left(\frac{\pi - \alpha}{2n}\right) }{\cos^2\left(\frac{\pi - \alpha}{2n}\right)\cos^2\left(\frac{\alpha}{2}\right)} - \frac{ \sin^2\left(\frac{\alpha}{2n}\right)}{ \cos^2\left(\frac{\alpha}{2n}\right)\cos^2\left(\frac{\alpha}{2}\right)}\right)+   \Delta_{r}(\theta)\left(\frac{ \sin^2\left(\frac{\alpha}{2}\right)}{\cos^2\left(\frac{\alpha}{2}\right)\cos^2\left(\frac{\pi - \alpha}{2n}\right)} -  \frac{1}{\cos^2\left(\frac{\alpha}{2n}\right)}\right) \ge 0 \label{eq:prod3}
\end{align}
Note that Equation \eqref{eq:prod1} follows from the fact that $\cos^2\left(\frac{\alpha}{2n}\right) \ge \cos^2\left(\frac{\pi - \alpha}{2n}\right) > 1/2$ since  $\frac{\pi}{4} \ge \frac{\pi - \alpha}{2n} \ge \frac{\alpha}{2n} \ge 0$.

To prove Equation \eqref{eq:prod2}, observe that it is equivalent to proving
\begin{align*}
 \sin\left(\frac{\alpha}{2}\right) \sin\left(\frac{\pi - \alpha}{2n}\right) \cos\left(\frac{\alpha}{2n}\right) - \cos\left(\frac{\alpha}{2}\right) \cos\left(\frac{\pi - \alpha}{2n}\right) \sin\left(\frac{\alpha}{2n}\right) \ge 0  
\end{align*}
Further simplifying the left hand side, we get
\begin{align*}
    &2\left(\sin\left(\frac{\alpha}{2}\right) \sin\left(\frac{\pi - \alpha}{2n}\right) \cos\left(\frac{\alpha}{2n}\right) - \cos\left(\frac{\alpha}{2}\right) \cos\left(\frac{\pi - \alpha}{2n}\right) \sin\left(\frac{\alpha}{2n}\right) \right)\\
    &= \sin\left(\frac{\alpha}{2}\right)\left(\sin\left(\frac{\pi}{2n}\right) + \sin\left(\frac{\pi - 2\alpha}{2n}\right)\right) - \cos\left(\frac{\alpha}{2}\right)\left(\sin\left(\frac{\pi}{2n}\right) - \sin\left(\frac{\pi - 2\alpha}{2n}\right)\right) \\
    &= \sin\left(\frac{\pi}{2n}\right) \left(\sin\left(\frac{\alpha}{2}\right) - \cos\left(\frac{\alpha}{2}\right)\right) + \sin\left(\frac{\pi - 2\alpha}{2n}\right) \left(\sin\left(\frac{\alpha}{2}\right) + \cos\left(\frac{\alpha}{2}\right)\right)\\
    &= \sqrt{2}\left(\sin\left(\frac{\pi - 2\alpha}{2n}\right)\cos\left(\frac{\pi - 2\alpha}{4}\right)-\sin\left(\frac{\pi}{2n}\right)\sin\left(\frac{\pi - 2\alpha}{4}\right)\right).
\end{align*}
Observe that $\frac{\sin(x)}{x}$ is a decreasing function for $0 < x \le \frac{\pi}{4}$, therefore we have $\frac{\sin\left(\frac{\pi - 2\alpha}{2n}\right)}{\frac{\pi - 2\alpha}{2n}} \ge \frac{\sin\left(\frac{\pi}{2n}\right)}{\frac{\pi}{2n}}$. Substituting this back, we get
\begin{align*}
    &\sqrt{2}\left(\sin\left(\frac{\pi - 2\alpha}{2n}\right)\cos\left(\frac{\pi - 2\alpha}{4}\right)-\sin\left(\frac{\pi}{2n}\right)\sin\left(\frac{\pi - 2\alpha}{4}\right)\right)\\
    &\ge \sqrt{2}\sin\left(\frac{\pi}{2n}\right)\left(\frac{\pi - 2 \alpha}{\pi} \cos\left(\frac{\pi - 2\alpha}{4}\right) - \sin\left(\frac{\pi - 2\alpha}{4}\right)\right) \ge 0.
\end{align*}
Here the last inequality follows from observing that $\tan(x) \le \frac{4x}{\pi}$ for $0 \le x \le \frac{\pi}{4}$. This proves Equation \eqref{eq:prod2}. Now it remains to prove Equation \eqref{eq:prod3}. Further simplifying the equation, it is equivalent to,
\begin{align*}
    \frac{\Delta_{nr}(\theta)}{\Delta_{r}(\theta)}\left(\sin^2\left(\frac{\pi - \alpha}{2n}\right)\cos^2\left(\frac{\alpha}{2n}\right) - \sin^2\left(\frac{\alpha}{2n}\right)\cos^2\left(\frac{\pi - \alpha}{2n}\right) \right) +  \sin^2\left(\frac{\alpha}{2}\right)\cos^2\left(\frac{\alpha}{2n}\right)-  \cos^2\left(\frac{\alpha}{2}\right)\cos^2\left(\frac{\pi - \alpha}{2n}\right)\ge 0
\end{align*}
Let us prove this inequality. We have
\begin{align*}
    &    \frac{\Delta_{nr}(\theta)}{\Delta_{r}(\theta)}\left(\sin^2\left(\frac{\pi - \alpha}{2n}\right)\cos^2\left(\frac{\alpha}{2n}\right) - \sin^2\left(\frac{\alpha}{2n}\right)\cos^2\left(\frac{\pi - \alpha}{2n}\right) \right) +  \sin^2\left(\frac{\alpha}{2}\right)\cos^2\left(\frac{\alpha}{2n}\right)-  \cos^2\left(\frac{\alpha}{2}\right)\cos^2\left(\frac{\pi - \alpha}{2n}\right)\\
    &  =  \frac{\Delta_{nr}(\theta)}{\Delta_{r}(\theta)}\left(\sin^2\left(\frac{\pi - \alpha}{2n}\right) - \sin^2\left(\frac{\alpha}{2n}\right) \right) + \cos^2\left(\frac{\alpha}{2}\right)\left(\cos^2\left(\frac{\alpha}{2n}\right)- \cos^2\left(\frac{\pi - \alpha}{2n}\right)\right) +  \left(\sin^2\left(\frac{\alpha}{2}\right) - \cos^2\left(\frac{\alpha}{2}\right)\right)\cos^2\left(\frac{\alpha}{2n}\right)\\
    & = \frac{\Delta_{nr}(\theta)}{\Delta_{r}(\theta)}\sin\left(\frac{\pi - 2\alpha}{2n}\right)\sin\left(\frac{\pi}{2n}\right) +  \cos^2\left(\frac{\alpha}{2}\right)\left(\cos^2\left(\frac{\alpha}{2n}\right)- \cos^2\left(\frac{\pi - \alpha}{2n}\right)\right) -  \cos(\alpha)\cos^2\left(\frac{\alpha}{2n}\right)\\
    &\ge \frac{\Delta_{nr}(\theta)}{\Delta_{r}(\theta)}\sin\left(\frac{\pi - 2\alpha}{2n}\right)\sin\left(\frac{\pi}{2n}\right)  -  \sin\left(\frac{\pi - 2\alpha}{2}\right).
\end{align*}
Here the last inequality follows from the fact that $\cos\left(\frac{\alpha}{2n}\right) \ge \cos\left(\frac{\pi - \alpha}{2n}\right)$ since $\alpha \le \frac{\pi}{2}$ and $\cos(\alpha)\cos^2\left(\frac{\alpha}{2n}\right) \le \cos(\alpha) = \sin\left(\frac{\pi - 2\alpha}{2}\right)$.

By the composition property and part (ii) of Lemma~\ref{lem:cheb-expand}, we have
\begin{align*}
\frac{\Delta_{nr}(\theta)}{\Delta_{r}(\theta)} =
\frac{\Tcheb_n( \Tcheb_{r}(\theta) ) - 1}{\Tcheb_{r}(\theta) - 1} \leq
\frac{1 + n^2 ( \Tcheb_{r}(\theta) - 1 ) - 1}{\Tcheb_{r}(\theta) - 1} \leq n^2.
\end{align*}
We also know from before that $\frac{\sin(x)}{x}$ is decreasing in $0 < x \le \frac{\pi}{4}$, therefore, we have $\frac{\sin\left(\frac{\pi}{2n}\right)}{\frac{\pi}{2n}} \ge \frac{\sin\left(\frac{\pi}{4}\right)}{\frac{\pi}{4}}$ and $\frac{\sin\left(\frac{\pi - 2\alpha}{2n}\right)}{\frac{\pi - 2\alpha}{2n}} \ge \frac{\sin\left(\frac{\pi - 2\alpha}{2}\right)}{\frac{\pi - 2\alpha}{2}}$. Combining these and substituting back, we get
\begin{align*}
    \frac{\Delta_{nr}(\theta)}{\Delta_{r}(\theta)}\sin\left(\frac{\pi - 2\alpha}{2n}\right)\sin\left(\frac{\pi}{2n}\right)  -  \sin\left(\frac{\pi - 2\alpha}{2}\right) &\ge n^2 \cdot  \frac{\sin\left(\frac{\pi - 2\alpha}{2}\right)}{n} \cdot \frac{2\sin\left(\frac{\pi}{4}\right)}{n} - \sin\left(\frac{\pi - 2\alpha}{2}\right)\\
    &= (\sqrt{2} - 1)\sin\left(\frac{\pi - 2\alpha}{2}\right) \ge 0.
\end{align*}
This completes the proof of the tree exchange lemma.
\end{proof}

\subsection{Completing the main theorems}

\begin{manualtheorem}{\ref{thm:lebedev-main}}[Prefix and suffix bounds]
For a fractal Chebyshev schedule with $m,M,T$, and all $1 \leq s, t \leq T$:
\begin{enumerate}
    \item[(i)] $\norm{p_{1:t}}_{L_\infty([m,M])} \leq \pa{ \frac{M}{m} - 1 } \prod_{j \in \bits'(t)} \frac{2}{1+\Tcheb_{2^j}(\theta)}$;
    \item[(ii)] $\norm{p_{s:T}}_{L_\infty([m,M])} \leq \prod_{j \in \bits(T+1-s)} \frac{2}{1+\Tcheb_{2^j}(\theta)}$,
\end{enumerate}
where $\bits(n)$ denotes the indices in the binary expansion of $n$, formally  defined as the unique sequence $\{j_1, j_2 \ldots \}$ with $j_1 > j_2 > j_3 \ldots$ such that $n = \sum_{j_k \in \bits(n)} 2^{j_k}$. Further we define $\bits'(n) := \bits(n) \setminus j_1 $. For example, when $n = 6 = (110)_2$, $\bits(n) = \{2,1\}$, and $\bits'(n) = \{1\}$.
\end{manualtheorem}

Starting with a prefix decomposition and repeatedly applying Lemma~\ref{lem:tree-exchange}:
\begin{align*}
    \norm{p_{1:s}}_{L_\infty([-1,1])} &= \max_{-1 \le z \le 1}\left| \prod_{i=1}^k \badcolor{\Pcal_{2^{s_i}, \alpha_i}}(z)\right| \tag{Using \eqref{eqn:prefixdecomp}}\\
    & \le  \prod_{i=1}^k \max_{-1 \le z \le 1}\left| \badcolor{\Pcal_{2^{s_i}, \alpha_i}}(z)\right|\\
    & = \prod_{i=1}^k \badcolor{\Bcal_{2^{s_i}, \alpha_i}}\\
    &= \left(\prod_{i=1}^{k-2}\badcolor{\Bcal_{2^{s_i}, \alpha_i}}\right) \cdot \badcolor{\Bcal_{2^{s_{k-1}}, \alpha_{k-1}}} \cdot \badcolor{\Bcal_{2^{s_{k}}, \frac{\pi - \alpha_{k-1}}{2^{s_{k-1} - s_{k}}}}}\\
    &\le \goodcolor{\Bcal_{2^{s_{k-1}}, \pi - \alpha_{k-1}}} \cdot \left( \prod_{i=1}^{k-2}\badcolor{\Bcal_{2^{s_i}, \alpha_{k-1}}} \right)\cdot  \badcolor{\Bcal_{2^{s_{k}},\frac{ \alpha_{k-1}}{2^{s_{k-1} - s_{k}}}}} \tag{using Lemma \ref{lem:tree-exchange}}\\
    &= \goodcolor{\Bcal_{2^{s_{k-1}}, \pi - \alpha_{k-1}}} \cdot \left( \prod_{i=1}^{k-2}\badcolor{\Bcal_{2^{s_i}, \alpha_{k-1}}} \right) \cdot \badcolor{\Bcal_{2^{s_{k}},\frac{ \pi - \alpha_{k-2}}{2^{s_{k-2} - s_{k}}}}}  \tag{using recurrence angle relation}\\
    &= \goodcolor{\Bcal_{2^{s_{k-1}}, \pi - \alpha_{k-1}}} \cdot \left( \prod_{i=1}^{k-3}\badcolor{\Bcal_{2^{s_i}, \alpha_{k-1}}} \right) \cdot \badcolor{\Bcal_{2^{s_{k-2}}, \alpha_{k-2}}} \cdot \badcolor{\Bcal_{2^{s_{k}},\frac{ \pi - \alpha_{k-2}}{2^{s_{k-2} - s_{k}}}}} \\
    &= \goodcolor{\Bcal_{2^{s_{k-1}}, \pi - \alpha_{k-1}}}\cdot \goodcolor{\Bcal_{2^{s_{k-2}}, \pi - \alpha_{k-2}}} \left( \prod_{i=1}^{k-3}\badcolor{\Bcal_{2^{s_i}, \alpha_{k-1}}} \right) \cdot \badcolor{\Bcal_{2^{s_{k}},\frac{ \alpha_{k-2}}{2^{s_{k-2} - s_{k}}}}}  \tag{using Lemma \ref{lem:tree-exchange}}\\
    &\qquad\vdots  \tag{repeating this switching argument iteratively}\\
    &\le   \left(\prod_{i=1}^{k-1}\goodcolor{\Bcal_{2^{s_{i}}, \pi - \alpha_{i}}} \right)\cdot \badcolor{\Bcal_{2^{s_{k}}, \frac{\alpha_{1}}{2^{s_{1} - s_{k}}}}}\\
    &=   \left(\prod_{i=1}^{k-1}\goodcolor{\Bcal_{2^{s_{i}}, \pi - \alpha_{i}}}\right) \cdot \badcolor{\Bcal_{2^{s_{k}}, \frac{2^{s_k}\pi}{2T}}}.
\end{align*}

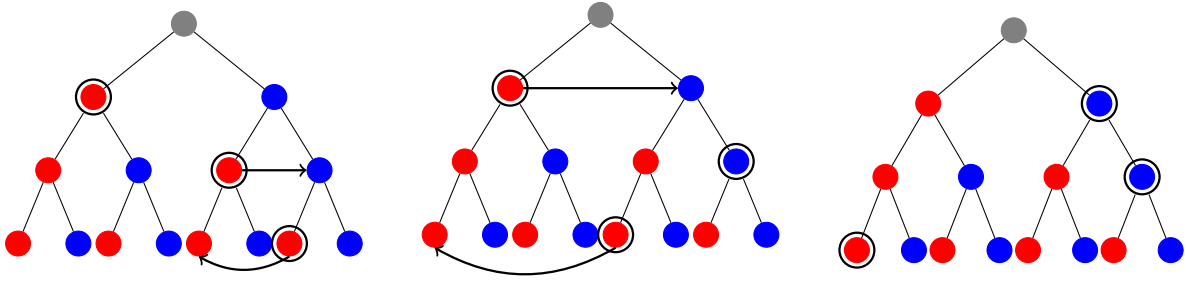
\begin{figure}
\captionsetup[subfigure]{justification=centering}
\centering
\begin{subfigure}{0.3\linewidth}
\begin{tikzpicture}
[scale=0.65, font=\small, level/.style={sibling distance=37mm/#1}, align=center ]
\node [circle,draw=gray,fill=gray] (a){}
  child {node [circle,draw=red, fill=red] (b){} node[circle,draw=black, scale=1.4, line width=0.3mm] {}
    child {node [circle,draw=red, fill=red] (c){} 
        child{node[circle,draw=red, fill=red] (d){}
        }
        child{ node[circle,draw=blue, fill=blue] (e){}
        }
    }
    child {node [circle,draw=blue, fill=blue] (f){}
        child{ node[circle,draw=red, fill=red] (g){}
        }
        child{ node[circle,draw=blue, fill=blue] (h){}
        }
    }
  }
 child {node [circle,draw=blue, fill=blue] (i) {}
    child {node[circle,draw=red, fill=red] (j){} node[circle,draw=black, scale=1.4, line width=0.3mm] {}
        child{ node[circle,draw=red, fill=red] (k){}
        }
        child{ node[circle,draw=blue, fill=blue] (l){}
        }
    }
    child {node[circle,draw=blue, fill=blue] (m){}
        child{ node[circle,draw=red, fill=red] (n){} node[circle,draw=black, scale=1.4, line width=0.3mm] {}
        }
        child{ node[circle,draw=blue, fill=blue] (o){}
        }
    }
 };
 
 \draw[->, line width=0.3mm](j) to (m);
 \draw[->, line width=0.3mm](n.south) to [out=210,in=330] (k.south);

\end{tikzpicture}
\end{subfigure}%
~$\quad$
\begin{subfigure}{0.3\linewidth}
\begin{tikzpicture}
[scale=0.65, font=\small, level/.style={sibling distance=37mm/#1}, align=center ]
\node [circle,draw=gray,fill=gray] (a){}
  child {node [circle,draw=red, fill=red] (b){} node[circle,draw=black, scale=1.4, line width=0.3mm] {}
    child {node [circle,draw=red, fill=red] (c){} 
        child{node[circle,draw=red, fill=red] (d){}
        }
        child{ node[circle,draw=blue, fill=blue] (e){}
        }
    }
    child {node [circle,draw=blue, fill=blue] (f){}
        child{ node[circle,draw=red, fill=red] (g){}
        }
        child{ node[circle,draw=blue, fill=blue] (h){}
        }
    }
  }
 child {node [circle,draw=blue, fill=blue] (i) {} 
    child {node[circle,draw=red, fill=red] (j){} 
        child{ node[circle,draw=red, fill=red] (k){} node[circle,draw=black, scale=1.4, line width=0.3mm] {}
        }
        child{ node[circle,draw=blue, fill=blue] (l){}
        }
    }
    child {node[circle,draw=blue, fill=blue] (m){} node[circle,draw=black, scale=1.4, line width=0.3mm] {}
        child{ node[circle,draw=red, fill=red] (n){} 
        }
        child{ node[circle,draw=blue, fill=blue] (o){}
        }
    }
 };
 \draw[->, line width=0.3mm](b) to (i);
 \draw[->, line width=0.3mm](k.south) to [out=210,in=330] (d.south);
\end{tikzpicture}
\end{subfigure}~$\quad$
\begin{subfigure}{0.3\linewidth}
\begin{tikzpicture}
[scale=0.65, font=\small, level/.style={sibling distance=35mm/#1}, align=center ]
\node [circle,draw=gray,fill=gray] (a){}
  child {node [circle,draw=red, fill=red] (b){} 
    child {node [circle,draw=red, fill=red] (c){} 
        child{node[circle,draw=red, fill=red] (d){} node[circle,draw=black, scale=1.4, line width=0.3mm] {}
        }
        child{ node[circle,draw=blue, fill=blue] (e){}
        }
    }
    child {node [circle,draw=blue, fill=blue] (f){}
        child{ node[circle,draw=red, fill=red] (g){}
        }
        child{ node[circle,draw=blue, fill=blue] (h){}
        }
    }
  }
 child {node [circle,draw=blue, fill=blue] (i) {} node[circle,draw=black, scale=1.4, line width=0.3mm] {}
    child {node[circle,draw=red, fill=red] (j){} 
        child{ node[circle,draw=red, fill=red] (k){} 
        }
        child{ node[circle,draw=blue, fill=blue] (l){}
        }
    }
    child {node[circle,draw=blue, fill=blue] (m){} node[circle,draw=black, scale=1.4, line width=0.3mm] {}
        child{ node[circle,draw=red, fill=red] (n){} 
        }
        child{ node[circle,draw=blue, fill=blue] (o){}
        }
    }
 };
\end{tikzpicture}
\end{subfigure}
\caption{An example of successive exchanges to fix a prefix polynomial ($p_{1:7}$). In every exchange, the product of two bad nodes are exchanged with a product of a good node and a bad node. Eventually, one is left with exactly one bad node and the remaining good nodes.}
\label{fig:exchanges}
\end{figure}

Note that this repeated exchange leads to only one bad polynomial and rest all good polynomials. Using (ii) and (iii) of Lemma \ref{lem:cheb-skew}, we get
\begin{align*}
    \norm{p_{1:s}}_{L_\infty([-1,1])} &\le \prod_{i=1}^{k-1}\frac{2}{1 + \Tcheb_{2^{s_{i}}}(\theta)} \cdot \frac{2}{4^{s_{k}}(\theta - 1)}\\
   & \le \frac{\frac{M}{m} - 1}{4^{s_k}} \cdot \prod_{i=1}^{k-1}\frac{2}{1 + \Tcheb_{2^{s_{i}}}(\theta)} \tag{using the definition of $\theta$}
\end{align*}

\section{Proofs for Section~\ref{sec:quadratic}}
\label{sec:appendix-proofs}

\subsection{Basic facts about the schedule}
\label{subsec:appendix-basics}

We provide full statements and proofs of Proposition~\ref{prop:cheb-basics}:
\begin{manualproposition}{\ref{prop:cheb-basics}}
For all $m, M, T$, the fractal Chebyshev step sizes $\gamma_t^{-1}$ satisfy the following:
\begin{enumerate}
    \item[(i)] $\frac{1}{M} < \gamma_t^{-1} < \frac{1}{m}$.
    \item[(ii)] The number of step sizes greater than $\frac{2}{M}$ is $\pa{ \frac{1}{2} - \eps }T$, where $0 \leq \eps \leq O(1/\kappahat)$ as $\kappahat \rightarrow \infty$.
    \item[(iii)] For $t \leq \frac{T}{2}$, we have $\gamma_t^{-1} < \frac{1}{m + \frac{2(M-m)t^2}{T^2}}$. Further,
    
    $\frac{1}{T}\sum_{t=1}^T \gamma_t^{-1} =  \frac{\tanh\pa{ T \acosh\left(\frac{2m}{M-m}\right)}}{\sqrt{Mm}} < \frac{1}{\sqrt{Mm}}.$
\end{enumerate}
\end{manualproposition}
(i) is obvious from the construction of $\gamma_t$, keeping in the mind the fact that $-1 < \cos(x) < 1$ for $x \in (0,\pi)$.
\begin{proof}[Proof of (ii).]
It is obvious that $\gamma_t^{-1} \geq 2/M$ for at most half of the indices $t = 1, \ldots, T$: the nodes $\gamma_t$ are symmetric with respect to reflection around the axis $\frac{M+m}{2} < \frac{M}{2}$, so at least half of them are greater than or equal to $\frac{M}{2}$. Now, let us establish the lower bound on the number of these steps. We have $\gamma_t^{-1} \leq 2/M$ if and only if
\begin{align*}
&\gamma_t = \frac{M+m}{2} - \frac{M-m}{2} \cos \frac{(t - 1/2) \pi}{T} \geq \frac{M}{2} \\
\Leftrightarrow \qquad &\cos \frac{(t - 1/2) \pi}{T} \leq \frac{m}{M-m} \\
\Leftrightarrow \qquad &t \geq \frac{1}{2} + \frac{\arccos(\frac{m}{M-m})}{\pi} \cdot T.
\end{align*}
Call the right hand side the \emph{threshold} $t^*$. Then, using the fact that $\arccos(x) \leq \frac{\pi}{2} - x$ for $x \in [0, 1]$, we have
\[ t^* \leq \frac{T}{2} + \frac{T}{\pi} \cdot \frac{m}{M-m} + \frac{1}{2}, \]
as required.
\end{proof}

\begin{proof}[Proof of (iii).]
The first statement follows from the upper bound $\cos \alpha \leq 1 - (\frac{2\alpha}{\pi})^2$, which is valid for $\alpha \leq \frac{\pi}{2}$:
\begin{align*}
\gamma_t^{-1} &= \frac{1}{\frac{M+m}{2} - \frac{M-m}{2} \cos\pa{ \frac{t-\frac{1}{2}}{T}\pi }}
< \frac{1}{\frac{M+m}{2} - \frac{M-m}{2} \cos\pa{ \frac{t}{T}\pi }}
\leq \frac{1}{\frac{M+m}{2} - \frac{M-m}{2} \pa{1 - \pa{\frac{2t}{T}}^2} },
\end{align*}
from which the claim follows.

A cheap version of the second statement, which is tight up to a constant, can be proven by viewing the summation above as a Riemann sum for a continuous integral. Interestingly, this also provides a way to prove bounds on all moments: for example, using the identity
\[ \int_{0}^1 \frac{1}{(c + x^2)^2} dx = \frac{1}{2} \pa{ \frac{\mathrm{arccot}(\sqrt{c})}{c^{3/2}} + \frac{1}{c^2+c} } \leq O(1/c^{3/2}), \]
one can verify that the root-mean-square $\sqrt{\frac{1}{T}\sum_{t=1}^T \gamma_t^{-2}}$ is bounded by $O(\kappahat^{3/4}/M)$.

The exact statement in (iii) is subtler. We suspect that this is known in the literature on Chebyshev spectral methods, but could not find a reference. Consider $q(\lambda) = x^T p(1/\lambda)$, which is $p(\lambda)$ with its coefficients reversed. Then, the sum of reciprocal roots of $p(\lambda)$ we want is the sum of roots of $q(\lambda)$. By Vi\`ete's formula, this is $-a_{T-1}/a_T$, where $a_i$ is the $\lambda^i$ coefficient of $q(\lambda)$. By the coefficient reversal, $a_T$ is the constant term of $p(\lambda)$, and $a_{T-1}$ is its linear term. Thus, we have:
\[\sum_{t=1}^T \gamma_t^{-1} = \frac{\frac{d}{d\lambda} p(\lambda) |_{\lambda = 0}}{p(0)} = \frac{2}{M-m}\frac{\frac{d}{dz}\Tcheb_{T}(z)|_{z=\theta}}{\Tcheb_T(\theta)}. \]
To reason about the derivatives, we need to introduce the Chebyshev polynomials of the \emph{second} kind, $\Ucheb_n(z)$. They can be defined as the unique polynomial satisfying
\[ \Ucheb_n( \cos \alpha ) \sin \alpha = \sin( (n+1)\alpha ).\]
From this, the cosine characterization of $\Tcheb_n(z)$, and the trigonometric substitution $z = \cos \alpha$, we have
\[ \frac{d}{dz} \Tcheb_n(z) = n \, \Ucheb_{n-1}(z). \]

By definition we have that 

\[\Ucheb_{n-1}(z) = \frac{(z + \sqrt{z^2 - 1})^n - (z - \sqrt{z^2 - 1})^n}{2 \sqrt{z^2 - 1}}, \quad \Tcheb_{n-1}(z) = \frac{(z + \sqrt{z^2 - 1})^n + (z - \sqrt{z^2 - 1})^n}{2}.\]

Therefore,
\[\frac{\Ucheb_{n-1}(z)}{\Tcheb_{n}(z)} = \frac{(z + \sqrt{z^2 - 1})^n - (z - \sqrt{z^2 - 1})^n}{(z + \sqrt{z^2 - 1})^n + (z - \sqrt{z^2 - 1})^n } \cdot \frac{1}{\sqrt{z^2 - 1}} = \frac{\tanh(n\,\acosh(z))}{\sqrt{z^2 - 1}}. \]

Substituting these back, we get that

\[\sum_{t=1}^T \gamma_t^{-1} =  \frac{2}{M-m}\frac{\frac{d}{dz}\Tcheb_{T}(z)|_{z=\theta}}{\Tcheb_T(\theta)} = \frac{2T\tanh(T\,\acosh(\theta))}{(M-m)\sqrt{\theta^2 - 1}}
= \frac{T\tanh(T\,\acosh(\theta))}{\sqrt{Mm}}.\]

\end{proof}

Finally, we prove the simple observations about the fractal permutation:
\begin{manualproposition}{\ref{prop:cheb-fractal-basis}}
For all $m,M,T$ and $0 \leq i \leq \log_2 T$:
\begin{enumerate}
    \item[(i)] The largest $\frac{T}{2^i}$ steps $\eta_t$ in the fractal Chebyshev schedule occur when $t = 1 + 2^i(\tau-1)$, with $\tau = 1, \ldots, \frac{T}{2^i}$. 
    \item[(ii)] The subsampled sequence $\{\eta_{1+2^i(\tau-1)}\}$ has the same ordering as the fractal permutation of the same length:
    \[ \eta_{1+2^i\tau} = \gamma^{-1}_{1+2^i(\tau' - 1)}, \quad \text{ where } \tau' = \sigma_{T/2^i}(\tau). \]
\end{enumerate}
\end{manualproposition}
\begin{proof}
(i) and (ii) are true by the recursive interlacing construction. Let $T_i := \frac{T}{2^i}$. The interlacing step makes it true that the indices in $\sigma_{T_i}$ are every other element of $\sigma_{2T_i}$, the indices in $\sigma_{2T_i}$ are every other element in $\sigma_{4T_i}$, and so forth. Note that $\gamma_t^{-1}$ are decreasing in $t$.
\end{proof}

\subsection{Infix polynomial bounds}
\label{subsec:appendix-infix-proof}

\begin{manualtheorem}{\ref{thm:infix-bound}}
For the fractal Chebyshev schedule with $m,M,T$, and all $1 \leq s \leq t \leq T$:
\[     \norm{p_{s:t}}_{[m,M]} \le \left(\frac{M}{m} - 1\right)  \cdot \prod_{i \in \bits(\zeta+1-s)}  \frac{2}{1 + \Tcheb_{2^{i}}(\theta)} \cdot \prod_{i \in \bits^\prime(t-\zeta)}\frac{2}{1 + \Tcheb_{2^{i}}(\theta)}, \]
where $\zeta$ is the index such that $s-1 \leq \zeta \leq t$ and $\mathrm{lca}(\zeta, \zeta+1)$ is maximized, where
\[ \mathrm{lca}(a,b) := \max \{ j : j \in \bits(a) \text{ xor } j \in \bits(b) \} \]
is the index of the most significant bit at which the binary decompositions of $a,b$ differ.
\end{manualtheorem}
\begin{figure}
\centering
\begin{tikzpicture}
[font=\small,
level 1/.style ={sibling distance=80mm},
    level 2/.style ={sibling distance=40mm},
    level 3/.style ={sibling distance=20mm},
    level 4/.style ={sibling distance=10mm},align=center ]
\node [circle,draw=gray,fill=gray] {}
  child {node [circle,draw=red, fill=red] (a){}
    node[circle,draw=black, scale=1.7, line width=0.3mm] {}
    child {node [circle,draw=red, fill=red]{}
        child{node[circle,draw=red, fill=red] {}
            child {node[circle,draw=red, fill=red] {}
            }
            child {node[circle,draw=blue, fill=blue] {}
                child { node[above=4mm] (o){$\sigma(2)$} edge from parent[draw=none]
                }
            }
        }
        child{ node[circle,draw=blue, fill=blue] {}
            child {node[circle,draw=red, fill=red] {}
            }
            child {node[circle,draw=blue, fill=blue] {}
            }
        }
    }
    child {node [circle,draw=blue, fill=blue]{} 
        child{ node[circle,draw=red, fill=red] {}
            child {node[circle,draw=red, fill=red] {}
            }
            child {node[circle,draw=blue, fill=blue] {}
            }
        }
        child{ node[circle,draw=blue, fill=blue] {}
            child {node[circle,draw=red, fill=red] {}
                child { node[above=4mm] (n){$\sigma(7)$} edge from parent[draw=none]
                }
            }
            child {node[circle,draw=blue, fill=blue] {}
            }
        }
    }
  }
 child {node [circle,draw=blue, fill=blue] {} 
    child {node[circle,draw=red, fill=red] {}
        child{ node[circle,draw=red, fill=red] {}
            child {node[circle,draw=red, fill=red] {}
            }
            child {node[circle,draw=blue, fill=blue] {}
            }
        }
        child{ node[circle,draw=blue, fill=blue] {}
            child {node[circle,draw=red, fill=red] {}
            }
            child {node[circle,draw=blue, fill=blue] {}
            }
        }
    }
    child {node[circle,draw=blue, fill=blue] {}
        child{ node[circle,draw=red, fill=red] {}
            child {node[circle,draw=red, fill=red] {}
            }
            child {node[circle,draw=blue, fill=blue] {}
            }
        }
        child{ node[circle,draw=blue, fill=blue] {}
            child {node[circle,draw=red, fill=red] {}
            }
            child {node[circle,draw=blue, fill=blue] {}
            }
        }
    }
 };
 \draw[<->](o) to node [above, midway] (b){$p_{2:7}$}(n);
 \draw[dashed](a)+(0, -4mm) to (b);
 \node[] at (-40mm,-74mm) (q) {};
 \node[] at (-55mm,-15mm) {LCA of 1 and 8};
 \draw[dashed](b)+(0, -4mm) to (q);
 \draw[<-](q) to node [above] {suffix} +(-25mm, 0);
 \draw[->](q) to node [above] {prefix} +(25mm, 0);
\end{tikzpicture}
\caption{A schematic for the decomposition of the infix $p_{2:7}$ into a suffix and a prefix polynomial corresponding to the child subtrees of the Lowest Common Ancestor}
\label{fig:infix_decomp}
\end{figure}
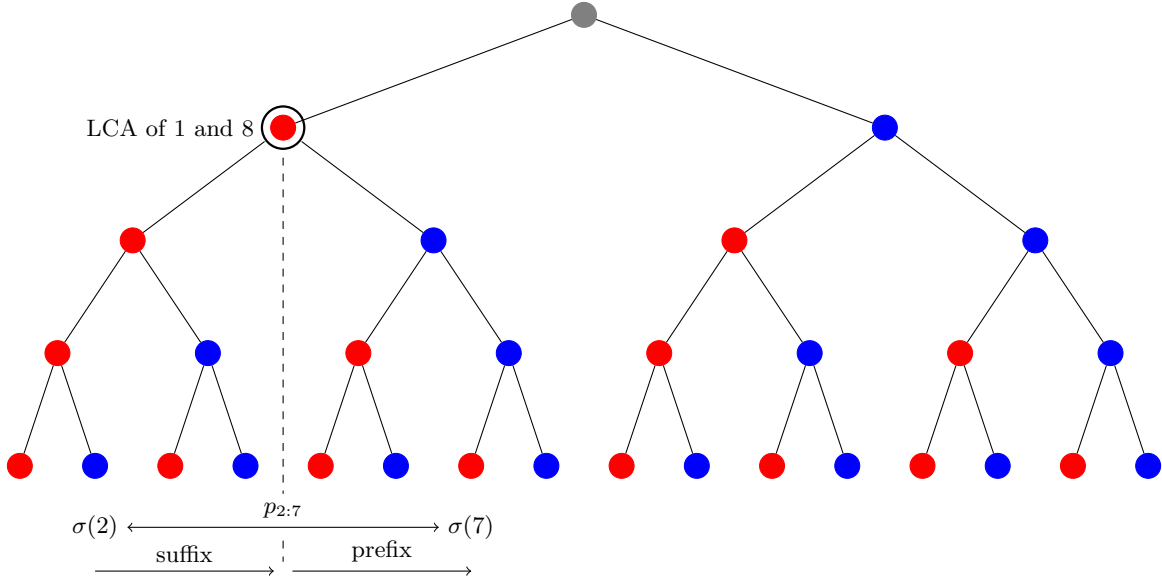%
\begin{proof}
We have previously shown how to bound the norms of the prefix and suffix polynomial. In this section, we will extend these arguments to bound any infix polynomial $p_{s:t}$ for $1< s < t < T$. To bound the norm of this polynomial, we will decompose the polynomial into two polynomials lying in disjoint subtrees. On one part we will use the suffix argument whereas on the other part we will use the prefix argument. 

To formalize this, we split based on the lowest common ancestor of node $s-1$ and $t+1$.\footnote{Note that we include $s-1$ and $t+1$ when doing this split to ensure that the polynomials do not cover the corresponding subtrees completely and are in fact prefixes and suffixes.} Consider the binary expansions of $s - 1 = 2^{s_1} + 2^{s_2} + \ldots + 2^{s_k}$ such that $s_1 > s_2 > \ldots s_k \ge 0$ and $t + 1 = 2^{t_1} + 2^{t_2} + \ldots + 2^{t_l}$ such that $t_1 > t_2 > \ldots t_l \ge 0$. Let $h$ be the minimum index such that $s_{h} \ne t_{h}$ and let $\zeta = \sum_{i = 1}^{h} 2^{t_h}$. Note that $h$ is the level of the lowest common ancestor in the tree, and $\zeta,\zeta+1$ are the indices splitting the infix between the lowest common ancestor's two subtrees. We will perform the following decomposition based on this,
\[
p_{s:t} = p_{s:\zeta} \cdot p_{\zeta+1:t}.
\]
It is not hard to see that this decomposition puts the two polynomials in disjoint subtrees: a subtree corresponding to $p_{\zeta-2^{t_h} + 1:\zeta}$ and $p_{\zeta + 1: \zeta+2^{t_h}}$. Observe that $p_{s:\zeta}$ is a suffix in the left subtree and $p_{\zeta+1:t}$ is a prefix in the right subtree. See Figure \ref{fig:infix_decomp} for a schematic depiction of the above decomposition. 

Let us first analyze $p_{s:\zeta}$. Note that $\zeta$ can be $< s$, in that case the polynomial is the empty product with norm $1$. Thus, let us assume $\zeta \ge s$. More formally, consider the binary expansions of $M \defeq \zeta+1-s = 2^{r_1} + 2^{r_2} + \ldots + 2^{r_j}$ such that $r_1 > r_2 > \ldots r_j \ge 0$. We now perform the following iterative decomposition of the  polynomial $p_{s:M}$,
\begin{align*}
p_{s:M} &= p_{s:T_1} \cdot \goodcolor{p_{T_1+1:\zeta}} &\text{where } T_1 := M - 2^{r_1}, \\
&= p_{s:T_2} \cdot \goodcolor{p_{T_2:T_1}} \cdot \goodcolor{p_{T_1+1:M}} &\text{where } T_2 := T_1 - 2^{r_2}, \\
&\ldots,
\end{align*}
As in the suffix argument, we will decompose the polynomial into good polynomials starting from the right. Recall that the right child of every node is a good polynomial. It can be seen that every intermediate polynomial $p_{T_i + 1:T_{i-1}}$ produced is a good polynomial because each one is the rightmost node at level $r_i$ (i.e. with distance $\log_2 \zeta - r_i$ from the root node of the subtree), restricted to the subtree rooted at the lowest common ancestor of roots $s$ through $T_{i-1}$ (setting $T_0 := \zeta$). Combining with statement (ii) in Lemma~\ref{lem:cheb-skew}, we get
\begin{equation}
\norm{p_{s:\zeta}}_{[m,M]} \le \prod_{i=1}^j \frac{2}{1 + \Tcheb_{2^{r_i}}(\theta)} = \prod_{i \in \bits(\zeta+1-s)}  \frac{2}{1 + \Tcheb_{2^{i}}(\theta)}
\end{equation}

Let us now look at polynomial $p_{\zeta+1:t}$. We will use the prefix argument as before on this polynomial. Consider the binary expansion of $t - \zeta = 2^{q_1} + 2^{q_2} + \ldots + 2^{q_j}$ such that $q_1 > q_2 > \ldots q_j \ge 0$. We decompose $p_{\zeta+1:t}$ into products in the following manner: starting with $\{\zeta + 1, \ldots, t\}$, we iteratively partition

\begin{align*}
p_{\zeta+1:t} &= \badcolor{p_{\zeta + 1:T_1}} \cdot p_{T_1+1:t} &\text{where } T_1 := 2^{q_1}, \\
&= \badcolor{p_{\zeta + 1:T_1}} \cdot \badcolor{p_{T_1 + 1:T_2}} \cdot p_{T_2+1:t}  &\text{where } T_2 := T_1 + 2^{q_2}, \\
&\ldots,
\end{align*}
until we reach $T_j+1:t$, which is the empty interval. Note that this partition results in all bad polynomials. We can in fact exactly characterize these polynomials. Define the angle recurrence with $\alpha_1$ being the angle corresponding to the subtree $p_{\zeta + 1: \zeta-2^{t_j}}$ and $\alpha_{i+1} = \frac{\pi - \alpha_i}{2^{q_i - q_{i+1}}}$. It can be seen that
\begin{equation*}
    p_{\zeta+1:t}= \prod_{i=1}^j \Pcal_{2^{q_i}, \alpha_i} \implies \norm{p_{\zeta+1:t}}_{[m,M]} \le \prod_{i=1}^j \Bcal_{2^{q_i}, \alpha_i}.
\end{equation*}
Using Lemma \ref{lem:tree-exchange} iteratively on this, we can see that
\begin{align*}
    \norm{p_{\zeta+1:t}}_{[m,M]} &\le \badcolor{\Bcal_{2^{q_{j}}, \frac{\alpha_1}{2^{q_1 - q_j}}}} \cdot \prod_{i=1}^{k-1}\goodcolor{\Bcal_{2^{q_{i}}, \pi - \alpha_{i}}}\\
    &\le \frac{2}{4^{q_{j}}(\theta - 1)} \cdot \prod_{i=1}^{j-1}\frac{2}{1 + \Tcheb_{2^{q_{i}}}(\theta)} \tag{using Lemma \ref{lem:cheb-skew} (ii) and (iii)}\\
   & \le \left(\frac{M}{m} - 1\right) \cdot \prod_{i \in \bits^\prime(t-\zeta)}\frac{2}{1 + \Tcheb_{2^{i}}(\theta)} \tag{using the definition of $\theta$}.
\end{align*}

This gives us the final bound on the infix as,
\begin{align}
    \norm{p_{s:t}}_{[m,M]} \le \left(\frac{M}{m} - 1\right)  \cdot \prod_{i \in \bits(\zeta+1-s)}  \frac{2}{1 + \Tcheb_{2^{i}}(\theta)} \cdot \prod_{i \in \bits^\prime(t-\zeta)}\frac{2}{1 + \Tcheb_{2^{i}}(\theta)}. \label{eq:infix}
\end{align}

\end{proof}

\subsection{Infix series bounds}
\label{subsec:appendix-infix-series-proof}

First, we provide a useful bound for an infix series contained entirely within a subtree:
\begin{lemma}
\label{lem:series-sum-main}
For any $N = 2^k$ for some $k > 0$ and any $\delta > 0$,
\[
\sum_{i=1}^N \prod_{i \in \bits(N+1-i)}  \frac{2}{1 + \Tcheb_{2^{i}}(1 + \delta)} \le \exp\left(\frac{1}{1 + \delta}\right)\cdot \left(\frac{1 + \delta}{\delta}\right)^{1/\log(4)}.
\]
\end{lemma}
\begin{proof}
Define $P_{N} := \sum_{i=1}^N \prod_{i \in \bits(N+1-i)}  \frac{2}{1 + \Tcheb_{2^{i}}(\theta)}$. Since $N$ is a power of 2, it is not hard to see that
\[
P_{N} = \left(\frac{2}{1 + T_{N/2}(\theta)}\right) \cdot P_{N/2} + P_{N/2} =  \left(1 + \frac{2}{1 + T_{N/2}(\theta)}\right) \cdot P_{N/2}.
\]
Recursively applying the above, we have
\[
P_{N} = \prod_{i=0}^{k-1}  \left(1 + \frac{2}{1 + T_{2^i}(\theta)}\right).
\]
To bound the above, let us take $\log$ on both sides. This gives us
\begin{align*}
    \log(P_{N}) &= \sum_{i=0}^{k-1} \log \left(1 + \frac{2}{1 + T_{2^i}(\theta)}\right)\\
    &\le \sum_{i=0}^{k-1}\frac{2}{1 + T_{2^i}(\theta)} \\
    &\le \sum_{i=0}^{k-1}\frac{1}{1 + 4^i\delta}\\
    &\le \frac{1}{1 + \delta} + \int_0^{k-1} \frac{1}{1 + 4^x\delta} dx \\
    &\le \frac{1}{1 + \delta} + \frac{\log(1 + \delta) - \log(\delta)}{\log 4}.
\end{align*}
Substituting back gives us the desired result.
\end{proof}

Now, we are ready to prove the main theorem about infix series sums.
\begin{manualtheorem}{\ref{thm:infix-series-bound}}
For a fractal Chebyshev schedule with $m,M,T$, and all $1 \leq s \leq t \leq T$:
\[ \sum_{t'=s}^t \norm{p_{t':t}}_{[m,M]} \leq 18\left(\frac{M}{m} - 1 \right) \left(\left(\frac{M+m}{2m}\right)^{1/\log(4)} \right) \left(1 +  \log\left(\frac{M+m}{2m}\right) \right). \]
\end{manualtheorem}
\begin{proof}
We prove the theorem only for $s=1$ which subsumes all the cases. To prove the statement we will first consider the binary expansion of $t$, $\mathrm{bits}(t)$ which is the unique sequence of numbers $\{t_1 \ldots t_k\}$ such that $t = \sum_{j=1}^k 2^{t_i}$ and $t_1 > t_2 > t_3 \ldots$. Further, define the following the sequence 
\[ \bar{t}_0 = 0; \qquad \qquad \bar{t}_j = \sum_{j' = 1}^{j} 2^{t_{j'}} \;\;\forall j \in [k].\]

We will break the sum and analyze it in the following manner:
\[
\sum_{j=1}^{k} \underbrace{ \sum_{t'=\bar{t}_{j-1}+1}^{\bar{t}_{j}} \norm{p_{t':t}}_{[m,M]}}_{:= T_j}.
\]

Before analyzing $T_j$, we establish some calculations, which will be useful. Firstly, note that for all $j$, the roots in the range $[\bar{t}_{j}+1, t]$ are all contained inside a subtree of height $2^{t_j}$ in the tree representation. Therefore the polynomial $p_{\bar{t}_{j}+1, t}$ forms a prefix polynomial within the subtree for which we can apply the bounds in Theorem \ref{thm:lebedev-main} (see usage in the proof of Theorem \ref{thm:infix-bound} on how Theorem \ref{thm:lebedev-main} applies to prefix of any subtree). Using the above, we get the following bounds:
\begin{equation}
\label{eqn:subeqn1}
    \|p_{\bar{t}_{j}+1 :t}\|_{[m,M]} \leq \left(\frac{M}{m} - 1 \right) \prod_{r=j+2}^{k} \frac{2}{1 + \Tcheb_{2^{t_r}}(\theta)}.
\end{equation}

Note that for the above and the rest of this section if the sum $\Sigma$ and product $\prod$ is over an empty set then they are assumed to be $0$ and $1$ respectively. Furthermore, note that for any $j$ and $t' \in [\bar{t}_{j-1}+2, \bar{t}_j]$, the range of roots $[t', \bar{t}_j]$, belongs to a subtree of height $2^{t_j}$ in the tree representation. In particular the infix polynomial $p_{t', \bar{t}_j}$ is a suffix polynomial within the tree and the bound from Theorem \ref{thm:lebedev-main} can be invoked to give 
\begin{equation}
\label{eqn:subeqn2}
\|p_{t', \bar{t}_j}\|_{[m,M]}  \leq \prod_{r \in \bits(\bar{t}_j+1-t')} \frac{2}{1+\Tcheb_{2^r}(\theta)}. 
\end{equation}

We are now ready to analyze the terms $T_j$ for any $j$ as follows.

\begin{align*}
    T_j &\leq \|p_{\bar{t}_{j-1}+1 :t}\|_{[m,M]} +  \sum_{t'=\bar{t}_{j-1}+2}^{\bar{t}_{j}} \norm{p_{t':\bar{t}_j}}_{[m,M]}\norm{p_{\bar{t}_j+1:t}}_{[m,M]} \\
    &\leq \left(\frac{M}{m} - 1 \right) \left( \prod_{r=j+1}^{k} \frac{2}{1 + \Tcheb_{2^{t_r}}(\theta)} + \prod_{r=j+2}^{k} \frac{2}{1 + \Tcheb_{2^{t_r}}(\theta)} \left(\sum_{t'=\bar{t}_{j-1}+2}^{\bar{t}_{j}} \left( \prod_{r \in \bits(\bar{t}_j+1-t')} \frac{2}{1+\Tcheb_{2^r}(\theta)} \right) \right) \right) \\
    &\leq \left(\frac{M}{m} - 1 \right) \left( \prod_{r=j+1}^{k} \frac{2}{1 + \Tcheb_{2^{t_r}}(\theta)} + \prod_{r=j+2}^{k} \frac{2}{1 + \Tcheb_{2^{t_r}}(\theta)} \left(3\left(\frac{\theta}{\theta - 1}\right)^{1/\log(4)} \right) \right). \\ 
    &\leq 6\left(\frac{M}{m} - 1 \right) \left( \prod_{r=j+2}^{k} \frac{2}{1 + \Tcheb_{2^{t_r}}(\theta)} \left(\left(\frac{\theta}{\theta - 1}\right)^{1/\log(4)}  \right)\right).
\end{align*}
In the above, the first inequality follows from triangle inequality, the second from \eqref{eqn:subeqn1},\eqref{eqn:subeqn2} and the third from Lemma \ref{lem:cheb-expand}. Summing over $j$ now gives us the bound as follows:
\begin{align*}
    \sum_{j=1}^k T_j \leq & 6\left(\frac{M}{m} - 1 \right) \left( \sum_{j=1}^{k-2} \left( \prod_{r=j+2}^{k} \frac{2}{1 + \Tcheb_{2^{t_r}}(\theta)} \right) + 2 \right) \left(\left(\frac{\theta}{\theta - 1}\right)^{1/\log(4)}  \right). \\
    & \leq 6\left(\frac{M}{m} - 1 \right) \left( \sum_{j=1}^{k-2} \left(  \frac{2}{1 + \Tcheb_{2^{t_{j+2}}}(\theta)} \right) + 2 \right) \left(\left(\frac{\theta}{\theta - 1}\right)^{1/\log(4)}  \right). \\
     & \leq 6\left(\frac{M}{m} - 1 \right) \left( \log\left(\frac{\theta}{\theta-1}\right) + 3 \right) \left(\left(\frac{\theta}{\theta - 1}\right)^{1/\log(4)} \right). \\
     &\leq 18\left(\frac{M}{m} - 1 \right) \left(\left(\frac{M+m}{2m}\right)^{1/\log(4)} \right) \left(1 +  \log\left(\frac{M+m}{2m}\right) \right).
\end{align*}
This concludes the theorem. 

\end{proof}
\section{Proofs for Section~\ref{sec:quadratic}}
\label{sec:appendix-proofs-misc}

\subsection{Modifications to the fractal schedule}

\begin{proof}[Proof of Proposition~\ref{prop:cheb-sched-reverse} (i)]
For any $1 \leq t \leq T-1$, consider $p_{1:t}$ and $p_{1:t+1}$ (under the reversed permutation). Both decompose into a product of good polynomials, whose levels in the tree are given by $\bits(t)$ and $\bits(t+1)$.
Notice that $\bits(t+1) \setminus \bits(t)$ contains exactly one element (the index where the carrying operation stops in binary addition); call it $c$. Then, $\bits(t) \setminus \bits(t+1) = \{1, \ldots, c-1\}$. Thus, it suffices to prove that
\[ \frac{2}{1+\Tcheb_{2^c}(\theta)} \leq \prod_{i=0}^{c-1} \frac{2}{1+\Tcheb_{2^i}(\theta)}. \]
Notice that when $i \geq 1$, we have
\[ \frac{2}{1+\Tcheb_{2^i}(\theta)} = \frac{1}{\Tcheb_{2^{i-1}}^2(\theta)}, \]
so when $c \geq 1$, the statement we wish to prove reduces to
\[ \Tcheb_{2^{c-1}}^2(\theta) \geq \frac{1+\theta}{2} \cdot \prod_{i=0}^{c-2} \Tcheb_{2^i}^2(\theta). \]
This is true because $\frac{1+\theta}{2} \leq \theta^2 = \Tcheb_1^2(\theta)$, and we can recursively apply the following inequality:
\[\Tcheb_{2n}(\theta) = 2\Tcheb_n^2(\theta) - 1 \geq \Tcheb_n^2(\theta), \]
which holds because $\theta \geq 1$.
\end{proof}

The other proofs in Section~\ref{sec:quadratic} follow immediately from the definitions.

\subsection{Overstepping with conservative parameters}

\begin{proof}[Proof of Theorem \ref{thm:underoverstepping}]
With $p = p_{1:T}$ as the shifted Chebyshev polynomial with parameters $m,M$, we wish to bound $\norm{p}_{[\lambda_{\min}, M]}.$ In the range $[m,M]$, the bounds from Theorem~\ref{thm:cheb-convergence-rate} hold; moreover, the cosine formula for the Chebyshev polynomials implies that the inequality is tight at the boundary: the maximum is achieved at $p(m)$. In the remaining part $[\lambda_{\min}, M]$, which is outside the range of the roots of $p$, $p(\lambda)$ grows monotonically as $\lambda$ decreases in this interval. Thus, it will suffice to derive a bound for $p(\lambda_{\min})$.

When $m<M$, we use the notation $\theta = \frac{M+m}{M-m}$, and define $z_{\max} := \frac{M+m-2\lambda_{\min}}{M-m}$ (the image of $\lambda_{\min}$ under the bijection). 

\begin{align*}
    \frac{\Tcheb_T(z_{\max})}{\Tcheb_n(\theta)} &= \frac{\cosh(T \, \acosh(z_{\max}))}{\cosh(T \,\acosh(\theta))} = \frac{e^{T \, \acosh(z_{\max})} + e^{-T \, \acosh(z_{\max})}}{e^{T \, \acosh(\theta)} + e^{-T \, \acosh(\theta)}} \leq \frac{2e^{T \, \acosh(z_{\max})}}{e^{T \, \acosh(\theta)}} \\
    &= 2\left(\frac{z_{\max} + \sqrt{z_{\max}^2 - 1}}{\theta + \sqrt{\theta^2 - 1}}\right)^T
    = 2\left(1 - \frac{\theta + \sqrt{\theta^2 - 1} - z_{\max} - \sqrt{z_{\max}^2 - 1}}{\theta + \sqrt{\theta^2 - 1}} \right)^T.
\end{align*}
The quantity in the fraction is equal to

\begin{align*}
\frac{\frac{M+m+2\sqrt{Mm}}{M-m} - \frac{M+m+2\lambda_{\min} + 2\sqrt{(M-\lambda_{\min})(m-\lambda_{\min})} }{M-m}}{ \frac{M+m+2\sqrt{Mm}}{M-m} }
=
2 \cdot \frac{\lambda_{\min} + \sqrt{Mm} - \sqrt{(M-\lambda_{\min})(m-\lambda_{\min})}}{ (\sqrt{M} + \sqrt{m})^2 },
\end{align*}
as required. The same is concluded for $m=M$ by taking the limit $m \rightarrow M$.

\end{proof}

\subsection{Conjugate gradient schedule}

\label{subsec:appendix-cg}

The simplest definition of the conjugate gradient algorithm, without having to worry about how to implement the iterations in linear time, is the non-iterative formula
\begin{equation}
\label{eq:cg}
x_{t+1} := \min_{\substack{\deg p \leq t \\ p(0) = 1}} \norm{ p(A) (x_1 - x^*) }_A,
\end{equation}
where $\norm{x}_A := \sqrt{x^\top A x}$, and the minimization is over polynomials with real coefficients.

\begin{manualtheorem}{\ref{thm:cg-sched}}[Conjugate gradient schedule]
For all positive definite matrices $A \in \R^{d\times d}$ and $b \in \R^d$, there exists a multiset of real numbers $\{\eta_t\}$, all in the interval $[\frac{1}{\lambda_{\max}(A)}, \frac{1}{\lambda_{\min}(A)}]$, such that $x_{T+1}$ as defined by the conjugate gradient algorithm (Equation~(\ref{eq:cg})) is equal to $x_{T+1}$ as defined by gradient descent (Equation~(\ref{eq:gd})).
\end{manualtheorem}
\begin{proof}
Let
\[ p^* \in \argmin_{\substack{\deg p \leq t \\ p(0) = 1}} \norm{ p(A) (x_1 - x^*) }_A. \]
By the fundamental theorem of algebra applied to $p^* : \R \rightarrow \R$, and noting that 0 cannot be a root of $p^*$, this is obviously true if the step sizes $\eta_t$ are allowed to be arbitrary complex numbers. We will show that there exists a minimal real-rooted polynomial that achieves the minimum in Equation~(\ref{eq:cg}), with all roots lying in the specified interval. To do this, we will start with a $p^*$ with possibly complex roots, and transform it to fit our conditions, without increasing the residual norm.

Let $\Fcal(p)$ denote the functional that returns the squared residual of a residual polynomial:
\[\Fcal(p) := \norm{ p(A) (x_1 - x^*) }_A^2 = \sum_{ (\lambda_i, u_i) \in \mathrm{eigs}(A) } \lambda_i \bra{ p(\lambda_i) }^2 (u_i ^\top (x_1 - x^*))^2,\]
where $\mathrm{eigs}(A)$ denotes the eigendecomposition of $A$.
Define a partial ordering on functions $p : [\lambda_{\min}(A), \lambda_{\max}(A)] \rightarrow \R$:
\[p \succcurlyeq q \quad \Leftrightarrow \quad |p(\lambda)| \geq |q(\lambda)| \quad \forall \lambda \in [\lambda_{\min}(A), \lambda_{\max}(A)]. \]
Notice that $\Fcal(p)$ is monotone with respect to $(\succcurlyeq)$. That is,
\[p \succcurlyeq q \quad \rightarrow \quad \Fcal(p) \succcurlyeq \Fcal(q).\]

Now we can complete the proof.

\paragraph{Roots are real w.l.o.g.} By the complex conjugate root theorem, if $p^*$ has any complex roots, they come in conjugate pairs $(a \pm bi)$ with matching multiplicities. Multiplying these in pairs gives us quadratic factors $q(a,b) := (x-a)^2 + b^2$. But $|q(a,0)| \leq |q(a,b)|$, so we can construct a real-rooted polynomial $p'$ with the same degree as $p^*$ such that $p' \preccurlyeq p^*$, by deleting the complex parts of each root.

\paragraph{Roots lie within the eigenvalue range w.l.o.g.} By the above, $p^*$ is real-rooted; write $p^*(\lambda) = \prod_{i=1}^{\deg p} (1 - \lambda/\alpha_i)$. Split the real line into intervals $I_1 = (-\infty, 0)$, $I_2 = (0, \lambda_{\min}(A))$, $I_3 = [\lambda_{\min}(A), \lambda_{\max}(A)]$, and $I_4 = (\lambda_{\max}(A), \infty)$. We will show that we can move all the roots of $p^*$ into $I_3$ without increasing $\Fcal$. For roots $\alpha \in I_1$ and $\alpha \in I_4$, notice that $(1 - \lambda/\lambda_{\max}(A)) \preccurlyeq (1 - \lambda/\alpha)$, so we can change those roots to $\lambda_{\max}(A)$. For roots $\alpha \in I_2$, notice that $(1 - \lambda/\lambda_{\min}(A)) \preccurlyeq (1 - \lambda/\alpha)$, so that we can change those roots to $\lambda_{\min}(A)$. By making these changes, we have obtained a polynomial $p''$ with the desired properties, such that $\Fcal(p'') \leq \Fcal(p^*)$.

Finally, the roots of $p^*$ give the reciprocal step sizes needed for the final iterate of gradient descent to match that of conjugate gradient. If $\deg p < T$, then we have more step sizes to assign than roots, and we can simply assign the remaining $T - \deg p$ step sizes to 0. This completes the proof.
\end{proof}

\subsection{Non-convex combination lock}
\label{subsec:appendix-nonconvex}

This construction is a simple variant of the ``needle-in-haystack'' construction for global non-convex optimization with a first-order oracle. This statement can be strengthened, but we optimize for brevity.

\begin{manualproposition}{\ref{prop:nonconvex-combination-lock}}
Let $(\eta_1^*, \ldots, \eta_T^*)$ be any sequence of positive real numbers, and $0 < \delta \leq \frac{1}{2} \min_t \eta_t^*$. Then, there exists a function $f : \R^T \rightarrow \R$ for which:
\begin{itemize}
    \item $f$ is infinitely differentiable. All of its derivatives are $O(1/\delta)$.
    \item $-1 \leq f(x) \leq 2$ for all $x \in \R^T$, and $\min_{x \in \R^T} f(x) = -1$, where $\eta_{\min} := \min_t \eta^*_t$. The minimizer is unique.
    \item Let $x_{\out}$ be the final iterate of gradient descent, starting from $x_1 = 0$ and with learning rate schedule $(\eta_1, \ldots, \eta_t)$. Then, if $\eta_t = \eta^*_t$ for each $t$, then $x_{\out} = -1$. Furthermore, for any $t$ we have
    $|\eta_t - \eta^*_t| \geq \delta,$
    then 
    $f(x_{\out}) \geq 0.$
\end{itemize}
\end{manualproposition}
\begin{proof}
We will start by constructing a non-smooth such function.
For all $z \in \R$, $\eta > 0$, define
\[g^{(T)}_{\eta^*_T}(z) = \begin{cases}
2 & z \in (-\infty, -\delta/2] \\
1-z & z \in [-\delta/2, \eta^*_T-\delta/2) \\
-1 & z \in [\eta^*_T-\delta/2, \eta^*_T+\delta/2] \\
0 & z \in (\eta^*_T + \delta/2, \infty)
\end{cases}.\]
Starting at $z=0$, one step of gradient descent on $g^{(T)}$ with learning rate $\eta$ reaches a global minimizer only if $\eta = \eta^*_T$.

Now, for each $t = T-1, \ldots, 1$, define
\[g_{\eta^*_t}^{(t)}(z,z_{t+1},\ldots,z_T) = \begin{cases}
2 & z \in (-\infty, -\delta/2] \\
1-z & z \in [-\delta/2, \eta^*_t-\delta/2) \\
g^{(t+1)}_{\eta^*_{t+1}}(z_{t+1},\ldots,z_T) & z \in [\eta^*_t-\delta/2, \eta^*_t+\delta/2] \\
0 & z \in (\eta^*_t + \delta/2, \infty)
\end{cases},\]
and so forth. Then let $g := g_{\eta^*_1}^{(1)}$ be our unsmoothed function of choice: define $f = g * \psi(2x/\delta)$, where
\[\psi(x) = \begin{cases}
\frac{1}{Z} e^{-\frac{1}{1-\norm{x}^2}} & \norm{x} < 1 \\
0 & \text{ otherwise}
\end{cases},\]
where $Z = \int_{x \in \R^T} \psi(2x/\delta) \; dx \leq O(1/\delta)$. Then:
\begin{itemize}
    \item $f$ is infinitely differentiable because $g$ is bounded and $\psi(2x/\delta)$ is infinitely differentiable.
    \item $-1 \leq f(x) \leq 2$ by Young's inequality. Since the support of $\psi(2x/\delta)$ is $\delta/2$ times the unit sphere, and $g = -1$ exactly on the $\ell_\infty$ ball of radius $\delta/2$ centered at $(\eta_1^*, \ldots, \eta_T^*)$, the $f$ has a unique minimum at $(\eta_1^*, \ldots, \eta_T^*)$.
    \item By the construction, gradient descent with learning rates $\{\eta_t^*\}$, starting at $0$, encounters the gradient sequence $\{e_t\}$, the elementary unit vectors, so it outputs the minimizer. At each iteration $t$, $x_{t+1}$ must lie in the span of $\{e_1, \ldots, e_t\}$ in order for $x_{\out}$ to reach the minimizer. If $|\eta_{t'} - \eta^*_{t'}| \geq \delta$ at any iteration $t$, this invariant cannot hold, since the next gradient is in the span of $e_t$.
\end{itemize}

\end{proof}

It may not be overly pessimistic to think of tuning the learning rate schedule in deep learning as a ``needle-in-haystack'' search problem. Learning rate schedules have been observed to affect generalization behavior in practice \citep{jiang2020characterizing,agarwal2020disentangling}, so that restarting training with a new schedule is the only way to escape poor local optima.

\subsection{No acceleration from the simple spiky schedule}
\label{subsec:appendix-vanilla-spiky}

A natural choice of self-stabilizing learning rate schedule is that which takes one large step of size $\eta^+$ to make progress in directions with shallow curvature, then several small steps of size $\eta^-$ to correct for the overshooting of the large step. This is the cyclic schedule considered by \cite{oymak2021super}, which is shown to obtain a $\log(\kappa)$ ``super-convergent'' rate under the assumption that the eigenvalues of $A$ lie in two clusters.
In this section, we provide a brief note on why this cannot obtain the $\sqrt{\kappa}$ rate on general strongly convex quadratics.

\begin{proposition}
Let $\eta^+, \eta^- \in [1/\lambda_{\max}, 1/\lambda_{\min}]$, and suppose $\eta^+ \geq 10\eta^-$. Let $n$ be a positive integer. Consider the polynomial
\[p(\lambda) := (1 - \eta^+ \lambda) (1 - \eta^- \lambda)^n.\]
Then, if $n \leq 0.1\eta^+ / \eta^-$, it must be true that
\[\norm{p^m}_{[\lambda_{\min}, \lambda_{\max}]} > 1.34^m \]
for all positive integers $m$.
\end{proposition}
\begin{proof}
We have
\[\frac{dp}{d\lambda} = -\eta^+(1-\eta^-\lambda)^n - n\eta^-(1-\eta^+\lambda)(1-\eta^-\lambda)^{n-1}
= -(1-\eta^-\lambda)^{n-1}\pa{n\eta^-(1-\eta^+\lambda) + \eta^+(1-\eta^-\lambda)},
\]
which has a root at $\lambda^* := \frac{\eta^+ + n\eta^-}{(n+1)\eta^+\eta^-}$.
Then,
\begin{align*}
\norm{p}_{[\lambda_{\min}, \lambda_{\max}]} &\geq |p(\lambda^*)|
= \frac{\frac{\eta^+}{\eta^-} - 1}{n+1} \pa{ \pa{1-\frac{\eta^-}{\eta^+}}\pa{1-\frac{1}{n+1}} }^n \\
&\geq \frac{1}{e} \frac{\frac{\eta^+}{\eta^-} - 1}{n+1} \pa{1 - \frac{\eta^-}{\eta^+}} e^{-n\eta^-/\eta^+}
\geq \frac{1}{e} \frac{\frac{\eta^+}{\eta^-} - 1}{0.1 \frac{\eta^+}{\eta^-} + 1} \pa{1 - \frac{\eta^-}{\eta^+}} e^{-n\eta^-/\eta^+} \\
&\geq \frac{1}{e} \cdot \frac{9}{2} \cdot 0.9 \cdot e^{-0.1} > 1.34.
\end{align*}
\end{proof}

Note that $p^m$ is the residual polynomial associated with repeating this cyclic schedule $m$ times. Thus, if the unstable step size in this schedule is $\kappa^\alpha$ times larger than the stable step size, the number of small steps required to prevent exponential blowup of the residual polynomial norm is $\Omega(\kappa^\alpha)$. For any $\alpha \in [0, 1]$, we have $\norm{p^m} \geq |p(\lambda_{\min})|^m \geq \pa{\exp\pa{-O(\sqrt{\kappa})} }^m$, so that $m \geq \Omega(\kappa^{1-\alpha} \log(1/\eps))$ cycles are required to make the residual norm at most $\eps$. But we have shown that each cycle requires $\Omega(\kappa^\alpha)$ steps; thus, no choice of $\eta^+, \eta^-, n$ can get a better unconditional convergence rate than $O(\kappa \log(1/\eps))$.
\section{Experimental details and supplements}
\label{sec:appendix-experiments}

\subsection{Visualization of the quadratic (theoretical) setting}
\label{subsec:appendix-perm-stability}
In Figure~\ref{fig:perm-stability}, we provide a simple illustrative numerical experiment visualizing the tradeoffs; details are below.

This is an instance in dimension $d=100$ with $A = L/\lambda_{\max}(L) + 0.1I$, where $L$ is the Laplacian matrix of the path graph on $100$ vertices; this objective is $2.2$-smooth and $0.2$-strongly convex, $b$ was sampled from $\N(0,I_{100})$, and $x_1=0$. Gradient descent (the non-accelerated baseline) was run with a learning rate of $0.9$, determined via grid search on $0.1$ increments (convergence was not significantly improved with a finer grid).
The Chebyshev nodes were chosen with $m = 0.2, M = 2.2, T = 32$, resulting in the four learning rate schedules shown.

This experiment was run with 80-bit (long double) precision, for illustrative purposes. At 32 or even 64 bits, or with larger $T$, the increasing schedule exhibits exponential blowup of numerical noise, even in this small setting. In the plot to the right, i.i.d. spherical Gaussian noise $\sim \N(0, 0.0005I)$ was added.

\subsection{One-dimensional counterexample}
\label{subsec:appendix-logcosh}

In Section~\ref{subsec:logcosh}, we noted $\log \cosh (x) + 0.01 x^2$ as a ``counterexample by numerical simulation'' to the hypothesis that gradient descent with the fractal Chebyshev schedule converges on general convex functions. This function is $1.02$-smooth and $0.02$-strongly convex.

To refute the possibility that Theorem~\ref{thm:cheb-convergence-rate} holds, it simply suffices to show that there is some setting $m \leq 0.02, M \geq 1.02$ and $T$ such that the theoretical bound does not hold. We chose $m = 0.01$, $M = 5$, $T=32$ (noting that it was quite easy to generate counterexamples). The initial iterate was set to $x_1=2$. The trajectory compared to the theoretical bound is shown in Figure~\ref{fig:logcosh}.

Results are shown in Figure~\ref{fig:logcosh}. Gradient descent (constant step size $1/M$) and Nesterov's accelerated gradient (constant step size $1/M$; momentum parameter $\gamma = 1 - \sqrt{1/0.02}$) are shown for comparison. Of course, none of these parameters are optimized; in the one-dimensional case, it is possible to reach the exact minimizer in one iteration.

\begin{figure}
    \centering
    \includegraphics[width=0.5\linewidth]{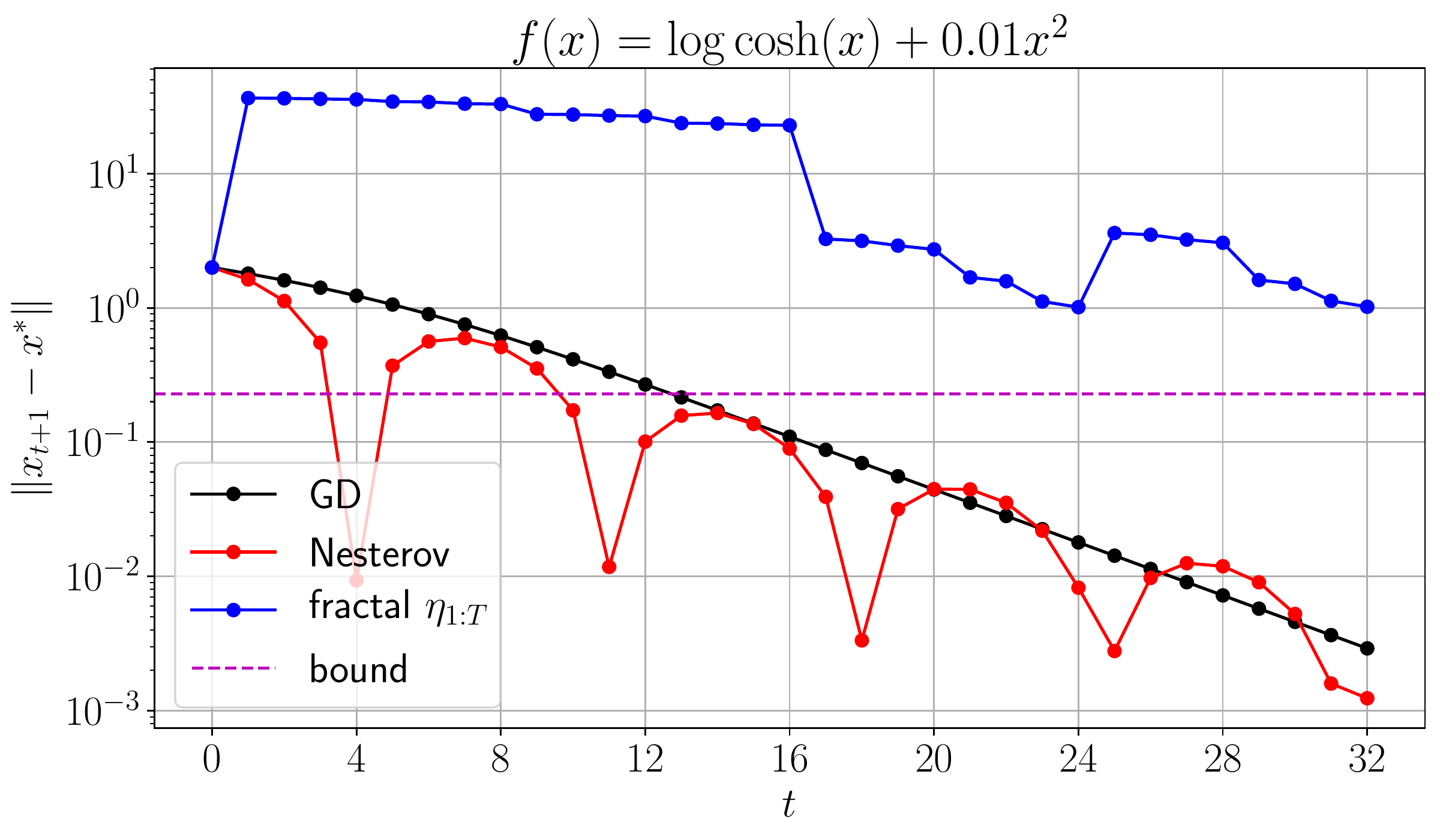}
    \caption{Non-convergent behavior of the fractal Chebyshev schedule on $f(x) = \log \cosh (x) + 0.01x^2$. The final iterate fails to follow the convergence bound from Theorem~\ref{thm:cheb-convergence-rate}.}
    \label{fig:logcosh}
\end{figure}

We conjecture that stronger negative results (say, an infinite family of counterexamples for all $T$) can be constructed.

\subsection{Convex experiments}
\label{subsec:appendix-convex-experiments}

To examine the empirical behavior of gradient descent with the fractal Chebyshev learning rate schedule on a deterministic higher-dimensional convex (but not quadratic) loss, we used the benchmark of logistic regression (with trainable biases, thus totaling $d=7850$ trainable parameters) on the MNIST dataset \cite{lecun1989backpropagation} with normalized raw pixel features, with an $\ell_2$ regularization coefficient of $10^{-3}$. The initial iterate was set to zero during all runs, for a completely deterministic setting. To measure the global minimum, we ran L-BFGS \cite{liu1989limited} until convergence to the 64-bit numerical precision floor.

\begin{figure}
    \centering
    \includegraphics[width=0.8\linewidth]{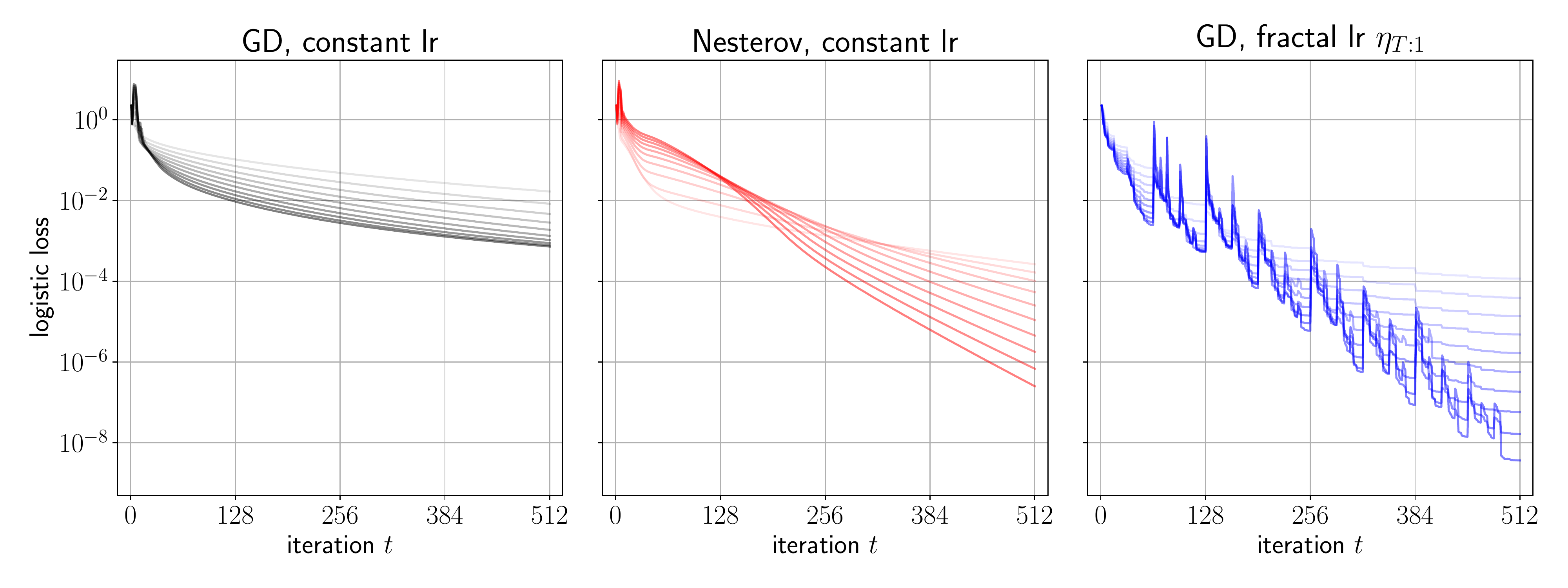}
    \caption{Convex deterministic MNIST experiments: comparison of classic non-accelerated and accelerated algorithms, and the fractal Chebyshev schedule. Most opaque curves correspond to the optimal tuned constant multiplier; lower opacity corresponds to shrinking the steps in equally spaced increments, down from $0.5$ to $0.1$. Gradient descent with the fractal schedule makes non-local progress and converges at the accelerated rate in practice.}
    \label{fig:convex-large}
\end{figure}

We compared three iterative algorithms: gradient descent with a constant learning rate (known in theory to get the slow rate), Nesterov's accelerated gradient descent with a constant learning rate and momentum parameter $0.9$ (known in theory to get the accelerated rate), and gradient descent with the reversed fractal schedule $\eta_{T:1}$ (no known theoretical guarantees in this setting). We tuned the constant learning rates in increments of $0.1$ until divergence (arriving at $0.5$ as the largest stable learning rate). For the fractal schedule, we chose $m=0.0006,M=5$, and tuned the global learning rate multiplier in increments of $0.1$ (arriving also at $0.5$). The step sizes in the schedule were in the range $[0.05, 408.65]$, with a mean of $4.56$ (\emph{much} larger than the maximum stable constant learning rate).

Figure~\ref{fig:convex-large} shows our results: in this setting, gradient descent can achieve accelerated convergence by overstepping the threshold of guaranteed local progress. We used the reverse schedule here (largest step last), as suggested for parameter stability in the noiseless setting. The forward schedule converged with accelerated rates for some choices of hyperparameters, but convergence on this non-quadratic objective was sensitive to initial large steps.

It is not the purpose of Figure~\ref{fig:convex-large} to demonstrate a comparison between Nesterov's acceleration and the fractal schedule; this is a somewhat brittle comparison and is sensitive to hyperparameter choice and floating-point precision. The quantitative comparison from this experiment is between gradient descent with the optimal constant learning rate and any fractal Chebyshev schedule which outperforms it. The Nesterov training loss curves are provided as an illustration only.

\subsection{Deep learning experiments}
\label{subsec:appendix-deeplearning}

We present some simple experiments for the fractal Chebyshev schedule on deep neural networks.
The purpose of this preliminary study is to demonstrate that the constant learning rate ``edge of stability'' can be overstepped without causing training to diverge, using a carefully designed schedule. We do not make claims about end-to-end performance improvements that are robust under ablation and tuning other hyperparameters.
A more systematic examination of the behavior of ``spiky'' learning rate schedules in deep learning is left for future work.

In the deep learning experiments, it is most convenient to think of a learning rate schedule as a time-varying multiplier on top of an existing baseline. Thus, it is most helpful to set the scaling hyperparameters to let the fractal schedule act as a ``learning rate bonus'': set $M = 1$, so that the smallest multiplier is approximately $1$, the largest $\approx 1/m$, and the mean $\approx 1/\sqrt{m}$.

\paragraph{CIFAR-10/ResNet experiments.} The experiments were conducted on the CIFAR-10 dataset on a pre-activation ResNet-18 model \cite{he2016identity} with $d\approx 11\mathrm{M}$ parameters. As a baseline, we trained the network with vanilla minibatch SGD with batch size 8192; the choice of a large batch was made to reduce confounding factors arising from stochasticity, and we omitted the usual practice of momentum in order to remove temporal correlations between step sizes. To find the edge of stability for training with a constant learning rate, we searched over the fixed learning rate parameter on an exponential grid of powers of 2, as depicted in Figure~\ref{fig:cifar-large} (right): the learning rate of $0.05$ leads to stable and the best results; at $0.1$, training is subject to destabilizing outliers, and at $0.2$, the model does not train at all; this is summarized in the fainter training loss curves in Figure~\ref{fig:cifar-large} (left).

\begin{figure}
    \centering
    \includegraphics[width=0.8\linewidth]{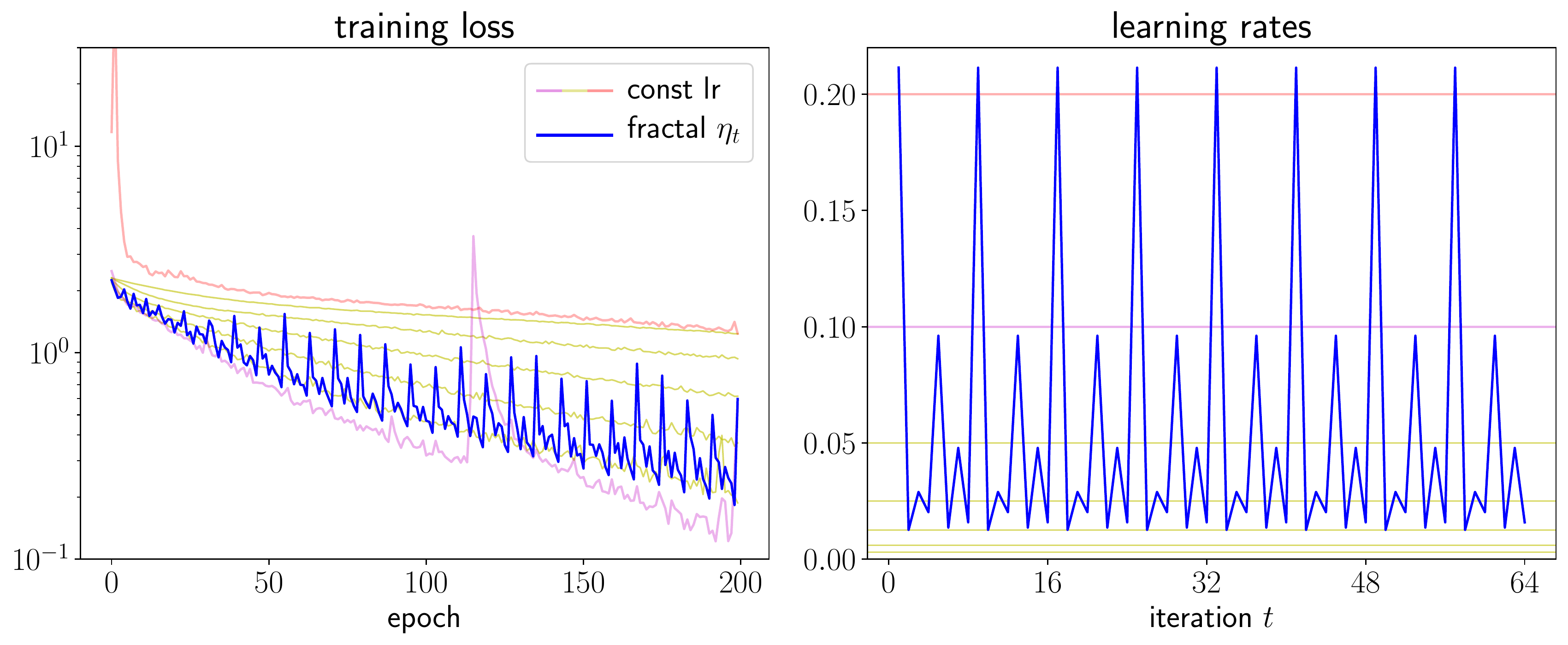}
    \caption{ResNet-18/CIFAR-10 training with batch size $8192$ and a repeated $T=8$ fractal Chebyshev schedule. \emph{Left:} Training loss curves. \emph{Right:} Learning rates; the schedule pokes through the edge of stability (magenta and red) without destabilizing training.}
    \label{fig:cifar-large}
\end{figure}

We applied a cyclic fractal schedule with $m=0.05$, $M=1$, and $T=8$, as a periodic multiplier on top of the constant learning rate $0.125$; this is pictured in Figure~\ref{fig:cifar-large} (right) as the blue curve. Although this schedule uses large learning rates that would cause unstable training, the fractal Chebyshev schedule periodically surpasses these learning rates while maintaining stable training.

We did not evaluate the model based on generalization performance (indeed, we have removed the usual practices of momentum, random cropping image augmentation, and a decaying learning rate schedule), but in this set of experiments we found the test accuracy to be slightly higher (83\%) than the best constant learning rate baseline (81\%). The stability results were consistent over 5 trials. These experiments were run in PyTorch with an $8\times$ NVIDIA Tesla V100 GPU machine, and each run took less than 30 minutes for 200 epochs.

\paragraph{MNIST experiments with a small neural network.} We chose a simpler and cheaper-to-train model to present a few more empirical insights on the behavior of fractal Chebyshev schedules beyond known theory. Namely, we use the model for MNIST classification from the PyTorch tutorial\footnote{\texttt{https://pytorch.org/tutorials/recipes/recipes/defining\_a\_neural\_network.html}}: two convolutional layers, followed by two fully-connected layers, with a total of $\sim1.2\textrm{M}$ parameters. The model was trained with SGD with batch size 1024.

\begin{figure}
    \centering
    \includegraphics[width=0.95\linewidth]{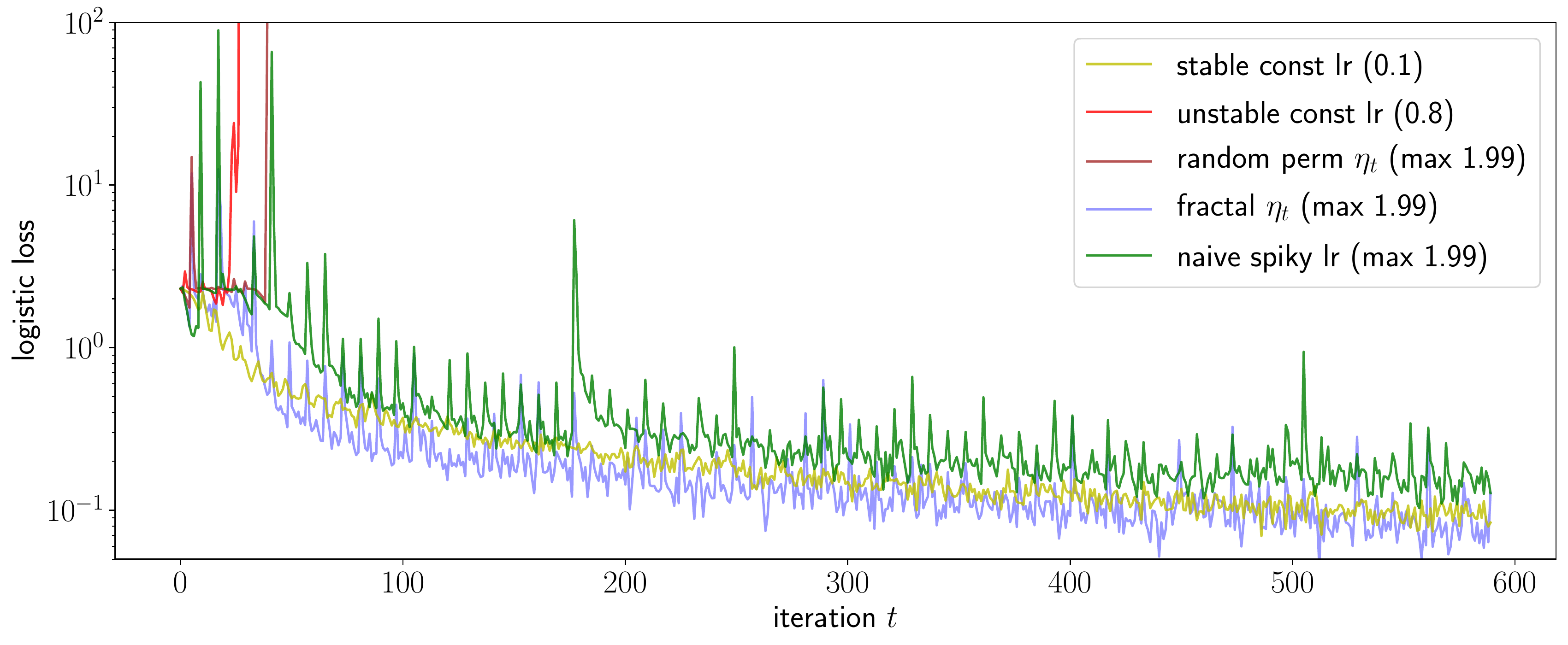}
    \caption{MNIST experiments, to show that baselines fail to stabilize training as successfully as the fractal Chebyshev schedule. The constant learning rate edge of stability is surpassed at a constant learning rate of 0.8, while the fractal schedule can take steps of up to 1.99. Random permutations of the same schedule cause divergent training, as does a simple ``spiky'' schedule which only oversteps once per cycle.}
    \label{fig:mnist}
\end{figure}

In this setup, the same methodology as the larger-scale experiments was used on a finer grid (linearly spaced between 0.1 and 0.8) to determine a stable constant learning rate (0.1) and an unstable one (0.8). A fractal Chebyshev schedule with $m=1/20, M=1, T=64$ accelerated convergence when applied to the stable constant learning rate baseline. However, randomly permuting this schedule caused divergent training. Furthermore, applying a periodic multiplier of $(20,1,1,1,1,1,1,1)$ resulted in \emph{worse} convergence. This exploratory study suggests that some of the self-stabilizing behavior of the fractal schedule in the theoretical setting (where large steps are stabilized by internal copies of Chebyshev iteration, which also consist of large steps) may hold, even for deep networks. Results were consistent over 10 trials.

These experiments were run in PyTorch on a $1\times$ NVIDIA Tesla P100 GPU machine, and each run took around less than 1 minute for 10 epochs.
\section{Additional discussion on related work}
\label{sec:appendix-literature}

\subsection{Fractal cyclic Chebyshev iterative methods}
\label{subsec:appendix-lebedev-literature}
We provide a review of the line of work that serves as the origin of these fractal permutations. These were motivated by the setting of \emph{cyclic iteration methods} for solving linear equations by least-squares in Banach spaces, a primitive in finite element methods for solving partial differential equations. All citations we could find for this line of work have been in the context of numerical methods for least-squares; much of it is in the Russian language, untranslated. We have not encountered prior work linking these methods to machine learning.

\citet{lebedev1971order} construct the fractal permutation seen in this paper, and proved the prefix and suffix bounds, as well a series bound for all prefixes (as opposed to infixes). This remarkable paper is the starting point for us (as well as the authors, evidently). Appendix~\ref{sec:appendix-lebedev} is an attempt to make that paper more accessible (it is far longer than the original paper).

\citet{lebedev1973solution} consider generalized versions of the construction, where $T$ is any positive integer, and the polynomial splitting is performed with the prime factors. They describe general conditions under which stability of a cyclic method (thus only prefix, suffix, and series bounds) can be achieved, and prove stability theorems like the previous work about constructions where the only prime factors of $T$ and 2 and 3.

\citet{lebedev1976utilization}, working in this generalized setting, also analyze the stability of a single cycle of a fractal Chebyshev schedule. They consider series sums of infixes where the series terminate at indices $d_n$ which form a divisor chain of $T$, rather than general indices.

\citet{lebedev2002construction,lebedev2004construction} provide English-language overviews of selected methods and theorems from this line of work.

In all of these works, the authors were motivated by high-precision settings in computational physics, rather than statistical or model error. This is perhaps why they were motivated to generalize the results of \citep{lebedev1971order} in their follow-up papers in a different way than ours. Thus, even though the ideas and motivations overlap with the ones considered in our work, especially in \citep{lebedev1976utilization}, the authors might not have found it important to bound the noise stability of every intermediate iteration of the algorithm. However, thinking of the perturbations as arising from statistical error or model misspecification, this is a natural notion for our setting. We could not find a way to immediately derive our estimates from any theorem or intermediate lemma in \citep{lebedev1976utilization}.

It remains an interesting direction for future work to find efficient algorithms to compute stable schedules for general $T$, and analyze their every-iterate stability like in our work. We could not see immediate ways to extend Theorems~\ref{thm:infix-bound} and \ref{thm:infix-series-bound} to their more general classes of schedules.

\subsection{Learning rate schedules and tradeoffs in practice}
\label{subsec:appendix-lr-schedule-practice}

State-of-the-art models do not show any signs of consensus towards principled or fully-understood learning rate schedules, adaptive or otherwise. A common practice has been to use a cosine learning rate schedule, originally proposed for cyclic warm restarts \cite{loshchilov2016sgdr} but widely adopted in its one-cycle form. For example, GPT-3 \citep{brown2020language} was trained with a cosine schedule. Large-scale empirical studies \cite{shallue2019measuring} indicate that the optimal choice of learning rate schedule is sensitive to the batch size. See the discussion on learning rate schedules in \cite{you2019large} for a discussion of recent empirical observations in pretraining large-scale language models.

Several papers study the theoretical tradeoffs between stability and acceleration in large-scale stochastic optimization: \cite{bottou2007tradeoffs,devolder2014first,chen2018stability,agarwal2020stochastic}. A common message throughout these papers is that the best choice of iterative optimization algorithm depends in general on the data, model, and computational resources.

\cite{cohen2020gd} provide an empirical account of the insufficiency of second-order Taylor approximations of the loss function in deciding the correct learning rate. \cite{agarwal2020disentangling} point out that learning rate schedules are entangled with adaptive gradient methods.


\subsection{Learning rate schedules in theory}
\label{subsec:appendix-lr-schedule-theory}

While learning rate schedules while ubiquitously used in practice, the diversity of existing practical learning rate schedules has received little theoretical treatment. In convex optimization, learning rate schedules have primarily been employed in stochastic settings, in particular to correctly average the zero-mean noise. It is well known in the stochastic and online optimization literature that a step schedule akin to $t^{-1/2}$ is necessary for the convergence of stochastic gradient descent. In the case of zero-mean bounded variance stochastic noise, the AC-SA algorithm proposed by \cite{lan2012optimal} which achieves optimal convergence rates employs an effective step decay schedule of $t^{-3/2}$. In a complementary line of work, \cite{ge2019step} show that for the streaming least squares regression problems, no polynomial decay of learning rates achieves the minimax rate; on the other hand the rate is achieved by the geometric decay learning rate schedule which is very popular in practice. 

An alternative point of view towards the power of learning rate schedules arises from the Polyak step size \cite{polyak1987introduction, hazan2019revisiting}, which is a single learning rate per step which generalizes the classical gradient descent oblivious to the smoothness/strong convexity properties of the function. The Polyak step size requires the knowledge of the optimality gap at any point. The vanilla version of the Polyak step size is unable to provide accelerated rates; an extension of these ideas to momentum has been carried out by \citet{barre2020complexity}. 

Practical deep learning models due to the presence of normalization layers lead to homogeneous models. For such models, \citet{li2019exponential} perhaps surprisingly show that the standard training algorithm which includes weight decay and momentum is equivalent to performing an exponentially increasing learning rate schedule. \citet{li2020reconciling} further explore the intricate interaction of weight decay and learning rates in such models proposing the notion of an intrinsic learning rate. The practice of using a large initial learning rate in optimization from the point of view of better generalization has been theoretically investigated in \cite{li2019towards} (see references herein for a detailed treatment of the topic). 

A line of work \citep{orabona2017training,cutkosky2018black} derives parameter-free algorithms for selecting learning rates which are optimal in the noise-dominated (as opposed to curvature-dominated) regime. These algorithms are shown to be practical for training deep neural networks with small batch size (e.g. a convolutional network for CIFAR-10 with batch size 128). The theory presented in this paper is only applicable to large batch/curvature-dominated settings which is the regime, where one might hope to isolate the benefits of acceleration. In small-batch/noise-dominated settings, the precise role of acceleration/learning rate schedules is muddled with confounding factors (e.g. variance reduction); see the next section for a discussion of this point. Designing adaptive algorithms which interpolate between these results and ours, like the analysis of Nesterov's acceleration under additive noise \cite{lan2012optimal}, is an interesting direction for future work; we hope that this will lead to new practical algorithms for large-scale settings.

\subsection{Acceleration methods and momentum} 

The phenomenon of acceleration in numerical analysis and optimization is a classical concept which has manifested through a large variety of viewpoints , algorithms, and analyses over the years. We provide a very short and limited summary of these manifestations, focusing on more modern machine learning focused developments. For an in-depth treatment, we strongly recommend the reader to refer the recent monograph \cite{d2021acceleration}. Possibly the earliest works on non-linear acceleration in numerical analyses date back to Aitken's $\Delta^2$ \cite{aitken1927xxv}, Wynn's epsilon algorithm \cite{wynn1956device}, and Anderson acceleration \cite{anderson1965iterative} (see \cite{sidi1986acceleration} for an in-depth survey, or the blogpost \cite{BachBlog} for a condensed description). The recent work of \cite{li2020fast} establishes an optimum rate for an Anderson acceleration method based on Chebyshev polynomials. The more standard suite of acceleration algorithms applied in machine learning arise from the direct acceleration algorithms like Polyak's momentum (also known as the heavy ball algorithm) \cite{POLYAK19641} and Nesterov's breakthrough result \cite{nesterov1983method} which established the optimal rates for general smooth convex optimization.

More recently, various acceleration algorithms \cite{allen2014linear, bubeck2015geometric} have been proposed, with more intuitive analyses than Nesterov's. Another line of work stemming from the work of \cite{su2014differential, wibisono2015accelerated} derives Nesterov-like methods via discretizations of appropriate continuous-time differential equations. A lesser known (but relevant to our work) version of direct acceleration is Nemirovski's acceleration based on a line search (\emph{not} a search over $\eta_t$ like the greedy steepest descent method); see \cite{BubeckBlogNemirovski} for a concise exposition. An alternative methodology for acceleration \cite{monteiro2013accelerated, lin2018catalyst} comes about via iteratively solving appropriate (strongly convex) proximal point problems using classical iterative methods. This latter line of work has been influential in deriving optimal accelerated versions of higher-order methods \cite{nesterov2008accelerating, bubeck2019near}.

In stochastic and/or non-convex settings (including deep learning), the role of acceleration is not fully clear. In the general convex case, worst case theory \cite{lan2012optimal} suggests that acceleration leads to benefits only in curvature-dominated regime (as opposed to the noise dominated regime). Nevertheless, heavy-ball momentum and Nesterov acceleration are part of the core toolkit in state-of-the-art optimization for optimization in various batch size regimes \cite{sutskever2013importance, kingma2014adam, dozat2016incorporating}. Recent theoretical work \cite{cutkosky2019momentum} suggests that momentum can implicitly perform variance reduction (akin to a low-pass filter), leading to improved convergence rates for stochastic optimization in non-convex problems. Understanding the variance-reducing mechanisms of the fractal schedules (or any learning rate schedule in general) is an interesting direction for future work.
 
In a recent orthogonal line of inquiry into momentum methods, \citet{pmlr-v119-pedregosa20a, pmlr-v119-scieur20a} analyze an average case setting of the quadratic model, and establish the universality of Polyak momentum as an optimal algorithm. The analysis of globally-optimized learning rate schedules in average-case settings is an interesting direction for future work.

\subsection{Optimization as a dynamical system}

Our approach to analyzing stability is most similar to the view of optimization algorithms as dynamical systems \cite{lessard2016analysis,li2017stochastic}. Of course, beyond the simplest objectives and noise models, optimization algorithms are nonlinear dynamical systems; thus, theory under this very general abstraction is very limited. \citet{bousquet2002stability} define related but stronger notions of stability, which can lead to generalization properties \cite{hardt2016train,chen2018stability,agarwal2020stochastic}.

In the dynamical systems view, our work shows that stable acceleration is obtained by treating the learning rate schedule as a long-horizon planning problem, accounting for the interactions between the choices of $\eta_t$ at different times and the global curvature of the loss. Even an open-loop control sequence (i.e. non-adaptive schedule) designed with global objectives has a provable benefit over a naive closed-loop controller (i.e. adaptive line search) which only uses instantaneous feedback (i.e. $x_t, g_t$). Thus, it may be beneficial for any closed-loop controller for the learning rate schedule to depend on global context or curvature, and possibly make negative local progress. In light of this, neural and reinforcement learning-based optimizer search \citep{bello2017neural} may be an enticing solution to the empirical problem of scheduling the learning rate with awareness of global curvature.

\end{document}